\documentclass[twoside]{article}

\usepackage[accepted]{aistats2022}

\usepackage[round]{natbib}

\usepackage{multirow}

\renewcommand{\cite}{\citep}

\usepackage{etoolbox}
\newtoggle{longv}
\togglefalse{longv}

\usepackage{tikz}
\newcommand\encircle[1]{%
  \tikz[baseline=(X.base)] 
    \node (X) [draw, shape=circle, inner sep=0] {\strut #1};}

\usepackage{amsmath}
\usepackage{booktabs}

\usepackage{url}
\usepackage{wrapfig}
\usepackage{enumitem}
\usepackage{stmaryrd}
\usepackage{array}

\newcommand{\methodName}{KDiffNet }
\newcommand{\methodNameE}{KDiffNet-E }
\newcommand{\methodNameV}{KDiffNet-G }
\newcommand{\methodNameEV}{KDiffNet-EG }
\newcommand{\diffee}{DIFFEE}

\newcommand{\titleName}{Beyond Data Samples: Aligning Differential Networks Estimation with Scientific Knowledge}

\usepackage{math-Header}

\newcommand{\sref}[1]{Section~\ref{#1}} 
\newcommand{\eref}[1]{Eq.~(\ref{#1})} 
\newcommand{\rref}[1]{Theorem~(\ref{#1})} 
 
\newcommand{\lref}[1]{Lemma~(\ref{#1})} 
\newcommand{\cref}[1]{Condition~(\ref{#1})} 
 
\newcommand{\coref}[1]{Corollary~(\ref{#1})}

\usepackage{amsmath}
\usepackage{amssymb}
\usepackage{epsfig}

\def\R{{\mathbb R}}        %

\def\P{{\mathbb P}}        %
\def\E{{\mathbb E}}        %
\def\1{{\mathbf 1}}        %

\newcommand{\topic}[4]{
   \pagestyle{myheadings}
   \thispagestyle{plain}
   \newpage
   \setcounter{page}{1}
   \noindent
   \begin{center}
   \framebox{
      \vbox{\vspace{2mm}
    \hbox to 6.28in { {\bf MATH 7360~Probability Theory I
                        \hfill Fall 2013} }
       \vspace{4mm}
       \hbox to 6.28in { {\Large \hfill Lecture #1: #2 \hfill} }
       \vspace{2mm}
			\hbox to 6.28in { {\it #4}\hfill}
       \hbox to 6.28in { {\it Lecturer:} Tai Melcher \hfill {\it Scribe:} #3}
      \vspace{2mm}}
   }
   \end{center}
   \markboth{Topic: #1}{Topic: #1}
   \vspace*{4mm}
}

\newcommand{\gqpar}{{\alpha}}

\newcommand{\numgroup}{{{s_{\GROUP}}}}
\newcommand{\SUPERGVEC}{{\vec{\gqpar}}}
\newcommand{\GROUP}{{\mathcal{G}}}

\usepackage{caption}
\usepackage{subcaption}
\usepackage{dblfloatfix}

\usepackage{algorithm}
\usepackage{algorithmic}

\usepackage{lipsum}
\usepackage{titlesec}

\usepackage{etoolbox}
\makeatletter
\preto{\@tabular}{\parskip=2pt}
\makeatother

\allowdisplaybreaks

\usepackage{enumitem}
\setlist[itemize]{leftmargin=*}

\usepackage{color}
\definecolor{colora}{rgb}{.7, .1, .1}
\newcommand{\fix}[1]{}

\newcommand{\cut}[1]{}

\let\OLDthebibliography\thebibliography
\renewcommand\thebibliography[1]{
  \OLDthebibliography{#1}
  \setlength{\parskip}{0pt}
  \setlength{\itemsep}{0pt plus 0.1ex}
}
\begin{document}

\twocolumn[

\aistatstitle{\titleName}

\aistatsauthor{ Arshdeep Sekhon \And Zhe Wang \And  Yanjun Qi }

\aistatsaddress{University of Virginia \And  University of Virginia \And University of Virginia } ] 

\begin{abstract}

Learning the differential statistical dependency network between two contexts is essential for many real-life applications, mostly in the high dimensional low sample regime. In this paper, we propose a novel differential network estimator that allows integrating various sources of knowledge beyond data samples. The proposed estimator is scalable to a large number of variables and achieves a sharp asymptotic convergence rate. Empirical experiments on extensive simulated data and four real-world applications (one on neuroimaging and three from functional genomics) show that our approach achieves improved differential network estimation and provides better supports to downstream tasks like classification. Our results highlight significant benefits of integrating group, spatial and anatomic knowledge during differential genetic network identification and brain connectome change discovery.

\end{abstract}

\section{INTRODUCTION}
\label{sec:intro}
\fix{Comments from rebuttal: reviewers kept asking why differential, we have it as part of intro, but should stress more}

\fix{group based is confusing to all reviewers}

\fix{novelty issues}

\fix{show how theoretical results show it is better, need to stress more on theoretical novelty from aistats meta reviewer}

\fix{organize content better was a major suggestion: add more details or remove those completely}

New technologies have enabled many scientific fields to measure variables at an unprecedented scale. Learning the change of variable dependencies (differential dependencies) between two contexts is an essential task in many scientific applications.  For example, when analyzing genomics signals, interests often are on how human genes interact differently when with and without an external stimulus such as SARS-CoV-2 virus \citep{ideker2012differential}. 
\iftoggle{longv}{%
  As another example, when analyzing functional magnetic resonance imaging (fMRI) samples, detecting the difference in brain connectivity networks across diseased and healthy human populations can shed light on understanding and designing treatments for psychiatric disorders \citep{di2014autism}. 
}{}
Such real world scientific needs present unique challenges and opportunities for structure discovery. %

This paper focuses on estimating structure changes of two Gaussian Graphical Models (GGMs) using samples from two different conditions.  We name this family of methods: differential GGMs, and more general as  differential network estimation. Literature includes multiple differential GGM estimators (details in Appendix Section~\ref{sec:rel}) and these estimators are mostly designed for the high dimensional data regime, with the fast-growing variable size $p$. All previous estimators made the sparsity assumption and used $\ell_1$ norm to enforce the learned differential graph as sparse.

However, this assumption mostly does not apply in the real world because there are many other beliefs real applications prefer. Previous differential network estimators can not integrate the rich set of scientific knowledge real-world tasks naturally can provide. For instance, many real-world networks include hub nodes that are densely-connected to many other nodes. Hub nodes are more prone to perturbations across two conditions (e.g., mutated p53 genes are hub nodes in the differential human gene regulatory network (gene interaction changes between cancer case and control case)~\citep{mohan2014node}).  
Therefore, allowing perturbed hubs in differential net estimation is one desired assumption; however, $\ell_1$ based regularization can't enforce such a prior. In another example, genes belonging to the same biological pathway tend to either interact with all others of the pathway (``co-activated" as a group; differential group-sparse) or not at all  (``co-deactivated," as a group; differential group-dense)  \citep{da2008systematic}. Again, the $\ell_1$ norm could not model this type of group-sparsity pattern. 
Besides, there are many sources of knowledge in real-world scientific domains, like neuroimaging experts know that spatially closed anatomical groups are more likely to connect functionally. Differential network estimators should include this complementary knowledge to help the learned models better reflect domain experts' beliefs \citep{watts1998collective}).

Unfortunately, all previous differential network estimators rely on observed samples alone. Recent advances in data generation by genomics and neuroscience call for developing new dependency identification methods tailored to the integration of multiple sources of information and provide robust results in the high dimensional low sample regime. 
This paper fills the gap by proposing a novel method, namely \methodName, to add additional \underline{K}nowledge in identifying \underline{DIFF}erential \underline{Net}works. By harnessing heterogeneous data across complementary sources, \methodName makes an essential step in enabling knowledge integration for differential dependency estimation beyond data samples. Figure~\ref{fig:kdiffnet} shows an overview of our method. This paper proposes \methodName plus multiple variations. We summarize our contributions as follows. \footnote{Due to space limit, we put details of theoretical proofs,  simulation data's setup, and detailed results when tuning hyper-parameters in the appendix. Section notations with alphabetical symbols (for example, `A:') as a prefix are for content in the appendix. We also wrap our code into an R toolkit and share via the zip appendix.}

\setlist{nolistsep}
\begin{itemize}
    \item \textbf{Beyond data samples:} \methodName is the first differential network estimator that can integrate multiple sources of evidence. We evaluate \methodName on more than 100 synthetic and multiple real-world datasets. \methodName consistently outperforms the state-of-the-art baselines and provides better down-stream prediction accuracy while achieving less or same time cost.  Our experiments showcase how \methodName can integrate knowledge like known edges, anatomical grouping, and spatial evidence when estimating differential graph from heterogeneous multivariate samples (Section~\ref{sec:exp}). We also design a meta-analysis strategy to avoid cases of mis-specified knowledge.

    \item \textbf{Theoretically Sound:} We theoretically prove the convergence error bounds of \methodName as  $O(\sqrt{\frac{\log p}{\min(n_c, n_d)}})$ , achieving the same error bound as the state-of-the-art, improving under some conditions(Section~\ref{sec:theory}). To the best of the authors’ knowledge, no known lower bounds about the convergence rate specifically under the additional knowledge setting were provided by the previous studies.
    \item \textbf{Scalable:} We design \methodName via an elementary estimator based framework and solve it using parallel proximal based optimization. \methodName scales to large $p$ and doesn't need to design knowledge-specific optimization (~\sref{sec:optm}). 
    
\end{itemize}

\section{METHOD: \methodName}
\label{sec:meth}

\subsection{Basics and $\Delta$ As Canonical Parameter of Exponential Family}
\label{sec:back}

Estimating differential GGMs includes two sets of observed samples, denoted as two matrices $\Xb_{c} \in \RR^{n_{c} \times p}$ and  $\Xb_{d} \in \RR^{n_{d} \times p}$. $\Xb_{c}$ and $\Xb_{d}$ assume i.i.d drawn from two normal distributions $N_p(\mu_c, \Sigma_c)$ and $N_p(\mu_d, \Sigma_d)$ respectively. Here $\mu_c, \mu_d \in \RR^{p}$ describe the mean vectors and  $\Sigma_c,\Sigma_d \in \RR^{p \times p}$ represent covariance matrices. The goal of differential GGMs is to estimate the structural change $\Delta$ defined by \cite{zhao2014direct} \footnote{For instance, on data samples from a controlled drug study, `c' may represent the `control' group and `d' may represent the `drug-treating' group. Using which of the two sample sets as `c' set (or `d' set) does not affect the computational cost and does not influence the statistical convergence rates.}. 
\begin{equation}
\label{def:diffNet}
\vspace{-2.8mm}
\Delta = \Omega_{d} - \Omega_{c}
\end{equation}
Here  $\Omega_c := (\Sigma_c)^{-1}$ and $\Omega_d := (\Sigma_d)^{-1}$ are two precision matrices. The sparsity pattern of the precision matrix of a GGM encodes the conditional dependency  structure of the GGM. This means, $\Delta$ describes how the magnitude of conditional dependency differs between two conditions. 
A sparse $\Delta$ means few of its entries are non-zero, indicating a differential network with few edges.

A naive approach to estimate $\Delta$ will learn $\hat{\Omega}_{d}$ and $\hat{\Omega}_{c}$ from $\Xb_{d}$ and $\Xb_{c}$ independently and calculate $\hat{\Delta}$ using~\eref{def:diffNet}. 
However, in a high-dimensional setting, the strategy needs to assume both $\Omega_{d}$ and $\Omega_{c}$ are sparse (to achieve consistent estimation of each) and has been found to produce many spurious differences \cite{de2010differential}. The assumption of this two-step procedure is often not true. For instance, genetic networks contain hub nodes, therefore not entirely sparse \cite{ideker2012differential}.
Recent literature in neuroscience has suggested that each subject's functional brain connection network may not be sparse, though differences across subjects may be sparse \cite{belilovsky2016testing}.

Interestingly, the density ratio between two Gaussian distributions falls naturally in the exponential family  (see detail proofs in  ~\sref{seca:backward}). $\Delta$ is one entry of the canonical parameter of this exponential family distribution.  According to \cite{wainwright2008graphical}, learning an exponential family distribution from data means to estimate its canonical parameter. Computing the canonical parameter of an exponential family through vanilla MLE can be expressed as a backward mapping from  given moments of the distribution \cite{wainwright2008graphical}. In the case of differential GGM, the backward mapping (i.e., the vanilla MLE solution for $\Delta$) is a simple closed form:  $\mathcal{B}(\hat{\phi}) = \mathcal{B}(\hat{\Sigma}_d, \hat{\Sigma}_c) =
\big(\hat{\Sigma}_d^{-1} - \hat{\Sigma}_c^{-1})$, easily inferred from the two sample covariance matrices. $\hat{\Sigma}$ denotes to the sample covariance matrix. However, when in high-dimensional regimes, $\mathcal{B}(\hat{\Sigma}_d, \hat{\Sigma}_c)$ is not well-defined because $\hat{\Sigma}_c$ and $\hat{\Sigma}_d$ are rank-deficient (thus not invertible). Here $\mathcal{B}$ refers to Backward Mapping. In next section, we design and use $\mathcal{B}^*$ that denotes {\it proxy} backward mapping (details later).

\subsection{$\mathcal{R}(\cdot)$ Norm based Elementary Estimators  (EE)}
\label{sec:backEE}

Multiple recent studies ~\cite{yang2014elementary,yang2014elementary1,yang2014elementary2,wang2017fastchange} followed a framework  ``Elementary estimators'':  
\begin{equation}
\label{eq:ee-back}
  \begin{split}
   & \argmin\limits_{\theta} \mathcal{R}(\theta), \\     
   & \text{Subject to: } \mathcal{R}^*(\theta -\hat{\theta}_n) \le \lambda_n 
    \end{split}
\end{equation}
Where $\mathcal{R}(\cdot)$ represents a decomposable regularization function. $\mathcal{R}^*(\cdot)$ is the dual norm of $\mathcal{R}(\cdot)$,  
\begin{equation}
\mathcal{R}^*(v) := \sup\limits_{u \ne 0}\frac{<u,v>}{\mathcal{R}(u)} = \sup\limits_{\mathcal{R}(u) \le 1}<u,v>.
\end{equation}

The design philosophy shared among elementary estimators is to construct $\hat{\theta}_n$ carefully from well-defined estimators that are easy to compute and come with strong statistical convergence guarantees.  For example, \citet{yang2014elementary1} conduct the high-dimensional estimation of $\ell_1$-regularized linear regression  by using the classical ridge estimator as $\hat{\theta}_n$  in \eref{eq:ee-back}. When $\hat{\theta}_n$ itself is closed-form and $\mathcal{R}(\cdot)$ is the $\ell_1$-norm, the solution of \eref{eq:ee-back} is naturally closed-form (as the dual norm of $\ell_1$ is $\ell_\infty$), therefore, easy and fast to compute, and scales to large $p$.

Following the above design philosophy, for our differential estimation task, $\Delta$ is the target canonical parameter $\theta$. We use a closed and well-defined form of $\hat{\theta}_n$ (suggested by \cite{wang2017fastchange}): 
\begin{equation}
   \vspace{-4mm}
\hat{\theta}_n=\mathcal{B}^*(\hat{\Sigma}_d, \hat{\Sigma}_c)=\left([T_v(\hat{\Sigma}_{d})]^{-1} - [T_v(\hat{\Sigma}_{c})]^{-1}\right)
\label{eq:backwsigma}
\end{equation}

\iftoggle{longv}{
EE based UGM estimator from \cite{yang2014elementary}  proposed the following generic formulation to estimate canonical parameter for an exponential family distribution via EE framework:
\begin{equation}
\begin{split}
 \argmin\limits_{\theta}  ||\theta||_1 ,
\hspace{3mm} \text{Subject to: }  || \theta - \mathcal{B}^*(\hat{\phi}) ||_{\infty} \le \lambda_n
\end{split}
\label{eq:eegm}
\end{equation}
For an exponential family distribution, $\theta$ is its canonical parameter to learn.  
}{}

$\mathcal{B}^*(\hat{\phi})$ denotes a so-called proxy of backward mapping for the target exponential family. Here $[T_v(A)]_{ij}:= \rho_v(A_{ij})$  
where $\rho_v(\cdot)$ was chosen as a soft-thresholding function. 
Importantly, the formulation in \eref{eq:ee-back} guarantees its solution to achieve a sharp convergence rate as long as  $\hat{\theta}_n$ is carefully chosen, well-defined, easy to compute and comes with a strong statistical convergence guarantee ~\cite{negahban2009unified}. In summary, \eref{eq:ee-back} provides an intriguing formulation to build simpler and possibly fast estimators accompanied by statistical guarantees. We, therefore, use it to design our method. To use \eref{eq:ee-back} for estimating our target parameter $\Delta$, we need to design $\mathcal{R}(\Delta)$.

\begin{figure}
    \centering
    \includegraphics[width=\linewidth]{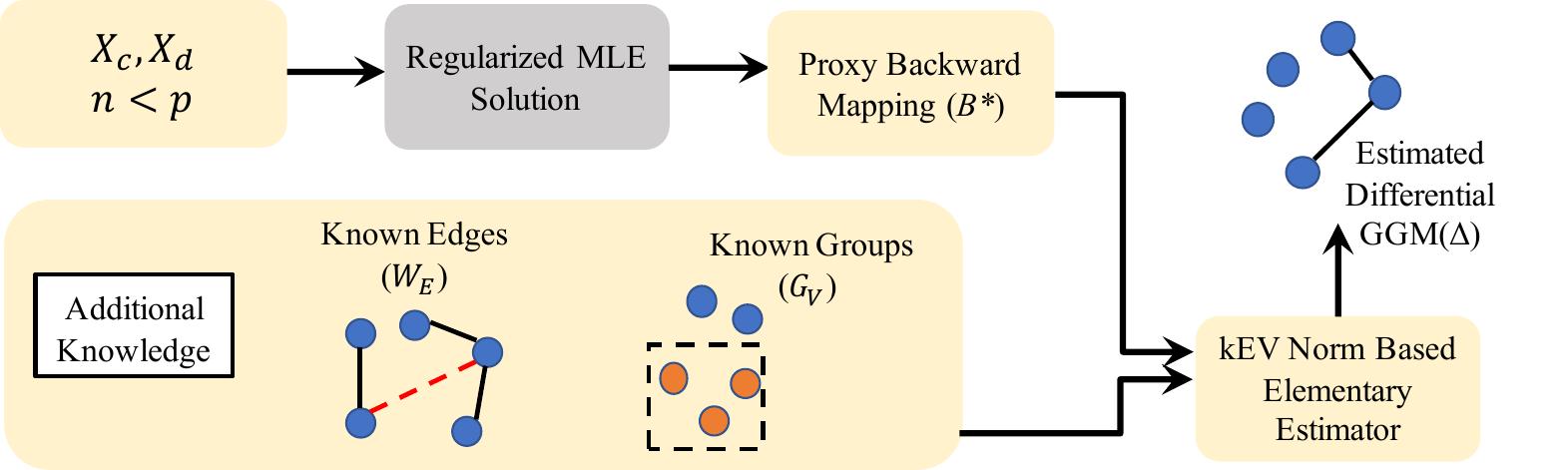}
    \caption{An overview of \methodName. \methodName integrates different types of extra knowledge for estimating differential GGMs using Elementary Estimators. As an example, the edge level knowledge can represent known edges (or non-edges) and group level knowledge represents information about multiple variables that function as groups. \label{fig:kdiffnet}}
\end{figure}

\subsection{Integrating Complementary Sources of Knowledge via a new kEV Norm: $\mathcal{R}(\Delta)$ } 
\label{sec:kevn}

All previous estimators made the sparsity assumption and used $\ell_1$ norm to enforce the learned differential graph as sparse. However, there exist many other assumptions real-life tasks may prefer. Our main goal is to enable differential network estimators to  integrate extra evidence beyond data samples. We group extra knowledge sources into two kinds: (1) edge-based, and (2) node-based.

\textbf{(1) Knowledge as Weight Matrix:} We propose to describe edge-level knowledge sources via positive weight matrices  $W_E \in \R^{p \times p}$. We use $W_E$ via a weighted $\ell_1$ formulation $||W_E \circ \Delta||_1$. This enforces the prior that the larger a weight entry in $W_E$ is, the less likely the corresponding edge belongs to the true differential graph. None of the previous differential GGMs  have explored this strategy.

The matrix $W_E$ can represent a good variety of prior knowledge. (1) For available hub nodes, we can design $W_E$ to assign all entries connecting to hubs with a smaller weight because genes tend to interact with hubs more, and hubs tend to get perturbed across conditions.
(2) As another example, $W_E$ can describe spatial distance among brain regions (publicly available in sites like openfMRI~\citep{poldrack2013toward}). This can nicely encode the domain prior that neighboring brain regions may be more likely to connect functionally.  When considering two conditions like case vs. control, these spatially close nodes tend to be the vital differential edges. (3) Another important example is when identifying gene-gene interactions from expression profiles. Many state-of-the-art bio-databases like HPRD \citep{prasad2009human} have collected information about direct ``house-keeping" physical interactions between proteins. This type of interaction tends to happen across many conditions. So we can use $W_E$ to describe that known information, proposing corresponding sparse entries in the differential net.

In summary, the $W_E$ matrix-based representation provides a powerful and flexible strategy that allows integration of many possible forms of knowledge to improve differential network estimation, as long as they can be formulated via edge-level weights.

\textbf{(2) Knowledge as Node Groups:} Many real-world applications include knowledge about how variables group into sets. For example, biologists have collected a rich set of group evidence about how genes belong to various biological pathways or exist in the same or different cellular locations \citep{da2008systematic}. Gene grouping information provides solid biological bias that genes belonging to the same pathway tend to be co-activated or co-deactivated. 

However, this type of group evidence cannot be described via the aforementioned $W_E$-based formulation. This is because it is safe to assume nodes in the same group share similar interaction patterns. However,  we do not know beforehand if a specific group functions the same across two conditions ("group sparsity" -- a block of sparse entries in the differential net) or differently between conditions ("dense sub-network" in the differential net).

To mitigate the issue, we propose to represent the group knowledge as a set of groups on feature variables (vertices) $\mathcal{G}_p$.  Mathematically, $\forall g_k \in \mathcal{G}_p$, $g_k = \{i\}$ where $i$ indicates that the $i$-th node belongs to the group $k$.  We propose integrating $\mathcal{G}_p$ knowledge into $\Delta$ by enforcing a group sparsity regularization on $\Delta$.

More specifically, we generate an ``edge-group" index $\mathcal{G}_V$ from the node group index $\mathcal{G}_p$.  This is done via defining $\mathcal{G}_V := \{ g'_k |  (i,j) \in g'_k, \forall i, \forall j \in g_k  \}$. For vertex nodes in each node group $g_k$, all possible pairs between these nodes belong to an edge-group $g'_k$. We propose to use the group,2 norm $||\Delta||_{\mathcal{G}_V,2}$ to enforce group-wise sparse structure on $\Delta$. None of the previous differential GGM estimators have explored this knowledge-integration strategy.

\textbf{kEV norm:} Now we design $\mathcal{R}(\Delta)$ as a hybrid norm that combines the two strategies above. First, we assume that the true parameter $ \Delta^*=\Delta^*_e + \Delta^*_g $: a superposition of two ``clean'' structures, $\Delta^*_e$ and $\Delta^*_g$. Then we define  $\mathcal{R}(\Delta)$ as the ``\underline{k}nowledge for \underline{E}dges and \underline{V}ertex norm (kEV-norm)":  
   \begin{equation} 
    \label{eq:kev-norm}
    \mathcal{R}(\Delta) = ||W_E \circ \Delta_e||_{1} + \epsilon||\Delta_g||_{\mathcal{G}_V,2}
    \end{equation}
Here the Hadamard product $\circ$ denotes element-wise product between two matrices (i.e. $[A\circ B]_{ij} = A_{ij}B_{ij}$).  $||\cdot||_{\mathcal{G}_V,2} = \sum\limits_k ||\Delta_{g'_k}||_2$ and $k$ denotes the $k$-th group. The positive matrix $W_E \in R^{p \times p}$ describes one aforementioned edge-level additional knowledge.  $\epsilon \geqslant 0$ is a hyperparameter.  $\mathcal{R}(\Delta)$ is the superposition of  edge-weighted $\ell_1$ norm  and the group structured norm. Our target parameter $ \Delta=\Delta_e + \Delta_g $.

\subsection{\underline{k}EV Norm based Elementary Estimator for identifying \underline{Diff}erential \underline{Net}:\methodName} 
\label{sec:ee-kdiff}

kEV-norm has three desired properties (see proofs in \sref{sec:proofnorm}): (i) kEV-norm is a norm function if $\epsilon$ and entries of $W_E$ are positive. (ii) If the condition in (i) holds, kEV-norm is a decomposable norm. (iii) The dual norm of kEV-norm is $\mathcal{R}^*(u)$. 
  \begin{equation} 
\label{eq:kev-dual}
\mathcal{R}^*(u) = \max(||(1 \varoslash W_E)\circ u||_{\infty}, \dfrac{1}{\epsilon}||u||^*_{\mathcal{G}_V,2})
   \end{equation}
Here, $(1 \varoslash W_E)$ indicates the element wise division. 

Now we define the proxy backward mapping using a closed-form formulation proposed by {\diffee}:  $\hat{\theta}_n = [T_v(\hat{\Sigma}_d)]^{-1} - [T_v(\hat{\Sigma}_c)]^{-1}$.  \sref{sec:proofbm} proves that the chosen $\hat{\theta}_n$ is theoretically well-behaved in high-dimensional settings. 

Now by plugging $\mathcal{R}(\Delta)$, its dual $\mathcal{R}^*(\cdot)$ and $\hat{\theta}_n$ into \eref{eq:ee-back}, we get the formulation of \methodName:
\begin{equation}
\begin{split}
    \label{eq:methodName_main}
    &\hspace{8mm}  \argmin\limits_{\Delta}||W_E\circ\Delta_e||_1 + \epsilon||\Delta_g||_{\mathcal{G}_V,2}  \\
    &\text{Subject to:  } \\
    &||(1 \varoslash W_E)\circ\left(\Delta - \left([T_v(\hat{\Sigma}_{d})]^{-1} -  [T_v(\hat{\Sigma}_{c})]^{-1}\right)\right)||_{\infty} \le \lambda_n\\
    &||\Delta - \left([T_v(\hat{\Sigma}_{d})]^{-1} - [T_v(\hat{\Sigma}_{c})]^{-1}\right)||^*_{\mathcal{G}_V,2}\le \epsilon\lambda_n \\
    & \Delta = \Delta_e + \Delta_g 
\end{split}
\end{equation}

\vspace{-7mm}

\subsection{Solving \methodName}
\label{sec:optm}
We then design a proximal based optimization to solve \eref{eq:methodName_main}, inspired by its distributed and parallel nature \citep{combettes2011proximal}. 
To simplify notations, we use  $\Delta_{tot}:=[\Delta_e; \Delta_g]$, where $;$ denotes the row wise concatenation. We also add three operator notations  including $L_e(\Delta_{tot})=\Delta_e$, $L_g(\Delta_{tot})=\Delta_g$ and $L_{tot}(\Delta_{tot})=\Delta_e+\Delta_g$. 
Now we re-formulate \methodName as: 
\vspace{-1mm}
\begin{equation}
\begin{split}
   \label{eq:distributed_main}
     &\argmin\limits_{\Delta_{tot}}||W_E\circ(L_e(\Delta_{tot}))||_1 + \epsilon||L_g(\Delta_{tot})||_{\mathcal{G}_V,2}  \\
    &\text{subject to:  } \\
    &||(1 \varoslash W_E)\circ(L_{tot}(\Delta_{tot}) - ([T_v(\hat{\Sigma}_{d})]^{-1} - [T_v(\hat{\Sigma}_{c})]^{-1}))||_{\infty} \le \lambda_n\\
    &||L_{tot}(\Delta_{tot}) - ([T_v(\hat{\Sigma}_{d})]^{-1} - [T_v(\hat{\Sigma}_{c})]^{-1})||^*_{\mathcal{G}_V,2}\le \epsilon\lambda_n \\
\end{split}
\end{equation}
~\eref{eq:\methodName} used proxy backward mapping $\mathcal{B}^*(\hat{\Sigma}_{d},\hat{\Sigma}_{c}) := [T_v(\hat{\Sigma}_{d})]^{-1}- [T_v(\hat{\Sigma}_{c})]^{-1}$.

Algorithm~\ref{alg:pp} in  \sref{sec:moremeth} summarizes the Parallel Proximal algorithm \citep{combettes2011proximal,yang2014elementary2} we propose for optimizing~\eref{eq:distributed_main}.
\sref{sec:Complexity} further proves its computational cost as $O(p^3)$. Detailed solutions for each proximal operator we proposed are summarized in \sref{sec:moremeth}.

\subsection{Variations  and Meta Formulation}
\label{sec:var}

There exist many variations of \methodName.\textbf{Closed-form Variations:} (1) \textbf{Edge Only or Group Only:} For instance, we can estimate the target $\Delta$ through a closed form solution if we have only one kind of additional knowledge. \sref{subsec:closed} provides the formulation and closed form solutions for edge-only or node-group-only cases. (2) \textbf{DIFFEE as our special case:} For the edge-only case, if we set $W_E$ as a matrix with all 1,  \eref{eq:methodName_main} becomes the {\diffee} formulation. \textbf{More Sets of Knowledge:} (3) We also generalize \methodName to multiple kinds of group knowledge plus multiple sources of weight knowledge in \sref{sec:moregeneral}. \textbf{Mis-specification:}  (4) When facing multiple types of evidence, misspecified evidence may exist for target goals. \sref{sec:0321-misspecify} proposes strategies to use prediction performance to guide the selective use of extra evidence sources. \textbf{Robust Covariance Estimation:} (5)  We also extend \eref{eq:methodName_main} with POET\citep{fan2013large} based robust covariance estimations when the sample size is extremely small in real-world datasets like in our two virus related gene expression experiments.

 \subsection{Analysis of Error Bounds}
 \label{sec:theory}
In this section, Theorem $2.1$ provides a  statistical analysis  under the `KEV Norm' structural constraints, leading to a non-probabilistic result that holds deterministically for all $\lambda_n$. Corollary $2.2$ provides the $asymptotic$ convergence rate in terms of how the error converges with number of dimensions $p$ and number of samples $n$, under KDiffNet's distributional assumptions. \methodName achieves a sharp convergence rate, the same convergence rate $O(\sqrt{{(\log {p})}/{n}}))$ as {\diffee}.   
We borrow the following conditions defined in \cite{yang2014elementary}, regarding the  decomposability of
regularization function $\mathcal{R}$ with respect
to the subspace pair $(\mathcal{M},\bar{\mathcal{M}}^{\perp})$:

 \textbf{(C1)} $\mathcal{R}(u+v) = \mathcal{R}(u) + \mathcal{R}(v)$, $\forall u \in \mathcal{M}, \forall v \in \bar{\mathcal{M}}^{\perp}$.

\textbf{(C2)} $\exists$ a subspace pair $(\mathcal{M},\bar{\mathcal{M}}^{\perp})$ such that the true parameter satisfies $\text{proj}_{\mathcal{M}^{\perp}}(\theta^*) = 0$

Now we introduce the following condition on `true' ${\Delta}^*$:
 \textbf{(EV-Sparsity):}   The `true'   ${\Delta}^*$ can be decomposed into two clear structures--$\{ {\Delta_e}^*$ and ${\Delta_g}^* \}$. ${\Delta_e}^*$ is exactly sparse with $s_E$ non-zero entries indexed by a support set $S_E$ . ${\Delta_g}^*$ is exactly sparse with $\sqrt{s_G}$ non-zero groups (with at least one non-zero entry) indexed by a support set $S_V$. $S_E\bigcap S_V = \emptyset$. All other elements  equal to $0$ (in $(S_E\bigcup S_V)^c$). 
 
 Section~\ref{sec:theoryMore} proves that kEV Norm satisfies conditions \textbf{(C1)} and  \textbf{(C2)}. This leads us to the following theorem (see proof ~\sref{sec:theoryMore}):
 \begin{theorem}
\label{theo:4main}
 Assuming $\Delta^*$ satisfies the condition \textbf{(EV-Sparsity)}  and $\lambda_n \geq \mathcal{R}^*(\hat{\Delta} - \Delta^*)$, then the optimal point $\hat{\Delta}$ has the following error bounds:
\end{theorem}
\vspace{-4mm}
\begin{equation}
||\hat{\Delta} - \Delta^*||_F \le 4\max(\sqrt{s_E},\epsilon\sqrt{s_G})\lambda_n 
\end{equation}
\vspace{-4mm}

 We state the following conditions on the true canonical parameter under additional knowledge defining  the class of differential GGMs: $\Delta^*=\Omega_d^*-\Omega_c^*$:

 \textbf{(C-MinInf$-\Sigma$):} The true $\Omega_c^*$ and $\Omega_d^*$ of \eref{def:diffNet} have bounded induced operator norm  i.e., $|||{\Omega_c}^*|||_{\infty} := \sup\limits_{w \ne 0 \in \R^p} \frac{||{\Omega_c}^*w||_{\infty}}{||w||_{\infty}} \le W_{E_{min}}^{c*}\kappa_1 $ and $|||{\Omega_d}^*|||_{\infty} := \sup\limits_{w \ne 0 \in \R^p} \frac{||{\Omega_d}^*w||_{\infty}}{||w||_{\infty}} \le W_{E_{min}}^{d*}\kappa_1$.
Here, intuitively, $W_{E_{min}}^{c*}$ corresponds to the largest ground truth weight index associated with non zero entries in $\Omega_c^{*}$. For set $S_{nz}=\{(i,j) | \Omega_{c_{ij}}^{*} = 0\}$, $W_{E_{S_{nz}}}> W_{E_{min}}^{c*}$.

\textbf{(C-Sparse-$\Sigma$):} The two true covariance matrices $\Sigma_c^*$ and $\Sigma_d^*$ are ``approximately sparse'' (following \cite{bickel2008covariance}). For some constant $0 \le q < 1$ and $c_0(p)$, $\max\limits_i\sum\limits_{j=1}^p|[\Sigma_{c}^*]_{ij} |^q \le c_0(p) $ and  $\max\limits_i\sum\limits_{j=1}^p|[\Sigma_{d}^*]_{ij}|^q \le c_0(p) $.  We additionally require $\inf\limits_{w \ne 0 \in \R^p} \frac{||\Sigma_c^*w||_{\infty}}{||w||_{\infty}} \ge \kappa_2$ and $\inf\limits_{w \ne 0 \in \R^p} \frac{||\Sigma_d^*w||_{\infty}}{||w||_{\infty}} \ge \kappa_2$.

 Using the above Theorem \ref{theo:4main} and conditions, we have the following corollary about the convergence rate of \methodName (see its proof in \sref{subsec:theo4proof}). 

\begin{corollary}
\label{cor:1}
    In the high-dimensional setting, i.e., $p > \max(n_c,n_d)$, let $v:= a\sqrt{\frac{\log p}{\min(n_c,n_d)}}$.  Then for $\lambda_n := \frac{\Gamma\kappa_1 a}{4\kappa_2}\sqrt{\frac{\log p}{\min(n_c,n_d)}}$ and $\min(n_c,n_d) > c \log p$, with a probability of at least $1-2C_1\exp(-C_2p\log (p))$, the estimated optimal solution $\hat{\Delta}$ has the following error bound:
 \begin{equation}
 \vspace{-4mm}
 \label{eq:rate}
 ||\hat{\Delta} - \Delta^*||_F
    \le  \frac{\Gamma a\max((\sqrt{s}_E),\epsilon\sqrt{s_G})}{\kappa_2}\sqrt{\frac{\log p}{\min(n_c,n_d)}}
    \end{equation}
Here $\Gamma=32\kappa_1\dfrac{\max(W_{E_{\min}}^{c*},W_{E_{\min}}^{d*})}{W_{E_{\min}}}$, where $a$, $c$,$C_1$, $C_2$, $\kappa_1$ and $\kappa_2$ are constants. $a$ depends on $\max_{i}\Sigma_{ii}^*$ and $c$ depends on $p,\tau,\max_{i}\Sigma_{ii}^*$. $\tau$ is a constant from Lemma 1 of \cite{ravikumar2011high}. 
\end{corollary}
\vspace{-2mm}

We can prove that under the same conditions above, {\diffee} achieves the same asymptotic convergence rate as \eref{eq:rate}. However its rate includes a different constant  $\Gamma=32\kappa_1{\max(W_{E_{\min}}^{c*},W_{E_{\min}}^{d*})}$. Notably, when $W_{E_{\min}}>1$, \methodName converging constant is better than {\diffee}. We have also included theoretical results when under misspecification assumptions and when using POET robust covariance estimation in Section~\ref{sec:theoryMore}. %

\subsection{Connecting to Relevant GGM Studies beyond Data Samples}

To the authors' best knowledge, only two loosely-related studies exist in the literature  to incorporate edge-level knowledge for other types of GGM estimation. (1) One study with the name NAK ~\citep{bu2017integrating} (following ideas from \cite{shimamura2007weighted}) proposed to integrate Additional Knowledge into the estimation of single-task graphical model via a weighted Neighbourhood selection formulation. (2) Another study with the name JEEK~\citep{wang2017fast} (following \cite{singh2017constrained}) considered edge-level evidence via a weighted objective formulation to estimate multiple dependency graphs from heterogeneous samples. Both studies only added edge-level extra knowledge in structural learning and neither of the approaches was designed for direct differential structure estimation.  Besides, JEEK uses a multi-task formulation.\footnote{Different from JEEK, our method directly estimates differential network \citep{fazayeli2016generalized}.} %

\begin{figure*}[htb]
    \centering
    \begin{subfigure}[t]{0.31\textwidth}
        \centering
        \includegraphics[width=\linewidth]{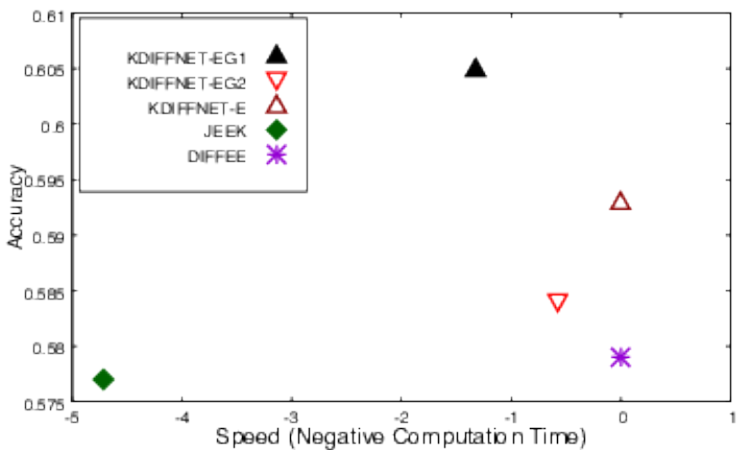}
        \caption{Classification Performance comparison on ABIDE Dataset:  \methodNameEV achieves highest Accuracy (averaged over $3$ random seeds) without sacrificing computation speed (points towards top right are better). \label{fig:abide_lambda}}
    \end{subfigure}
   \hfill 
    \begin{subfigure}[t]{0.33\textwidth}
        \centering
        \includegraphics[width=\linewidth]{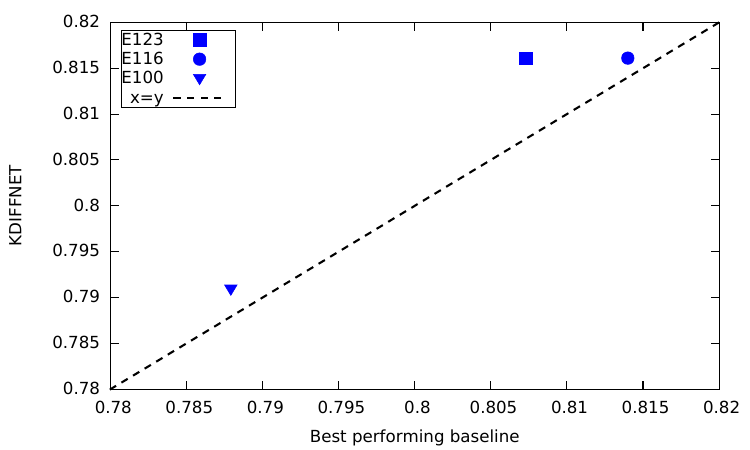}
        \caption{Classification Performance on three Epigenomic Datasets: \methodNameE achieves highest Accuracy (averaged over $3$ splits) in comparison to the best performing baseline. (points above the diagonal dashed line indicate ours is better).  \label{fig:hm_main}} 
    \end{subfigure}
    \hfill 
     \begin{subfigure}[t]{0.31\textwidth}
        \centering
        \includegraphics[width=\linewidth]{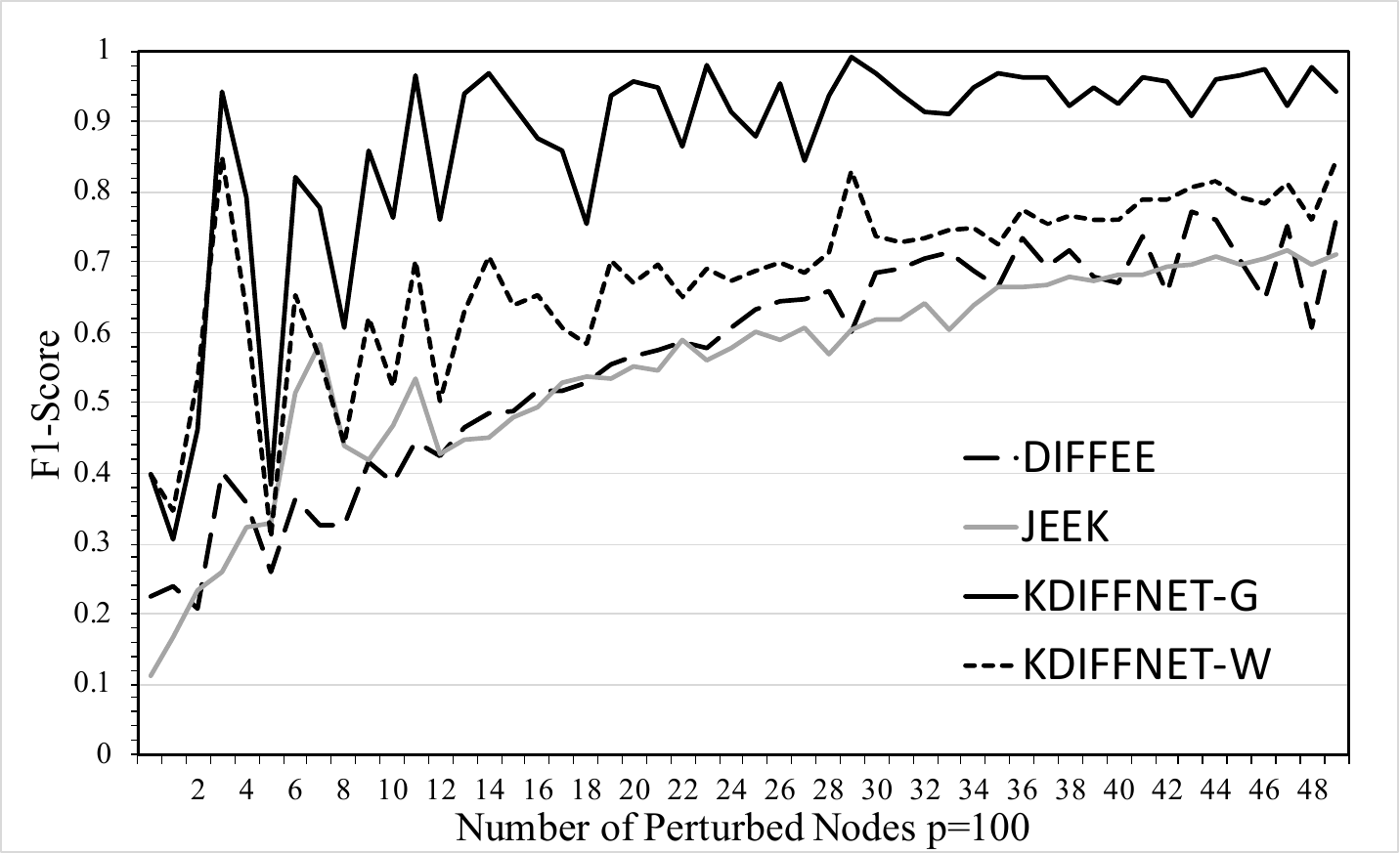}
        \caption{Edge-recovery Performance: F1-Score vs number of perturbed hub nodes. Real-life genetic networks include hub nodes that are being targeted the most by external stimulus (i.e. perturbed hubs). \label{fig:perturb}} 
    \end{subfigure}
    \vspace{-2mm}
    \caption{Classification Results for (a) Real-world Brain data (ABIDE) and (b) Real-world  epigenetic datasets, (c) Edge Recovery Accuracy for Simulation Data for Perturbed Nodes. }
\end{figure*}

\section{EXPERIMENTS}
 \label{sec:exp}

\textbf{Datasets:} We compare \methodName, variations and baselines on multiple datasets: (1) A total of $126$ different synthetic datasets representing various combinations of additional knowledge and hyper-parameter sensitivity analysis; and, (2) One fMRI dataset (ABIDE) for functional brain connectivity estimation,(3) Three epigenomic datasets for differential epigenetic network estimation, (4) Two gene expression datasets on virus (including SARS-CoV-2) infected and mock control samples for differential genetic network estimation. Results on virus related gene network identification and validation are in Section~\ref{subsec:covid-1}.

\textbf{Baselines:} 
We compare \methodName to estimators with additional knowledge: (1) JEEK\citep{wang2017fast},  (2) NAK\citep{bu2017integrating}, and estimators without any external evidence:  (3) SDRE \citep{liu2017learning},   (4) {\diffee} \citep{wang2017fastchange} and (5) JGLFUSED\citep{danaher2013joint}. We also check two variations of \methodName:  \methodNameE using only edge knowledge and \methodNameV using only group knowledge (~\sref{subsec:closed}). 

\textbf{Metrics:} For simulation datasets, we evaluate the methods in terms of  edge-level F1-Score.  \footnote{To calculate the F1-Score, we treat the number of 
true non-zero entries/edges as true positives and the number of true zero entries in the predicted $\Delta$ as true negatives. We select the best hyperparameter ($\lambda_n$,$\epsilon$) based on the best F1-Score on the training set and report the F1-Score on an unseen test set.} For the real-world datasets, due to lack of access to the ground truth $\Delta^*$, we use test accuracy obtained  using pairwise quadratic features(obtained from the edges in the difference matrix) as linear predictors.%

\textbf{Hyperparameters:} We tune the key hyper-parameters:
{\footnotesize
\begin{itemize}[noitemsep,topsep=0pt,parsep=0pt,partopsep=0pt]
    \item $v$ : To compute the proxy backward mapping, we vary $v$ in $\{ 0.001i|i = 1,2,\dots,1000 \}$ (to make $T_v(\Sigma_c)$ and $T_v(\Sigma_d)$  invertible).
    \item $\lambda_n$ :  According to our convergence rate analysis in Section~\ref{sec:theory}, $\lambda_n \ge C \sqrt{\frac{\log p}{\min(n_c,n_d)}}$,  we choose $\lambda_n$ from a range of $\{0.01 \times  \sqrt{\frac{\log p}{\min(n_c,n_d)}} \times i| i \in \{ 1,2,3,\dots, 100 \} \}$ using cross-validation.  
     For \methodNameV, we tune over $\lambda_n$ from a range of $\{0.1 \times  \sqrt{\frac{\log p}{\min(n_c,n_d)}} \times i| i \in \{ 1,2,3,\dots, 100 \} \}$\footnote{We use the same range to tune $\lambda_1$ for SDRE and $\lambda_2$ for JGLFUSED. We use  $\lambda_1 = 0.0001$(a small value) for JGLFUSED to ensure only the differential network is sparse. Tuning NAK is done by the package itself.}.
    \item $\epsilon$: For \methodNameEV, we tune $\epsilon \in \{ 0.0001,0.01,1,100 \}$.
\end{itemize}
}
\subsection{Experiment 1: Simulation Datasets}
\label{subsec:simgen}
In the following subsections, we present details about the data generation, followed by results under multiple settings.

\paragraph{Data Generation}:For overlapping Edge and Vertex Knowledge (KEG),  we generate simulated datasets (Data-EG) with a clear underlying differential structure between two conditions. We simulate the case of overlapping group and edge knowledge.   
  We select the block diagonals of size $m$ as groups in $\Delta^g$. If two variables $i,j$ are in a group $g'$, in $\Delta^g_{ij}=1/3$, else $\Delta^g_{ij}=0$, where $\Delta^g \in \R^{p \times p}$. For the edge-level knowledge component, given a known weight matrix $W_E$, we set $W^d = inv.logit(-W_E)$. Higher the value of $W_{E_{ij}}$, lower the value of $W^d_{ij}$, hence lower the probability of that edge to occur in the true precision matrix. We select different levels in the matrix $W^d$, denoted by $s$, where if $W^d_{ij}>s_l$, we set $\Delta^d_{ij}=1/3$, else $\Delta^d_{ij}=0$.  ${B}_I$ is a random graph with  each edge ${B}_{{I}_{ij}}=1/3$ with probability $p$. $ {\Omega}_d =\Delta^d + \Delta^g + {B}_I + \delta_d I 
   $, ${\Omega}_c ={B}_I + \delta_c I $, finally, $ {\Delta} = {\Omega}_d - {\Omega}_c$. $\delta_c$ and $\delta_d$ are selected large enough to guarantee positive definiteness. We generate two blocks of data samples following Gaussian distribution using $N(0,{\Omega}_c^{-1})$ and $N(0,{\Omega}_d^{-1})$. We use these data samples only to approximate the differential GGM to compare to the ground truth ${\Delta}$. For the other data settings(Data-G and Data-E), we have provided details in Section~\ref{sec:simulatedall}.

\subsubsection{Results on Simulation Experiments }

\begin{table*}[htb]
\centering
\scriptsize
\begin{tabular}{p{3cm}|r|r|r|r|r|r}\toprule
\multirow{2}{*}{Method} & \multicolumn{1}{c|}{Data-EG(Time)} &
      \multicolumn{3}{c|}{Data-EG(F1-Score)} &
      \multicolumn{1}{c|}{Data-G(Time)} & \multicolumn{1}{c}{Data-G(F1-Score)} \\
       \cmidrule(r){2-7}
  		
          &  W2($p=246$)      & W1($p=116$)     &   W2($p=246$)   & W3($p=160$)     & W2($p=246$)    & W2($p=246$)     \\\midrule
KDiffNet-EG & 3.270$\pm$0.182 & \bf{0.704}$\pm$0.022  & \bf{0.926}$\pm$0.001 & \bf{0.934}$\pm$0.002 & * &    * \\\hline
KDiffNet-G & 0.006$\pm$0.00 & 0.578$\pm$0.001 & 0.565$\pm$0.00 & 0.576$\pm$0.00 &   0.006$\pm$0.000 &    {\bf 0.860$\pm$0.000}   \\\hline
KDiffNet-E & 0.005$\pm$0.001 & 0.686$\pm$0.024 & 0.918$\pm$0.001 & 0.916$\pm$0.002 & * & *  \\\hline\midrule		
JEEK \citep{wang2017fast}     &  10.476$\pm$0.054   & 0.571$\pm$0.010 & 0.582$\pm$0.001 & 0.582$\pm$0.001 & * & *       \\\hline		
NAK\citep{bu2017integrating}       &   6.520$\pm$0.184 & 0.225$\pm$0.013 & 0.198$\pm$0.011 & 0.203$\pm$0.005 &     *   &    *      \\ \hline \midrule		
SDRE\citep{liu2014direct}     &   28.807$\pm$1.673 & 0.573$\pm$0.11 & 0.568$\pm$0.006 & 0.574$\pm$0.11 &    11.764$\pm$1.23 &   0.318$\pm$0.10     \\\hline 		
{\diffee}\citep{wang2017fastchange}    &   0.005$\pm$0.00   & 0.570$\pm$0.001 & 0.562$\pm$0.00 & 0.570$\pm$0.00 & 0.004$\pm$0.000 & 0.131$\pm$0.131  \\\hline		
JGLFUSED\citep{danaher2013joint}   &   109.15$\pm$13.659  & 0.512$\pm$0.001 & 0.489$\pm$0.001 & 0.504$\pm$0.001 &   112.441$\pm$6.362 &  0.060$\pm$0.00\\\hline
Number of Datasets   & 14    & 14 & 14 &  14 & 14   &  14\\\hline
\bottomrule
\end{tabular}
\caption{Mean Performance(F1-Score)  and Computation Time(seconds) with standard deviation for $10$ random seeds given in parentheses of \methodNameEV, \methodNameE, \methodNameV and baselines for  simulated data. We evaluate over $126$ datasets: $14$ variations in each of the three spatial matrices $W_E$: $p=116$(W1), $p=246$(W2),  and $p=160$(W3)  for the three data settings: Data-EG, Data-E and Data-G. $*$ indicates that the method is not applicable for a data setting.}
\label{tab:summary}
\vspace{-7mm}
\end{table*}

We present a summary of our results (partial) in Table~\ref{tab:summary}: the columns representing two cases of data generation settings (Data-EG and Data-G). 
Table~\ref{tab:summary} uses the $mean$ F1-score (across different settings of $p$, $n_c$, $n_d$, etc.) and the computational time cost to compare methods (rows). We repeat each experiment for $10$ random seeds. We can make several conclusions: 

(1) \textbf{\methodName outperforms baselines that do not consider knowledge.} Clearly, \methodName and its variations achieve the highest F1-score across all the $126$ datasets. SDRE and {\diffee} are differential network estimators but perform poorly indicating that adding additional knowledge improves differential GGM estimation. MLE-based JGLFUSED performs the worst in all cases. 

(2) \textbf{\methodName outperforms the baselines that consider knowledge, especially when group knowledge exists.} 
When under the Data-EG setting, while JEEK and NAK include the extra edge information, they cannot integrate group information and are not designed for differential estimation. This results in lower F1-Score (0.582 and 0.198 for W2) compared to \methodNameEV(0.926 for W2).  The advantage of utilizing both edge and node groups evidence is also indicated by the higher F1-Score of \methodNameEV with respect to \methodNameE and \methodNameV on the Data-EG setting (Top 3 rows in Table~\ref{tab:summary}). On Data-G cases, none of the baselines can model node group evidence. On  average \methodNameV performs  $6.4\times$ better than the baselines for $p=246$ with respect to F1.

(3) \textbf{\methodName achieves reasonable time cost versus the baselines and is scalable to large $p$.} 
Figure~\ref{fig:p_time} shows each method's time cost per $\lambda_n$ for large $p=2000$.
\methodNameEV is faster than JEEK, JGLFUSED and SDRE (Column 1 in Table~\ref{tab:summary}). \methodNameE and \methodNameV are faster than \methodNameEV owing to closed form solutions. On  Data-G dataset and Data-E datasets, our faster closed form solutions are able to achieve more  computational speedup. For example, on datasets using W2 $p=246$, \methodNameE and \methodNameV are on an average $21000\times$ and $7400\times$ faster (Column 5 in Table~\ref{tab:summary}) than the baselines, respectively. 

(4) \textbf{\methodNameV outperforms baselines on Knowledge of the perturbed hub
nodes} In Figure~\ref{fig:perturb}, we consider the scenario when a group of nodes is perturbed in the case condition relative to the control condition. Details for the data generation are in ~\ref{subsec:perturb_data}. \methodNameV can directly take into account the group of perturbed nodes and hence shows the best performance when compared to the baselines.  

(5) \textbf{\methodNameEV outperforms the baselines irrespective of hyperparameter $\lambda_n$ choice:} Besides F1-Score, we also analyze \methodName's performance when varying  hyper-parameter $\lambda_n$ using ROC curves.  \methodName achieves the highest Area under Curve (AUC) in comparison to all other baselines, indicating it is not sensitive to varying hyperparameters. In ~\sref{sec:simulatedall} , we use  three different subsections to present more analysis results for all the $126$ datasets under the three different data simulation settings.

(6) \textbf{\methodNameEV outperforms deep learning based structure learning methods:} In \ref{subsec:gat}, we compare edge recovery of \methodName against state-of-the-art deep learning models that can learn graph structure from data. Table~\ref{tab:gat} and Table~\ref{tab:dl_edge_acc} indicate that in such high dimensional cases, deep models are not able to learn the correct differential structures, as indicated by lower F1 score.

\subsection{Experiment 2: Human Brain Connectivity from fMRI  }
\label{subsec:exp1abide}
Real world scientific datasets present unique challenges and opportunities for structure discovery. While their ground truth graphs are unknown, experimental studies have led to a plethora of disparate external sources of structure evidence. We evaluate \methodName in a real-world downstream classification task on a publicly available resting-state fMRI dataset: ABIDE\citep{di2014autism}.   The ABIDE data aims to understand human brain connectivity and how it reflects neural disorders \citep{van2013wu}. 

{\bf Data Processing:} The data is retrieved from the Preprocessed Connectomes Project \citep{craddock2014preprocessed}. %
ABIDE includes two groups of human subjects: autism and control. After preprocessing with Configurable Pipeline for the Analysis of Connectomes (CPAC) \citep{craddock2013towards} pipeline, $871$ individuals remain ($468$ diagnosed with autism). Signals for the 160 (number of features $p=160$) regions of interest (ROIs) in the often-used Dosenbach Atlas \citep{dosenbach2010prediction} are examined. 

{\bf Sources of Additional Knowledge}: We utilize three types of collated evidence in neuroscience: first, as spatially distant regions are less likely
to be connected in the brain\citep{watts1998collective,vertes2012simple}, we employ $W_E$ derived from the spatial distance  between $160$ brain regions of interest(ROI)  \citep{dosenbach2010prediction}. Further,  scientists have classified two types of groups of brain regions that behave similarly(functionally or connective) from Dosenbach Atlas\citep{dosenbach2010prediction}: (1)  macroscopic brain structures with $40$ unique groups  (G1) and (2)  $6$ higher level node groups having the same functional connectivity(G2).

{\bf Results:}  To evaluate the learnt differential structure in the absence of a ground truth graph, we utilize the non-zero edges from the estimated graph in  downstream classification. The subjects are randomly partitioned into three equal sets: a training set, a validation set, and a test set. Each estimator produces $\hat{\Omega}_{c}- \hat{\Omega}_{d}$ using the training set. Then, the nonzero edges in the difference graph are used for feature selection. Namely, for every edge between ROI x and ROI y, the mean
value of $x\times y$ over time was selected as a feature. These features are fed to a logistic regressor with ridge penalty, which is tuned via cross-validation. Accuracy is reported on the test set.  For all methods, we tune $\lambda_n$ to vary the fraction of zero edges(non-edges) of the inferred graphs from $0.01\times i | i \in \{50,51,52,\dots,70\}$.   We repeat the experiment for $3$ random seeds and report the average test accuracy. Figure~\ref{fig:abide_lambda} compares  \methodNameEV and baselines on ABIDE, using the $y$ axis for classification test accuracy (the higher the better) and the $x$ axis for the computation speed per $\lambda_n$ (negative seconds, the more right the better).  \methodName-EG1, incorporating both edge($W_E$) and (G1) group  knowledge, achieves the highest accuracy of $60.5\%$  for distinguishing the autism vs the control subjects without sacrificing computation speed. We show the learnt differential network in Figure~\ref{fig:abide_viz}.\footnote{While higher accuracy has been reported in the literature, e.g. \citep{niu2020multichannel},  they utilize complicated deep learning architectures designed for classification. Instead we use classification as a linear probe to evaluate the learnt graph.}  %
\begin{figure}
    \centering
    \includegraphics[width=0.5\textwidth]{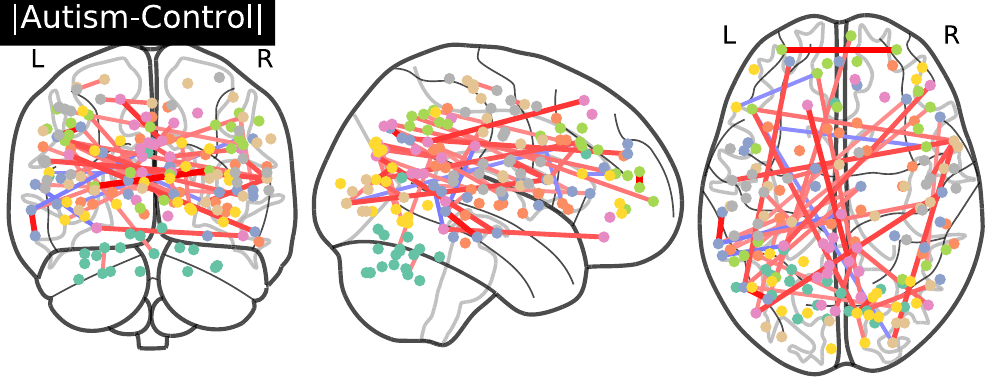}
    \caption{Differential Graph (of three views) between Autism and Control visualized using nilearn package.}
    \label{fig:abide_viz}
\end{figure}

\subsection{Experiment 3: Epigenetic Network from Histone Modifications}
\label{subsec:hmmore}
\label{subsec:exp2more}
\vspace{-1mm}
In this experiment, we evaluate \methodName and baselines for estimating the differential epigenetic network between low and high gene expression. Cellular diversity is attributed to cell type-specific patterns of gene expression, in turn associated with a complex regulation mechanism. Studies have shown that  epigenetic factors(like histone modifications(HMs)), act combinatorially to regulate gene expression \citep{suganuma2008crosstalk,berger2007complex}.  %

{\bf Data Processing: }We consider five core HM marks (H3K4me3, H3K4me1, H3K36me3, H3K9me3, H3K27me3) and three major cell types(K562 Leukemia Cells(E123), GM12878 Lymphoblastoid Cells(E116) and  Psoas Muscle(E100)) with genome-level gene expression profiled in the REMC database \citep{kundaje2015integrative}. 

{\bf Sources of Additional Knowledge: }Signals closer to each other relative to the transcription start site for each gene are more likely to interact  in the gene regulation process. We design a  $W_E$ matrix based on  this genomic distance.   

{\bf Results: } Figure~\ref{fig:hm_main} reports the average test set performance(average across $3$ data splits) for the three cell types.  We plot the test accuracy achieved by \methodName on the $y-$axis, with the best performing baseline on the $x-$axis. \methodName outperforms {\diffee} that does not use $W_E$ as well as JEEK, that can incorporate this information but estimates the two networks separately. Figure~\ref{fig:hmheat} shows a qualitative comparison of the epigenetic networks learnt by \methodName and  {\diffee}.

\section{CONCLUSIONS}
\label{sec:concl}
In this paper, we show that \methodName is flexible in incorporating different kinds of available evidence, leading to improved differential network estimation, without additional computational cost and can improve downstream tasks like classification. We believe the flexibility and scalability provided by \methodName can be beneficial in many real-world tasks. We plan to generalize from Gaussian to semi-parametric distributions or to Ising models next.

\bibliographystyle{plainnat}
\bibliography{gm,BrainGGM}

\clearpage
\appendix

\thispagestyle{empty}
\twocolumn[ \makesupplementtitle ]
\begin{center}
{\Large \bf Part A: Supplementary Materials for Optimization, Error Bounds, Proofs and Theoretical Backgrounds\\}
\end{center}
\section{CONNECTING TO DIFFERENTIAL GGM FORMULATIONS}
\label{sec:rel}

Recent literature includes multiple differential network estimators to go beyond the naive strategy.  They roughly fall into four categories (\sref{sec:intro}). We present one estimator from each group here.

\textbf{Multitask MLE based: JGLFused:}~One study "Joint Graphical Lasso" (JGL)  \citep{danaher2013joint} used multi-task MLE formulation for joint learning of multiple sparse GGMs. JGL can estimate a differential network when using an additional sparsity penalty called the fused norm:
\begin{equation}
\label{eq:fusedGLasso}
\begin{split}
 \argmin\limits_{\Omega_{c}, \Omega_{d} \succ 0, \Delta} & n_c(-\log\text{det}(\Omega_c) + <\Omega_c, \hat{\Sigma}_c>)  \\
 + &n_d(-\log\text{det}(\Omega_d) + <\Omega_d, \hat{\Sigma}_d>)\\
 + &\lambda_2 (||\Omega_c||_1 + ||\Omega_d||_1) + \lambda_n||\Delta||_1
 \end{split}
\end{equation}

Another study \citep{honorio2010multi} used $\ell_1$/$\ell_{\infty}$ regularization via a similar multi-task MLE formulation. Studies in this group jointly learn two GGMs and the difference. However, these multi-task methods do not work if each graph is dense but the change is sparse.

\textbf{Density ratio based: SDRE:}~\cite{liu2014direct} proposed to directly estimate \underline{S}parse differential networks for exponential family by \underline{D}ensity \underline{R}atio \underline{E}stimation: 
\begin{equation}
\label{eq:densityratio}
\argmax\limits_{\Delta} \mathcal{L}_{\text{KLIEP}}(\Delta) - \lambda_n \parallel \Delta \parallel_1 - \lambda_2 \parallel \Delta \parallel_2
\end{equation}
$\mathcal{L}_{\text{KLIEP}}$ minimizes the KL divergence between the true probability density $p_d(x)$ and the estimated %
without explicitly modeling the true $p_c(x)$ and $p_d(x)$. This estimator uses the elastic-net penalty for enforcing sparsity.

\textbf{Constrained $\ell_1$ minimization based: Diff-CLIME}: The study by \cite{zhao2014direct} directly learns $\Delta$ through a constrained optimization formulation.  
\begin{equation}
\label{eq:diffclime}
\begin{split}
&\argmin\limits_{\Delta}||\Delta||_1 \\
\text{Subject to:  } &||\hat{\Sigma}_{c}\Delta\hat{\Sigma}_{d} - (\hat{\Sigma}_{c} - \hat{\Sigma}_{d})||_{\infty} \le \lambda_n
\end{split}
\end{equation}
The optimization reduces to multiple linear programming problems  with a computational complexity of $O(p^8)$. This method doesn't scale to large $p$.

\textbf{Elementary estimator based: {\diffee}:} 
EE based UGM estimator from \cite{yang2014elementary}  proposed the following generic formulation to estimate canonical parameter for an exponential family distribution via EE framework:
\begin{equation}
\begin{split}
 \argmin\limits_{\theta}  ||\theta||_1 ,
\hspace{3mm} \text{Subject to: }  || \theta - \mathcal{B}^*(\hat{\phi}) ||_{\infty} \le \lambda_n
\end{split}
\label{eq:eegm}
\end{equation}
For an exponential family distribution, $\theta$ is its canonical parameter to learn.  
\cite{wang2017fastchange} proposed a so-called {\diffee} for estimating sparse structure changes in high-dimensional GGMs directly:
\begin{equation}
\label{eq:ee-diff}
\begin{split}
&\argmin\limits_{\Delta}||\Delta||_{1,,\text{off}} \\
\text{Subject to:  } &||\Delta - \mathcal{B}^*(\hat{\Sigma}_d, \hat{\Sigma}_c)||_{\infty,\text{off}} \le \lambda_n
\end{split}
\end{equation}
The design of \cite{wang2017fastchange} follows a so-called family of elementary estimators. We explain details of $\mathcal{B}^*(\hat{\Sigma}_d, \hat{\Sigma}_c)$ in \sref{sec:backEE}. 
{\diffee}'s solution is a closed-form entry-wise thresholding operation on $\mathcal{B}^*(\hat{\Sigma}_d, \hat{\Sigma}_c)$  to ensure the desired sparsity structure of its final estimate.
Here $\lambda_n > 0$ is the tuning parameter. Empirically, {\diffee} scales to large $p$ and is faster than SDRE and Diff-CLIME.

\eref{eq:ee-diff} is a special case of \eref{eq:ee-back}, in which $\mathcal{R}(\cdot)$ is the $\ell_1$-norm for enforcing sparsity. The differential network $\Delta$ is the $\theta$ in \eref{eq:ee-back} that we aim to estimate. 
$\mathcal{R}^*(\cdot)$ in \eref{eq:ee-diff}  is the dual norm of $\ell_1$, therefore \eref{eq:ee-diff} used $\ell_\infty$.

 {\diffee}: \eref{eq:ee-diff} is a special case of \eref{eq:eegm}. \eref{eq:eegm} is a special case of \eref{eq:ee-back}.
\section{MORE DETAILS OF RELEVANT STUDIES BEYOND DATA SAMPLES}
\label{sec:appRelated}
\textbf{NAK \citep{bu2017integrating}: }
 For the single task sGGM, one recent study~\citep{bu2017integrating} (following ideas from \cite{shimamura2007weighted}) proposed to use a weighted \underline{N}eighborhood selection formulation to integrate edge-level \underline{A}dditional \underline{K}nowledge (NAK) as: $\hat{\beta}^j = \argmin\limits_{\beta,\beta_j = 0} \frac{1}{2}||\bX^j - \bX\beta||_2^2 + ||\mathbf{r}_j \circ \beta||_1$. 
 Here $\hat{\beta}^j$ is the $j$-th column of a single sGGM $\hat{\Omega}$. Specifically, $\hat{\beta}^j_k = 0$ if and only if $\hat{\Omega}_{k,j} = 0$.  $\mathbf{r}_j$ represents a weight vector designed using available  extra knowledge for estimating a brain connectivity network from samples $\bX$ drawn from a single condition. The NAK formulation can be solved by a classic Lasso solver like glmnet.  %

\textbf{JEEK\citep{wang2017fast}:}  Two related studies, JEEK\citep{wang2017fast} and W-SIMULE\citep{singh2017constrained} incorporate edge-level extra knowledge in the joint discovery of
$K$ heterogeneous graphs. In both these studies, each sGGM corresponding to a condition $i$ is assumed to be composed of a task specific sGGM component $\Omega^{(i)}_I$ and a shared component $\Omega_S$ across all conditions, i.e., $\Omega^{(i)} = \Omega^{(i)}_I + \Omega_S$. The minimization objective of W-SIMULE is as follows: objective:
\begin{align}
\label{eq:wsimule}
\argmin\limits_{\Omega^{(i)}_I,\Omega_S}&\sum\limits_i ||W\circ\Omega^{(i)}_I||_1+ \epsilon K||W\circ\Omega_S||_1  \\
\text{subject to:} \; & ||\Sigma^{(i)}(\Omega^{(i)}_I + \Omega_S) - I||_{\infty} \le \lambda_{n}, \; i =
1,\dots,K \nonumber
\end{align}

W-SIMULE is very slow when $p > 200$ due to the expensive computation cost $O(K^4p^5)$. In comparison, JEEK is an EE-based  optimization formulation: 
\begin{equation}
\begin{split}
    \label{eq:JEEK}
     \argmin\limits_{\Omega^{tot}_I, \Omega^{tot}_S} & ||W^{tot}_I \circ \Omega^{tot}_I||_1 + ||W^{tot}_S\circ \Omega^{tot}_S|| \\
    \text{ subject to: } & ||\dfrac{1}{W^{tot}_I} \circ (\Omega^{tot} - B^*(\hat{\phi}))) ||_{\infty} \le \lambda_n\\
    & ||\dfrac{1}{W^{tot}_S} \circ (\Omega^{tot} - B^*(\hat{\phi})) ||_{\infty} \le \lambda_n \\
    & \Omega^{tot} = \Omega^{tot}_S + \Omega^{tot}_I 
\end{split}
\end{equation}
Here, $\Omega_I^{tot} = (\Omega_I^{(1)},\Omega_I^{(2)},\dots,\Omega_I^{(K)})$ and $\Omega_S^{tot}=(\Omega_S,\Omega_S,\dots,\Omega_S)$. The edge knowledge of the task-specific graph is represented as weight matrix $\{W^{(i)}\}$ and $W_S$ for the shared network. JEEK differs from W-SIMULE in its constraint formulation, that in turn makes its
optimization much faster and scalable than WSIMULE. 
In our experiments, we use JEEK as our baseline.

\textbf{Drawbacks:}~However, none of these studies are flexible to incorporate other types of additional knowledge like node groups or cases where overlapping group and edge knowledge are available for the same target parameter. Further,  these studies are limited by the assumption of sparse single condition graphs. Estimating a sparse difference graph directly is more flexible as it does not rely on this assumption.

\subsection{Related Work on Genetic Network Identification}
\label{sec:geneticback} A genetic interaction network describes biological interactions among genes and provides a systematic understanding of how components communicate and influence each other during cellular signaling and regulatory processes. 
To reverse engineer genetic networks from observed gene expression profiles (like from multiple tissue samples), the bioinformatics literature includes methods from four categories: 
\begin{itemize}
\item (a) Correlation and partial correlation based methods\citep{schafer2005shrinkage,meinshausen2006high}. 
Correlation networks are vulnerable to
false positives. Partial correlation based probabilistic GGMs \citep{dobra2004sparse,allen2013local} successfully avoid this problem with some additional assumptions like Gaussianity. This has been shown to be a reasonable assumption in case of inferring genetic networks from gene expression data. 
\item  (b) Regression based approaches \citep{van2010inferring,haury2012tigress}. These suffer from poor performance in the cases of limited data.
\item  (c) Bayesian Networks\citep{mukherjee2008network,werhli2007reconstructing}: These are probabilistic graphical models representing conditional dependencies in the form of Directed Acyclic Graphs. Bayesian network based methods are limited in how scalable they are, especially when facing high-dimensional genome-wide data sets; 
\item  (d) Information theory based or non-parametric based models. They  measure non linear associations\citep{zhang2012inferring,zhang2015conditional,margolin2006aracne,meyer2007information}.  

\item Some recent methods also focus on Differential Genetic Network Analysis. \citep{yu2011link} proposed to identify differential gene pairs of co-expression
networks. \citep{zhang2016differential} inferred differential correlation-based networks by decomposing them to
global and group-specific components.

\end{itemize}

\section{OPTIMIZATION OF \methodName AND ITS VARIANTS}
\label{sec:moremeth}

In summary, the three added operators are affine mappings and can write as:
$L_e=A_e \Delta_{tot}$, $L_g=A_g \Delta_{tot}$, and $L_{tot}=A_{tot} \Delta_{tot}$, where 
   $A_e=[\boldsymbol{I}_{p \times p}\quad \boldsymbol{0}_{p \times p}]$, $A_g=[\boldsymbol{0}_{p \times p}\quad \boldsymbol{I}_{p \times p}]$ and $A_{tot}=[\boldsymbol{I}_{p \times p}\quad \boldsymbol{I}_{p \times p}]$.

Now we reformulate \eref{eq:\methodName} to the following equivalent and distributed formulation:
\begin{equation}
\begin{split}
\label{eq:distributed}
     &\argmin\limits_{\Delta_{tot}}  F_1(\Delta_{{tot}_1})+F_2(\Delta_{{tot}_2})   + G_1(\Delta_{{tot}_3}) +
    G_2(\Delta_{{tot}_4}) \\
   & \text{subject to: } \Delta_{{tot}_1}=\Delta_{{tot}_2}=\Delta_{{tot}_3}=\Delta_{{tot}_4}=\Delta_{tot}
    \end{split}
\end{equation}
Where $F_1(\cdot)=||W_E\circ(L_e(\cdot))||_1$, $G_1(\cdot)=\mathcal{I}_{||(1 \varoslash W_E)\circ\left(L_{tot}(\cdot) - \mathcal{B}^*(\hat{\Sigma}_{d},\hat{\Sigma}_{c})\right)||_{\infty} \le \lambda_n}$, $F_2(\cdot)=\epsilon||L_g(\cdot)||_{\mathcal{G}_V,2}$ and $G_2(\cdot)=\mathcal{I}_{||L_{tot}(\cdot) -\mathcal{B}^*(\hat{\Sigma}_{d},\hat{\Sigma}_{c}) ||^*_{\mathcal{G}_V,2}\le \epsilon\lambda_n}$. Here $\mathcal{I}_C(\cdot)$ represents the indicator function of a convex set $C$ denoting that $\mathcal{I}_C(x) = 0$ when $x \in C$ and otherwise $\mathcal{I}_C(x) = \infty$.

\begin{algorithm}[th]

   \caption[\methodName]{Parallel Proximal Algorithm for \methodName}
   \label{alg:pp}
   {
\begin{algorithmic}[1]
    \INPUT Two data matrices $\Xb_c$ and $\Xb_d$, The weight matrix $W_E$ and $\mathcal{G_V}$. \\
    Hyperparameters: $\alpha$, $\epsilon$, $v$, $\lambda_n$ and $\gamma$. Learning rate: $0< \rho <2$. Max iteration number $iter$.
    \OUTPUT $\Delta$
    \STATE Compute $\mathcal{B}^*(\hat{\Sigma}_{d},\hat{\Sigma}_{c})$ from $\Xb_d$ and $\Xb_c$  
   \STATE Initialize $ A_e=[\boldsymbol{I}_{p \times p}\quad \boldsymbol{0}_{p \times p}]$, 
    $A_g=[\boldsymbol{0}_{p \times p}\quad \boldsymbol{I}_{p \times p}]$,
     $A_{tot}=[\boldsymbol{I}_{p \times p}\quad \boldsymbol{I}_{p \times p}]$, 
         \STATE Initialize $\Delta_{{tot}_1}$, $\Delta_{{tot}_2}$, $\Delta_{{tot}_3}$, $\Delta_{{tot}_4}$ and $\Delta_{{tot}}=\dfrac{\Delta_{{tot}_1}+\Delta_{{tot}_2}+\Delta_{{tot}_3}+\Delta_{{tot}_4}}{4}$
   \FOR{$i=0$ {\bfseries to} $iter$}
    \STATE$p^i_1 = \text{prox}_{4\gamma F_1}\Delta_{{tot}_1}^i$; $p^i_2 = \text{prox}_{4\gamma F_2}\Delta_{{tot}_2}^i$; $p^i_3 = \text{prox}_{4\gamma G_1}\Delta_{{tot}_3}^i$; $p^i_4 = \text{prox}_{4\gamma G_2}\Delta_{{tot}_4}^i$
    \STATE $p^i = \frac{1}{4}(\sum\limits_{j=1}^4 p_j^i)$
    \FOR{$j =1,2,3,4$}
    \STATE$\Delta_{{tot}_j}^{i+1} = \Delta_{{tot}}^i + \rho(2p^i - \Delta^i_{{tot}} - p_j^i)$
    \ENDFOR
    \STATE$\Delta_{tot}^{i+1} = \Delta^i_{tot} + \rho(p^i - \Delta_{tot}^i)$
   \ENDFOR
   \STATE $\hat{\Delta} = A_{tot}\Delta_{tot}^{iter}$
   \OUTPUT $\hat{\Delta}$
\end{algorithmic}
}
\end{algorithm}

\subsection{Optimization via Proximal Solution}
In this section, we present the detailed optimization procedure for \methodName. 
We assume $\Delta_{tot}=[\Delta_e; \Delta_g]$, where $;$ denotes row wise concatenation. Consider operator $L_{d}(\Delta_{tot})=\Delta_e$ and $L_g(\Delta_{tot})=\Delta_g$, $L_{tot}(\Delta_{tot})=\Delta_e + \Delta_g$.

\begin{equation}
\begin{split}
    \label{eq:\methodName}
     &\argmin\limits_{\Delta}||W_E\circ(L_e(\Delta_{tot}))||_1 + \epsilon||L_g(\Delta_{tot})||_{\mathcal{G}_V,2}  \\
    &\text{s.t.:  }\\
    &||(1 \varoslash W_E)\\
    &\circ\left(L_{tot}(\Delta_{tot}) - \left([T_v(\hat{\Sigma}_{d})]^{-1} -[T_v(\hat{\Sigma}_{c})]^{-1}\right)\right)||_{\infty} \\
    &\le \lambda_n\\
    &||L_{tot}(\Delta_{tot}) - \left([T_v(\hat{\Sigma}_{d})]^{-1} - [T_v(\hat{\Sigma}_{c})]^{-1}\right)||^*_{\mathcal{G}_V,2}\le \epsilon\lambda_n \\
\end{split}
\end{equation}
This can be rewritten as:
\begin{equation}
\begin{split}
     &\argmin\limits_{\Delta}  F_1(\Delta_{tot_1})+F_2(\Delta_{tot_2})   + G_1(\Delta_{tot_3}) +
    G_2(\Delta_{tot_4}) \\
    &\Delta_{tot}=\Delta_{tot_1}=\Delta_{tot_2}=\Delta_{tot_3}=\Delta_{tot_4}
    \end{split}
\end{equation}
Where:
\begin{equation}
    \begin{split}
        & F_1(\cdot)=||W_E\circ(L_e(\cdot))||_1\\
        &G_1(\cdot)=\mathcal{I}_{||(1 \varoslash W_E)\circ\left(L_{tot}(\cdot) - \left([T_v(\hat{\Sigma}_{d})]^{-1} - [T_v(\hat{\Sigma}_{c})]^{-1}\right)\right)||_{\infty} \le \lambda_n}\\
        &F_2(\cdot)=\epsilon||L_g(\cdot)||_{\mathcal{G}_V,2}\\
        & G_2(\cdot)=i_{||L_{tot}(\cdot) - \left([T_v(\hat{\Sigma}_{d})]^{-1} - [T_v(\hat{\Sigma}_{c})]^{-1}\right)||^*_{\mathcal{G}_V,2}\le \epsilon\lambda_n }
    \end{split}
\end{equation}
Here, $L_e$,$L_g$ and $L_{tot}$ can be written as Affine Mappings. By Lemma in \cite{}, 
\begin{equation}
\begin{split}
 & L_e=A_e \Delta_{tot}\\
    & A_e=[\boldsymbol{I}_{p \times p}\quad \boldsymbol{0}_{p \times p}]\\
    & L_g=A_g \Delta_{tot}\\
    & A_g=[\boldsymbol{0}_{p \times p}\quad \boldsymbol{I}_{p \times p}]\\
    & L_{tot}=A_{tot} \Delta_{tot}\\
    & A_{tot}=[\boldsymbol{I}_{p \times p}\quad \boldsymbol{I}_{p \times p}]
    \end{split}
\end{equation}

if $AA^{T}=\beta I$, and $h(x)=g(Ax)$, 
\begin{equation}
    prox_h(x)=x-\beta A^T(Ax-prox_{\beta^{-1}g}(Ax))
    \label{eq:affineproximal}
\end{equation}
$\beta_g=1$, $\beta_e=1$ and $\beta_{tot}=2$.

\textbf{Solving for each proximal operator}:

\paragraph{A. $F_1(\Delta_{tot})=||W_E\circ(L_e(\Delta_{tot}))||_1\\$}
 $L_e(\Delta_{tot})=A_e\Delta_{tot}=\Delta_e$. 
 \begin{equation}
\begin{split}
 & prox_{\gamma F1}(y)=y-A_{e}^{T}(x-prox_{\gamma f}(x)) \\
     & x=A_ey \\
    \end{split}
      \label{eq:prox1}
 \end{equation}
  Here, $x_{j,k}=\Delta_{e_{j,k}}$.
\begin{equation}
    \begin{split}
\text{prox}_{\gamma f_1}(x) &= \text{prox}_{\gamma ||W\cdot||_1}(x) \\
&= \left \{
\begin{array}{rcl}
x_{j,k} - \gamma w_{j,k}\text{, }x_{j,k}^{(i)} > \gamma w_{j,k}\\
0\text{, }|x_{j,k}^{(i)}| \le \gamma w_{j,k}\\
x_{j,k}^{(i)} + \gamma w_{j,k}\text{, }x_{j,k}^{(i)} < - \gamma w_{j,k}
\end{array}
\right.  
\label{eq:p1}
\end{split}
\end{equation}
 Here $j,k = 1,\dots, p$. This is an entry-wise operator (i.e., the calculation of each entry is only related to itself). This can be written in closed form:
 \begin{equation}
\begin{split}
\text{prox}_{\gamma f_1}(x) &= \max((x_{j,k} - \gamma w_{j,k})  , 0) \\
& + \min(0, (x_{j,k}  + \gamma w_{j,k})) 
\label{gpu:p1}
\end{split}
\end{equation}
 We replace this in \eref{eq:prox1}.

\paragraph{B. $F_2(\Delta_{tot}) = \epsilon||L_g(\Delta_{tot})||_{\mathcal{G}_V,2}$}

 Here, $L_g(\Delta_{tot})=A_g\Delta_{tot}=\Delta_g$. 
 \begin{equation}
\begin{split}
& x=A_gy \\
 & prox_{\gamma F2}(y)=y-A_{g}^{T}(x-prox_{\gamma f_2}(x)) \\
 \end{split}
      \label{eq:prox2}
 \end{equation}

  Here, $x_{j,k}=\Delta_{g_{j,k}}$.
\begin{equation}
\begin{split}
\text{prox}_{\gamma f_2}(x_g) &= \text{prox}_{\gamma ||\cdot||_{\mathcal{G},2}}(x_g)\\ &= \left \{
\begin{array}{rcl}
x_g - \epsilon\gamma\frac{x_g}{||x_g||_2}\text{, }||x_g||_2 > \epsilon \gamma\\
0\text{, }||x_g||_2 \le \epsilon \gamma
\end{array}
\right.  
\label{eq:p2}
\end{split}
\end{equation}
Here $g \in \mathcal{G_\mathcal{V}}$. This is a group entry-wise operator (computing a group of entries is not related to other groups). In closed form:
\begin{equation}
\begin{split}
\text{prox}_{\gamma f_2}(x_g) &= \text{prox}_{\epsilon \gamma ||\cdot||_{\mathcal{G},2}}(x_g)\\
&=  x_g \max((1 - \frac{\epsilon \gamma}{||x_g||_2}), 0)
\label{gpu:p2}
\end{split}
\end{equation}
We replace this is~\eref{eq:prox2}.

\paragraph{C. $G_1$:}
 $G_1(\Delta_{tot})=\mathcal{I}_{||(1 \varoslash W_E)\circ\left(L_{tot}(\Delta_{tot}) -  \left([T_v(\hat{\Sigma}_{d})]^{-1} - [T_v(\hat{\Sigma}_{c})]^{-1}\right)\right)||_{\infty} \le \lambda_n}$
 
Here, $L_{tot}=A_{tot} \Delta_{tot}$ and 
    $A_{tot}=[\boldsymbol{I}_{p \times p}\quad \boldsymbol{I}_{p \times p}]$.
    
     \begin{equation}
\begin{split}
& x=A_{tot}y \\
 & prox_{\gamma G1}(y)=y-2A_{tot}^{T}(x-prox_{2^{-1}\gamma g_1}(x)) \\
 \end{split}
      \label{eq:prox3}
 \end{equation}
\begin{equation}
    \begin{split}
\text{prox}_{\gamma g_1}(x) &= \text{proj}_{||1\varobslash(W_E)\circ(x - a)||_{\infty} \le \lambda_n} \\
&= \left \{
\begin{array}{rcl}
x_{j,k}\text{, }|x_{j,k} - a_{j,k}|\le w_{j,k}\lambda_n\\
a_{j,k}+w_{j,k}\lambda_n\text{, }x_{j,k} > a_{j,k} + w_{j,k}\lambda_n\\
a_{j,k}-w_{j,k}\lambda_n\text{, }x_{j,k} < a_{j,k} - w_{j,k}\lambda_n
\end{array}
\right.
\label{eq:p3}
\end{split}
\end{equation}

  In closed form:
  \begin{equation}
\begin{split}
\text{prox}_{\gamma g_1}(x) &= \text{proj}_{||x - a||_{\infty} \le \lambda_n} \\
&= \min(\max(x_{j,k} - a_{j,k}, -w_{j,k}\lambda_n), w_{j,k}\lambda_n) + a_{j,k}
\label{gpu:p3}
\end{split}
\end{equation}
We replace this in~\eref{eq:prox3}.

\paragraph{D. $G_2(\Delta_{tot}) = \mathcal{I}_{\{ ||L_{tot}(\Delta_{tot}) -  B^*||_{\mathcal{G},2}^* \le {\epsilon \lambda_n}\}}$}
Here, $L_{tot}=A_{tot} \Delta_{tot}$ and 
    $A_{tot}=[\boldsymbol{I}_{p \times p}\quad \boldsymbol{I}_{p \times p}]$.
    
     \begin{equation}
\begin{split}
& x=A_{tot}y \\
 & prox_{\gamma G2}(y)=y-2A_{tot}^{T}(x-prox_{2^{-1}\gamma g_2}(x)) \\
 \end{split}
      \label{eq:prox4}
 \end{equation}
\begin{equation}
    \begin{split}
\text{prox}_{\gamma g_2}(x_g) &= \text{proj}_{||x - a||_{\mathcal{G},2}^* \le \epsilon \lambda_n}\\
 &= \left \{
\begin{array}{rcl}
x_g\text{, } ||x_g - a_g||_2\le \epsilon\lambda_n\\
\epsilon\lambda_n \frac{x_g - a_g}{||x_g-a_g||_2} + a_g\text{, }||x_g - a_g||_2 > \epsilon\lambda_n
\end{array}
\right.
\label{eq:p4}
\end{split}
\end{equation}
This operator is group entry-wise. In closed form:
\begin{equation}
\begin{split}
\text{prox}_{\gamma g_2}(x) &= \text{proj}_{||x - a||_{\mathcal{G},2}^* \le \lambda_n}\\
 &= \min(\frac{\epsilon\lambda_n}{||x_g - a_g||_2}, 1) (x_g - a_g) + a_g 
\label{gpu:p4}
\end{split}
\end{equation}
We replace this in~\eref{eq:prox4}.

\begin{table*}[t]

\caption{\fix{rewrite}The four proximal operators}

\label{tab:prox}

\begin{center}
\begin{tabular}{|c| c | }
\hline
$[\text{prox}_{\gamma f_1}(x)]^{(i)}_{j,k}$      & $\max((x_{j,k} - \gamma w_{j,k})  , 0) + \min(0, (x_{j,k}  + \gamma w_{j,k}))$ \\
 \hline 
$\text{prox}_{\gamma }(x_g)$      & $ x_g \max((1 - \frac{\epsilon\gamma}{||x_g||_2}), 0)$ \\
 \hline 
$[\text{prox}_{\gamma f_3}(x)]^{(i)}_{j,k}$       & $\min(\max(x_{j,k}-a_{j,k},-w_{j,k}\lambda_n), w_{j,k}\lambda_n) + a_{j,k}$  
\\
\hline 
$\text{prox}_{\gamma f_4}(x_g)$        & $\min(\frac{\epsilon\lambda_n}{||x_g - a_g||_2}, 1) (x_g - a_g) + a_g $
\\
\hline
\end{tabular}
\end{center}

\end{table*}

\subsection{Computational Complexity}
\label{sec:Complexity}

Another critical property of recent data generations is how the measured variables grow at an unprecedented scale.  On $p$ variables, there are $O(p^2)$ possible pairwise interactions we aim to learn from samples. For even a moderate $p$, searching for pairwise relationships is computationally expensive. $p$ in popular applications ranges from hundreds (e.g., \#brain regions) to tens of thousands (e.g., \#human genes). This challenge motivates us to make the design of \methodName build upon the more scalable class of elementary estimators.

We optimize \methodName through a proximal algorithm, while \methodNameE and \methodNameV through closed-form solutions.  The resulting computational cost for \methodName is $O(p^3)$, broken down into the following steps: 
\begin{itemize}[noitemsep]
    \item Estimating two covariance matrices: The computational complexity is $O(max(n_c, n_d)p^2)$.
    \item Backward Mapping: The element-wise soft-thresholding operation $[T_v(\cdot)]$ on the estimated covariance matrices, that costs $O(p^2)$. This is followed by matrix inversions
$[T_v(\cdot)]^{-1}$ to get the proxy backward mapping, that cost $O(p^3)$. 
\item Optimization: For \methodName, each operation in the proximal algorithm is group entry wise or entry wise, the resulting computational cost is $O(p^2)$. In addition, the matrix multiplications cost $O(p^3)$. For \methodNameE and \methodNameV versions, the solution is the element-wise soft-thresholding operation $S_{\lambda_n}$, that costs $O(p^2)$.
\end{itemize}

\subsection{Closed-form solutions for Only Edge(\methodNameE) Or Only Node Group Knowledge (\methodNameV)}
\label{subsec:closed}
In cases, where we do not have superposition structures in the differential graph estimation, we can estimate the target $\Delta$ through a closed form solution, making the method scalable to larger $p$. In detail:

\textbf{\methodNameE Only Edge-level Knowledge $W_E$:}~
If additional knowledge is only available in the form of edge weights, the \eref{eq:\methodName} reduces to : 
\begin{equation}
\begin{split}
    \label{eq:onlyW}
     &\argmin\limits_{\Delta}||W_E\circ\Delta||_1 \\
    &\text{subject to:  } \\
    &||(1 \varoslash W_E)\circ\left(\Delta - \left([T_v(\hat{\Sigma}_{d})]^{-1} - [T_v(\hat{\Sigma}_{c})]^{-1}\right)\right)||_{\infty} \le \lambda_n
\end{split}
\end{equation}
This has a closed form solution:
\begin{equation}
\begin{split}
\label{eq:solutionOnlyW}
    &\hat{\Delta} = S_{\lambda_n*W_E}\left(\mathcal{B}^*(\hat{\Sigma}_{d},\hat{\Sigma}_{c})\right)\\
    &[S_{\lambda_{ij}W_{E_{ij}}}(A)]_{ij} = \text{sign}(A_{ij})\max(|A_{ij}|-\lambda_{n}{W}_{E_{i,j}}, 0)
    \end{split}
\end{equation}

\textbf{\methodNameV Only Node Groups Knowledge $G_V$:}~If additional knowledge is only available in the form of groups of vertices $\mathcal{G}_V$, the \eref{eq:\methodName} reduces to : 
\begin{equation}
\begin{split}
    \label{eq:onlyG}
    &\argmin\limits_{\Delta}||\Delta||_{\mathcal{G}_V,2}  \\
    \text{Subject to:  } &||\Delta - \mathcal{B}^*(\hat{\Sigma}_{d},\hat{\Sigma}_{c})||^*_{\mathcal{G}_V,2}\le \lambda_n
\end{split}
\end{equation}

Here, we assume nodes not in any group as individual groups with cardinality$=1$.
The closed form solution is given by:
\begin{equation}
    \label{eq:solutionOnlyG}
    \hat{\Delta} = (S_{\mathcal{G}_V,\lambda_n}(\mathcal{B}^*(\hat{\Sigma}_{d},\hat{\Sigma}_{c})))
\end{equation}

Where $[S_{\mathcal{G},\lambda_n}(u)]_{g} = \max(||u_g||_2 -\lambda_n, 0) \frac{u_g}{||u_g||_2}$ and $\max$ is the element-wise max function.

Algorithm~\ref{alg:closed} shows the detailed steps of the \methodName estimator. Being non-iterative, the closed form solution helps \methodName achieve significant computational advantages. 
\begin{algorithm}[th]
   \caption[\methodName]{\methodNameE and \methodNameV}
   \label{alg:closed}
   {\footnotesize
\begin{algorithmic}[1]
    \INPUT Two data matrices $\Xb_c$ and $\Xb_d$. The weight matrix $W_E$ OR $\mathcal{G_V}$.
    \INPUT Hyper-parameter: $\lambda_n$ and $v$
    \OUTPUT $\Delta$
    \STATE Compute $[T_v(\hat{\Sigma}_{c})]^{-1}$ and $[T_v(\hat{\Sigma}_{d})]^{-1}$ from $\hat{\Sigma}_c$ and $\hat{\Sigma}_d$.
    \STATE Compute $\mathcal{B}^*(\hat{\Sigma}_{d},\hat{\Sigma}_{c})$.
     \STATE Compute \text{$\hat{\Delta}$  \eref{eq:solutionOnlyW}($W_E$ only)/  \eref{eq:solutionOnlyG}($\mathcal{G_V}$ only)}
    \OUTPUT $\hat{\Delta}$
\end{algorithmic}
}
\end{algorithm}

\section{GENERALIZING \methodName}
\label{sec:moregeneral}

\subsection{Generalizing \methodName to multiple $W_E$ and multiple groups $G_V$}
We generalize \methodName to multiple groups and multiple weights. We consider the case of two weight matrices $W_{E1}$ and $W_{E2}$, as well as two groups $\mathcal{G}_{V1}$ and $\mathcal{G}_{V2}$. In detail, we optimize the following objective:
\begin{equation}
\begin{split}
    \label{eq:methodName}
    &\argmin\limits_{\Delta}||W_{E1}\circ\Delta_{e1}||_1 + \epsilon_e||W_{E2}\circ\Delta_{e2}||_1 + \\
    &\epsilon_{g1}||\Delta_{g1}||_{\mathcal{G}_{V1},2} + \epsilon_{g2}||\Delta_{g2}||_{\mathcal{G}_{V2},2}  \\
    &\text{subject to:  } \\
    &||(1 \varoslash W_{e1})\circ\left(\Delta - \left([T_v(\hat{\Sigma}_{d})]^{-1} -  [T_v(\hat{\Sigma}_{c})]^{-1}\right)\right)||_{\infty} \\
    &\le \lambda_n\\
    &||(1 \varoslash W_{e2})\circ\left(\Delta - \left([T_v(\hat{\Sigma}_{d})]^{-1} -  [T_v(\hat{\Sigma}_{c})]^{-1}\right)\right)||_{\infty}\\
    &\le \lambda_n\\
    &||\Delta - \left([T_v(\hat{\Sigma}_{d})]^{-1} - [T_v(\hat{\Sigma}_{c})]^{-1}\right)||^*_{\mathcal{G}_{V1},2}\le \epsilon_1\lambda_n \\
    &||\Delta - \left([T_v(\hat{\Sigma}_{d})]^{-1} - [T_v(\hat{\Sigma}_{c})]^{-1}\right)||^*_{\mathcal{G}_{V2},2}\le \epsilon_2\lambda_n \\
    & \Delta = \Delta_{e1} + \Delta_{e2} + \Delta_{g1} + \Delta_{g2}
\end{split}
\end{equation}
To simplify notations, we add a new notation  $\Delta_{tot}:=[\Delta_{e1};\Delta_{e2}; \Delta_{g1};\Delta_{g2}]$, where $;$ denotes the row wise concatenation. We also add three operator notations  including $L_{e1}(\Delta_{tot})=\Delta_e$,$L_{e2}(\Delta_{tot})=\Delta_{e2}$, $L_g(\Delta_{tot})=\Delta_g$, $L_{g2}(\Delta_{tot})=\Delta_{g2}$ and $L_{tot}(\Delta_{tot})=\Delta_{e1} + \Delta_{e2} + \Delta_{g1} + \Delta_{g2}$.
The added operators are affine mappings:
$L_{e1}=A_{e1} \Delta_{tot}$, $L_{g1}=A_{g1} \Delta_{tot}$, $L_{e2}=A_{e2} \Delta_{tot}$, $L_{g2}=A_{g2} \Delta_{tot}$ and $L_{tot}=A_{tot} \Delta_{tot}$, where 
   $A_{e1}=[\boldsymbol{I}_{p \times p}\quad \boldsymbol{0}_{p \times p}\quad \boldsymbol{0}_{p \times p}\quad \boldsymbol{0}_{p \times p}]$, $A_{e2}=[\boldsymbol{0}_{p \times p}\quad \boldsymbol{I}_{p \times p}\quad \boldsymbol{0}_{p \times p}\quad \boldsymbol{0}_{p \times p}]$, $A_{g1}=[\boldsymbol{0}_{p \times p}\quad \boldsymbol{0}_{p \times p}\quad \boldsymbol{I}_{p \times p}\quad \boldsymbol{0}_{p \times p}]$, $A_{g2}=[\boldsymbol{0}_{p \times p}\quad \boldsymbol{0}_{p \times p}\quad \boldsymbol{0}_{p \times p}\quad \boldsymbol{I}_{p \times p}]$ and $A_{tot}=[\boldsymbol{I}_{p \times p}\quad \boldsymbol{I}_{p \times p}\quad \boldsymbol{I}_{p \times p}\quad \boldsymbol{I}_{p \times p}]$.

Algorithm~\ref{alg:pp_general} summarizes the Parallel Proximal algorithm \cite{combettes2011proximal,yang2014elementary2} we propose for optimizing~\eref{eq:\methodName}.
  More concretely in Algorithm~\ref{alg:pp_general}, we simplify the notations by denoting $\mathcal{B}^*(\hat{\Sigma}_{d},\hat{\Sigma}_{c}) := [T_v(\hat{\Sigma}_{d})]^{-1}- [T_v(\hat{\Sigma}_{c})]^{-1}$, 
and reformulate \eref{eq:\methodName} to the following equivalent and distributed formulation:
\begin{equation}
\begin{split}
     &\argmin\limits_{\Delta_{tot}}  F_1(\Delta_{{tot}_1})+F_2(\Delta_{{tot}_2})   + G_1(\Delta_{{tot}_3}) +
    G_2(\Delta_{{tot}_4})\\
    & +F_3(\Delta_{{tot}_5})+F_4(\Delta_{{tot}_6})   + G_3(\Delta_{{tot}_7}) +
    G_4(\Delta_{{tot}_8}) \\
   & \text{subject to: }\\ 
   &\Delta_{{tot}_1}=\Delta_{{tot}_2}=\Delta_{{tot}_3}=\Delta_{{tot}_4}\\
   & =\Delta_{{tot}_5}=\Delta_{{tot}_6}=\Delta_{{tot}_7}=\Delta_{{tot}_8}=\Delta_{tot}
    \end{split}
\end{equation}

Where $F_1(\cdot)=||W_{E1}\circ(L_{e1}(\cdot))||_1$, $G_1(\cdot)=\mathcal{I}_{||(1 \varoslash W_{E1})\circ\left(L_{tot}(\cdot) - \mathcal{B}^*(\hat{\Sigma}_{d},\hat{\Sigma}_{c})\right)||_{\infty} \le\lambda_n}$, $F_2(\cdot)=\epsilon_1||L_{g1}(\cdot)||_{\mathcal{G}_{V1},2}$,  $G_2(\cdot)=\mathcal{I}_{||L_{tot}(\cdot) -\mathcal{B}^*(\hat{\Sigma}_{d},\hat{\Sigma}_{c}) ||^*_{\mathcal{G}_{V1},2}\le \epsilon_1\lambda_n}$,

$F_3(\cdot)=\epsilon_e||W_{E2}\circ(L_{e2}(\cdot))||_1$, $G_3(\cdot)=\mathcal{I}_{||(1 \varoslash W_{E2})\circ\left(L_{tot}(\cdot) - \mathcal{B}^*(\hat{\Sigma}_{d},\hat{\Sigma}_{c})\right)||_{\infty} \le \epsilon_e\lambda_n}$, $F_4(\cdot)=\epsilon_2||L_{g2}(\cdot)||_{\mathcal{G}_{V2},2}$,  $G_4(\cdot)=\mathcal{I}_{||L_{tot}(\cdot) -\mathcal{B}^*(\hat{\Sigma}_{d},\hat{\Sigma}_{c}) ||^*_{\mathcal{G}_{V2},2}\le \epsilon_2\lambda_n}$. Here $\mathcal{I}_C(\cdot)$ represents the indicator function of a convex set $C$ denoting that $\mathcal{I}_C(x) = 0$ when $x \in C$ and otherwise $\mathcal{I}_C(x) = \infty$. 
 The detailed solution of each proximal operator is summarized in Table~\ref{tab:prox} and \sref{sec:moremeth}.

\begin{algorithm}[th]

   \caption[\methodName]{A Parallel Proximal Algorithm to optimize \methodName}
   \label{alg:pp_general}
   {
\begin{algorithmic}[1]
    \INPUT Two data matrices $\Xb_c$ and $\Xb_d$, The weight matrix $W_{E1},W_{E2}$ and $\mathcal{G}_{V1},\mathcal{G}_{V2}$. \\
    Hyperparameters: $\alpha$, $\epsilon_{e}$,$\epsilon_{1}$,$\epsilon_{2}$, $v$, $\lambda_n$ and $\gamma$. Learning rate: $0< \rho <2$. Max iteration number $iter$.
    \OUTPUT $\Delta$
    \STATE Compute $\mathcal{B}^*(\hat{\Sigma}_{d},\hat{\Sigma}_{c})$ from $\Xb_d$ and $\Xb_c$  
   \STATE Initialize $ A_{e1}=[\boldsymbol{I}_{p \times p}\quad \boldsymbol{0}_{p \times p}\quad \boldsymbol{0}_{p \times p}\quad \boldsymbol{0}_{p \times p}]$, $A_{e2}=[\boldsymbol{0}_{p \times p}\quad \boldsymbol{I}_{p \times p}\quad \boldsymbol{0}_{p \times p}\quad \boldsymbol{0}_{p \times p}]$
    $A_{g1}=[\boldsymbol{0}_{p \times p}\quad \boldsymbol{0}_{p \times p}\quad \boldsymbol{I}_{p \times p}\quad \boldsymbol{0}_{p \times p}]$, $A_{g2}=[\boldsymbol{0}_{p \times p}\quad \boldsymbol{0}_{p \times p}\quad \boldsymbol{0}_{p \times p}\quad \boldsymbol{I}_{p \times p}]$, 
     $A_{tot}=[\boldsymbol{I}_{p \times p}\quad \boldsymbol{I}_{p \times p}\quad \boldsymbol{I}_{p \times p}\quad \boldsymbol{I}_{p \times p}]$, 
         \STATE Initialize $\Delta_{{tot}_k}$ $\forall k \in \{1,\dots,8\}$ 
         \STATE Initialize $\Delta_{{tot}}=\dfrac{\sum_{k=1}^8\Delta_{{tot}_k}}{8}$
   \FOR{$i=0$ {\bfseries to} $iter$}
    \STATE$p^i_1 = \text{prox}_{8\gamma F_1}\Delta_{{tot}_1}^i$; $p^i_2 = \text{prox}_{8\gamma F_2}\Delta_{{tot}_2}^i$; $p^i_3 = \text{prox}_{8\gamma G_1}\Delta_{{tot}_3}^i$; $p^i_4 = \text{prox}_{8\gamma G_2}\Delta_{{tot}_4}^i$, $p^i_5 = \text{prox}_{8\gamma F_3}\Delta_{{tot}_5}^i$; $p^i_6 = \text{prox}_{8\gamma F_4}\Delta_{{tot}_6}^i$; $p^i_7 = \text{prox}_{8\gamma G_3}\Delta_{{tot}_7}^i$; $p^i_8 = \text{prox}_{8\gamma G_4}\Delta_{{tot}_8}^i$
    \STATE $p^i = \frac{1}{8}(\sum\limits_{j=1}^8 p_j^i)$
    \FOR{$j =1,\dots,8$}
    \STATE$\Delta_{{tot}_j}^{i+1} = \Delta_{{tot}}^i + \rho(2p^i - \Delta^i_{{tot}} - p_j^i)$
    \ENDFOR
    \STATE$\Delta_{tot}^{i+1} = \Delta^i_{tot} + \rho(p^i - \Delta_{tot}^i)$
   \ENDFOR
   \STATE $\hat{\Delta} = A_{tot}\Delta_{tot}^{iter}$
   \OUTPUT $\hat{\Delta}$
\end{algorithmic}
}
\end{algorithm}

\subsection{Strategy to Handle Mis-specifications}
\label{sec:0321-misspecify}

While our method can incorporate multiple sources of knowledge, we also address the case where we may have different $W_E$, with possible mis-specification. This refers to the situation when our prior knowledge is not correct for the task. Our downstream pairwise classification evaluation strategy provides a way to deal with possible mis-specified or noisy additional knowledge. When faced with this choice of multiple, potentially mis-specified, additional knowledge, we use a validation strategy. Here, we treat the average accuracy across validation sets as a way to select a good $W_E$. To evaluate the learnt differential structure in the absence of a ground truth graph, we utilize the non-zero edges from the estimated graph in  downstream classification. We tune over $\lambda_n$ and pick the best $\lambda_n$ with highest validation accuracy. Further, we use the validation accuracy obtained across the different available $W_E$ or group knowledge $\mathcal{G}_V$ to direct us towards the source of additional knowledge one that best fits the data. 

In Section~\ref{subsec:misspecify_bound}, we show \methodName's convergence rate under misspecification setting, i.e. when the prior knowledge is misspecified. 
Further, we empirically analyze model misspecification under the setting when we  have prior knowledge only about some edges. Figure~\ref{fig:proportion} compares the performance of \methodName when the complete $W_E$ is not known. We compare KDiffNet to baselines when varying the proportion of known entries in the W matrix. Here W is partly `mis-specified' because only some of the W entries are available to \methodName. This shows our method achieves a consistent better estimation than the baseline DIFFEE. The more entries of W we know, the better is the improvement.

\section{PROOFS ABOUT KEV NORM AND ITS DUAL NORM }
\label{sec:proofnorm}
\subsection{Proof for kEV Norm is a norm}
We reformulate kEV norm as\begin{equation}
    \mathcal{R}(\Delta) = ||W_E \circ \Delta_e||_1 + \epsilon||\Delta_g||_{\mathcal{G}_V,2}
\end{equation} to
\begin{equation}
    \mathcal{R}(\Delta) = \mathcal{R}_1(\Delta)+ \mathcal{R}_2(\Delta); \mathcal{R}_1(\cdot)=||W_E\circ \cdot||_1;  \mathcal{R}_2(\cdot)=\epsilon||\cdot||_{\mathcal{G}_{V,2}}
\end{equation}

\begin{theorem}
kEV Norm is a norm if and only if $\mathcal{R}_1(\cdot)$ and $\mathcal{R}_2(\cdot)$ are norms. 
\end{theorem}
\begin{proof}
By the following Theorem~\ref{theo:wnorm}, $R_1(\cdot)$ is a norm. If $\epsilon>0$, $R_2(\cdot)$ is a norm. Sum of two norms is a norm, hence kEV Norm is a norm. 
\end{proof}
\begin{lemma}
\label{le:1}
For kEV-norm, ${W_E}_{j,k}\ne 0$ equals to ${W_E}_{j,k}> 0$.
\end{lemma}
\begin{proof}
If ${W_E}_{j,k} < 0$, then $|{W_E}_{j,k} \Delta_{j,k}| = |-{W_E}_{j,k} \Delta_{j,k}|$. Notice that $-{W_E}_{j,k} > 0$.
\end{proof}
\begin{theorem}
\label{theo:wnorm}
$\mathcal{R}_1(\cdot)=||W_E\circ \cdot||_1$ is a norm if and only if $\forall 1\ge j,k\le p, {W_E}_{jk}\ne 0$.
\end{theorem}

\begin{proof}
To prove the $\mathcal{R}_1(\cdot)=||W_E\circ \cdot||_1$ is a norm,   we need to prove that $f(x) = ||W\circ x||_1$ is a norm function if $W_{i,j} > 0$.
1. $f(ax)= ||aW \circ x||_1 = |a|||W\circ x||_1 = |a|f(x)$.
2. $f(x +y) =||W\circ(x+y)||_1 = ||W\circ x +W\circ y||_1\le ||W\circ x||_1 + ||W\circ y||_1 = f(x) + f(y)$.
3. $f(x) \ge 0$.
4. If $f(x) = 0$, then $\sum |W_{i,j}x_{i,j}| = 0$. Since $W_{i,j}\neq 0$, $x_{i,j} =0$. Therefore, $x = 0$.
Based on the above, $f(x)$ is a norm function.
Since summation of norm is still a norm function, $\mathcal{R}_1(\cdot)$ is a norm function.
\end{proof}

\subsection{kEV Norm is a decomposable norm}
We show that kEV Norm is a decomposable norm within a certain subspace, with the following structural assumptions of the true parameter $\Delta^*$:\\

\textbf{(EV-Sparsity):} The 'true' parameter of  ${\Delta}^*$ can be decomposed into two clear structures--$\{ {\Delta_e}^*$ and ${\Delta_g}^* \}$. ${\Delta_e}^*$ is exactly sparse with $s_E$ non-zero entries indexed by a support set $S_E$ and ${\Delta_g}^*$ is exactly sparse with $\sqrt{s_G}$ non-zero groups with atleast one entry non-zero indexed by a support set $S_V$. $S_E\bigcap S_V = \emptyset$. All other elements  equal to $0$ (in $(S_E\bigcup S_V)^c$). \\

\begin{definition}(EV-subspace)
\label{def:m}
\begin{equation}
   \mathcal{M}(S_E\bigcup S_V) = \{ \theta_{j} = 0| \forall j \notin S_E\bigcup S_V\} 
\end{equation}
\end{definition}

\begin{theorem}
\label{theo:decomp}
 kEV Norm is a decomposable norm with respect to $\mathcal{M}$ and $\bar{\mathcal{M}}^{\perp}$
\end{theorem}
\begin{proof}

Assume $u\in \mathcal{M}$ and $v \in \bar{\mathcal{M}}^{\perp}$, $\mathcal{R}(u+v)=||W_E\circ (u_e + v_e)||_1+\epsilon||(u_g + v_g)||_{G_V,2} = ||W_E\circ u_e||_1+||W_E\circ v_e||_1 + \epsilon||u_g||_{G_V,2}+\epsilon||v_g||_{G_V,2} = \mathcal{R}(u)+\mathcal{R}(v)$. 
Therefore, kEV-norm is a decomposable norm with respect to the subspace pair $(\mathcal{M},\bar{\mathcal{M}}^{\perp})$.
\end{proof}
\subsection{Proofs of Dual Norms for kEV Norm}
\begin{theorem}
Dual Norm of kEV Norm is $\mathcal{R}^*(u) = \max(||(1 \varoslash W_E)\circ u||_{\infty}, \dfrac{1}{\epsilon}||u||^*_{\mathcal{G}_V,2})$.
\end{theorem}
\begin{proof}
Suppose $\mathcal{R}(\theta) = \sum\limits_{\alpha \in I} c_{\alpha}\mathcal{R}_{\alpha}(\theta_{\alpha})$, where $\sum\limits_{\alpha \in I} \theta_{\alpha} = \theta$. Then the dual norm $\mathcal{R}^{*}(\cdot)$ can be derived by the following equation. 
\begin{equation}
    \begin{split}
        \mathcal{R}^*(u) &= \sup\limits_{\theta}\frac{<\theta,u>}{\mathcal{R(\theta)}}\\
        &= \sup\limits_{\theta_{\alpha}} \frac{\sum\limits_{\alpha}<u,\theta_{\alpha}>}{\sum\limits_{\alpha}c_{\alpha}\mathcal{R}_{\alpha}(\theta_{\alpha})}\\
        &= \sup\limits_{\theta_{\alpha}} \frac{\sum\limits_{\alpha}<u/c_{\alpha},\theta_{\alpha}>}{\sum\limits_{\alpha}\mathcal{R}_{\alpha}(\theta_{\alpha})}\\
        &\le \sup\limits_{\theta_{\alpha}} \frac{\sum\limits_{\alpha}\mathcal{R}_{\alpha}^*(u/c_{\alpha})\mathcal{R}(\theta_{\alpha})}{\sum\limits_{\alpha}\mathcal{R}_{\alpha}(\theta_{\alpha})}\\
        &\le \max\limits_{\alpha \in I}\mathcal{R}_{\alpha}^*(u)/c_{\alpha}.
    \end{split}
\end{equation}
Connecting   $\mathcal{R}_1(\cdot)=||W_E\cdot||_1$ and $\mathcal{R}_2(\cdot)=\epsilon||\cdot||_{\mathcal{G}_V}$. By the following Theorem~\ref{theo:wdual}, $\mathcal{R}_1^*(u) = ||(1 \varoslash W_E)\circ u||_\infty $.  From \cite{negahban2009unified}, for $\mathcal{R}_2(\theta_2)= ||\Delta||_{\mathcal{G}_V,2}$, the dual norm is given by

\begin{eqnarray}
{\|v\|_{\GROUP, \SUPERGVEC^*}} &=& \max_{t = 1, \ldots, \numgroup} \|v\|
_{\gqpar^*_t}
\end{eqnarray}

where $\displaystyle
\frac{1}{\gqpar_t} + \frac{1}{\gqpar^*_t} = 1 $ are dual exponents.
where $\numgroup$ denotes the number of groups. As special cases of this general duality relation, this leads to a block $(\infty,
2)$ norm as the dual.

 Hence, $\mathcal{R}_2^*(u) = ||u||^*_{\mathcal{G}_V,2} $. 
Hence, the dual norm of kEV norm is $\mathcal{R}^*(u) = \max(||(1 \varoslash W_E) \circ     u||_{\infty},\dfrac{||u||^{*}_{\mathcal{G_V},2}}{\epsilon}) $.
\end{proof}

\begin{theorem}
\label{theo:wdual}
The dual norm of $||W_E\circ \cdot ||_1$ is:
\begin{equation}
   \mathcal{R}_1^*(\cdot) = ||(1 \varoslash W_E) \circ     u||_{\infty}
\end{equation}
\end{theorem}
For $\mathcal{R}_1(\cdot)=||W_E\circ||_1$, the dual norm is given by:
 \begin{equation}
     \begin{split}
         &\sup_{||W\circ u||_1 \leq 1} u^T x \\
         &\leq \sup_{||W\circ u||_1 \leq 1}\sum_{k=1}^p|u_k||x_k|\\
         &=\sup_{||W\circ u||_1 \leq 1}\sum_{k=1}^p \dfrac{|u_k||x_k||w_k|}{|w_k|}\\
         &=\sup_{||W\circ u||_1 \leq 1}\sum_{k=1}^p |w_k u_k|\Big|\dfrac{x_k}{w_k}\Big|\\
         &\leq\sup_{||W\circ u||_1 \leq 1}\left(\sum_{k=1}^p |w_k u_k|\right)\max_{k=1,\dots,p}\Big|\dfrac{x_k}{w_k}\Big|\\
         &=\Big|\Big|\dfrac{x}{w}\Big|\Big|_\infty
     \end{split}
 \end{equation}

\section{BACKGROUND OF PROXY BACKWARD MAPPING AND THEOREMS OF  $T_v$  BEING INVERTIBLE}
\label{seca:backward}

One key insight of differential GGM is that the density ratio of two Gaussian distributions is naturally an exponential-family distribution (see proofs in~\sref{sec:backmapM}).  The differential network $\Delta$ is one entry of the canonical parameter for this distribution. The MLE solution of estimating vanilla (i.e. no sparsity and not high-dimensional) graphical model in an exponential family distribution can be expressed as a backward mapping that computes the target model parameters from certain given moments.   When using vanilla MLE to learn the exponential distribution about differential GGM (i.e., estimating canonical parameter), the backward mapping of $\Delta$ can be easily inferred from the two sample covariance matrices using $(\hat{\Sigma}_d^{-1} - \hat{\Sigma}_c^{-1})$ (Section~\ref{sec:backmapM}).
Even though this backward mapping has a simple closed form, it is not well-defined when high-dimensional  because $\hat{\Sigma}_c$ and $\hat{\Sigma}_d$ are rank-deficient (thus not invertible) when $p>n$. Using \eref{eq:eegm} to estimate $\Delta$, Wang et. al. \cite{wang2017fastchange}  proposed the {\diffee} estimator for EE-based differential GGM estimation and used only the sparsity assumption on $\Delta$. This study proposed a proxy backward mapping as $\hat{\theta}_n = [T_v(\hat{\Sigma}_d)]^{-1} - [T_v(\hat{\Sigma}_c)]^{-1}$. Here $[T_v(A)]_{ij}:= \rho_v(A_{ij})$ and $\rho_v(\cdot)$ is chosen as a soft-threshold function.

Essentially the MLE solution of estimating vanilla graphical model in an exponential family distribution can be expressed as a backward mapping that computes the target model parameters from certain given moments. For instance, when learning Gaussian GM with vanilla MLE, the backward mapping is $\hat{\Sigma}^{-1}$ that estimates $\Omega$ from the sample covariance matrix (moment) $\hat{\Sigma}$. However, this backward mapping is normally not well-defined in high-dimensional settings. In the case of GGM,  when given the sample covariance $\hat{\Sigma}$, we cannot just compute the vanilla MLE solution as $[\hat{\Sigma}]^{-1}$ when high-dimensional since $\hat{\Sigma}$ is rank-deficient when $p>n$. Therefore Yang et al. \cite{yang2014elementary} proposed to use carefully constructed proxy backward maps for \eref{eq:eegm} that are both available in closed-form, and well-defined in high-dimensional settings for exponential GM models. For instance,  $[T_v(\hat{\Sigma})]^{-1}$  is the proxy backward mapping \cite{yang2014elementary} used for GGM.

\subsection{Backward mapping for an exponential-family distribution:}
\label{seca:backward}
\label{subs:bm}

The solution of vanilla graphical model MLE can be expressed
as a backward mapping\citep{wainwright2008graphical} for an exponential family distribution. It estimates the model parameters (canonical parameter $\theta$) from certain (sample) moments. We provide detailed explanations about backward mapping of exponential families,  backward mapping for Gaussian special case and backward mapping for differential network of GGM in this section. 

\textbf{Backward mapping:} Essentially the vanilla graphical model MLE can be expressed as a backward mapping that computes the model parameters corresponding to some given moments in an exponential family distribution. For instance, in the case of learning GGM with vanilla MLE, the backward mapping is $\hat{\Sigma}^{-1}$ that estimates $\Omega$ from the sample covariance (moment) $\hat{\Sigma}$. 

~Suppose a random variable $X \in \RR^p$ follows the exponential family distribution:
\begin{equation}
\P(X;\theta) = h(X)\text{exp}\{ <\theta, \phi(\theta)> - A(\theta) \}
\label{exp}
\end{equation}
Where $\theta \in \Theta \subset \RR^d$ is the canonical parameter to be estimated and $\Theta$ denotes the parameter space. $\phi(X)$ denotes the sufficient statistics as a feature mapping function $\phi : \RR^p \to \RR^d$, and $A(\theta)$ is the log-partition function. We then define mean parameters $v$ as the expectation of  $\phi(X)$: $v(\theta) := \E[\phi(X)]$, which can be the first and second moments of the sufficient statistics $\phi(X)$ under the exponential family distribution. The set of all possible moments by the moment polytope:
\begin{equation}
\mathcal{M} = \{ v | \exists p \text{ is a distribution s.t. } \E_p[\phi(X)] = v\}
\end{equation}
Mostly, the graphical model inference involves the task of computing moments $v(\theta) \in \mathcal{M}$ given the canonical parameters $\theta \in \encircle{H}$.
We denote this computing as \textbf{forward mapping} :
\begin{equation}
\mathcal{A} : \encircle{H} \to \mathcal{M} 
\end{equation}

The learning/estimation of graphical models involves the task of the reverse computing of the forward mapping, the so-called \textbf{backward mapping} \cite{wainwright2008graphical}. We denote the interior of $\mathcal{M}$ as $\mathcal{M}^0$. \textbf{backward mapping} is defined as:
\begin{equation}
\mathcal{A}^*: \mathcal{M}^0 \to \encircle{H}
\end{equation}
which does not need to be unique. For the exponential family distribution, 
\begin{equation}
	\label{eq:back}
\mathcal{A}^* : v(\theta) \to \theta = \nabla A^*(v(\theta)).
\end{equation}
 Where $A^*(v(\theta)) = \sup\limits_{\theta \in \encircle{H}} <\theta, v(\theta)> - A(\theta)$.

\subsection{Backward Mapping for Differential GGM}
\label{sec:backmapM}

When the random variables $X_c,X_d \in \RR^p$ follows the Gaussian Distribution $N(\mu_c,\Sigma_c)$ and $N(\mu_d,\Sigma_d)$, their density ratio (defined by ~\cite{liu2014direct}) essentially is a distribution in exponential families: 
\begin{equation}
    \begin{split}
    \label{eq:gaussianDensity}
    r(x,\Delta) &= \frac{p_d(x)}{p_c(x)} \\
    & = \frac{\sqrt{\text{det}(\Sigma_c)}\exp\left(-\frac{1}{2}(x-\mu_d)^T\Sigma_d^{-1}(x-\mu_d)\right)}{\sqrt{\text{det}(\Sigma_d)}\exp\left(-\frac{1}{2}(x-\mu_c)^T\Sigma_c^{-1}(x-\mu_c)\right)}\\
    & = \exp ( -\frac{1}{2}(x-\mu_d)^T\Sigma_d^{-1}(x-\mu_d)\\
    &+\frac{1}{2}(x-\mu_c)^T\Sigma_c^{-1}(x-\mu_c)\\ &-\frac{1}{2} (\log(\text{det}(\Sigma_d)) -  \log(\text{det}(\Sigma_c))))\\
    & = \exp\left( -\frac{1}{2}\Delta x^2 + \mu_{\Delta} x - A(\mu_{\Delta}, \Delta)\right)
    \end{split}
\end{equation}
Here $\Delta= \Sigma_d^{-1} - \Sigma_c^{-1}$ and $\mu_{\Delta}=\Sigma_d^{-1}\mu_d - \Sigma_c^{-1}\mu_c $. \\

The log-partition function 
\begin{equation}
    \begin{split}
    A(\mu_{\Delta}, \Delta) &= \frac{1}{2}\mu_d^T\Sigma_d^{-1}\mu_d - \frac{1}{2}\mu_c^T\Sigma_c^{-1}\mu_c +\\
    & \frac{1}{2}\log(\text{det}(\Sigma_d)) - \frac{1}{2}\log(\text{det}(\Sigma_c)) \\
  \end{split}
\end{equation}

The canonical parameter
\begin{equation}
    \begin{split}
    \theta & = \left(\Sigma_d^{-1}\mu_d - \Sigma_c^{-1}\mu_c,  -\frac{1}{2}(\Sigma_d^{-1} - \Sigma_c^{-1})\right) \\
    & = \left(\Sigma_d^{-1}\mu_d - \Sigma_c^{-1}\mu_c,  -\frac{1}{2}(\Delta)\right)
  \end{split}
\end{equation}

 The sufficient statistics $\phi([X_c,X_d])$ and the log-partition function $A(\theta)$: 
 \begin{equation}
    \begin{split}
   & \phi([X_c,X_d]) = ([X_c,X_d], [X_cX_c^T,X_dX_d^T]) \\
&    A(\theta) = \frac{1}{2}\mu_d^T\Sigma_d^{-1}\mu_d - \frac{1}{2}\mu_c^T\Sigma_c^{-1}\mu_c + \\ & \frac{1}{2}\log(\text{det}(\Sigma_d))- \frac{1}{2}\log(\text{det}(\Sigma_c))
 \end{split}
\label{eq:partDiff}
\end{equation}

And $h(x) = 1$.

Now we can estimate this exponential distribution ($\theta$) through vanilla MLE. By plugging \eref{eq:partDiff} into~ \eref{eq:back},  we get the following backward mapping via the conjugate of the log-partition function:
\begin{equation} 
  \begin{split}
\theta =&\left(\Sigma_d^{-1}\mu_d - \Sigma_c^{-1}\mu_c, -\frac{1}{2}(\Sigma_d^{-1} - \Sigma_c^{-1})\right)\\
= & \mathcal{A}^*(v)= \nabla A^*(v)
\end{split}
\end{equation}
The mean parameter vector $v(\theta)$ includes the moments of the sufficient statistics $\phi()$ under the exponential distribution. 
It can be easily estimated through $\E[([X_c,X_d], [X_cX_c^T,X_dX_d^T])]$. 

Therefore the backward mapping of $\theta$ becomes,  
\begin{equation} 
  \begin{split}
\hat{\theta} = &(((\E_{\theta}[X_dX_d^T]-\E_{\theta}[X_d]\E_{\theta}[X_d]^T)^{-1}\E_{\theta}[X_d]\\
&-(\E_{\theta}[X_cX_c^T]-\E_{\theta}[X_c]\E_{\theta}[X_c]^T)^{-1}\E_{\theta}[X_c]),\\
 &-\frac{1}{2}((\E_{\theta}[X_dX_d^T]-\E_{\theta}[X_d]\E_{\theta}[X_d]^T)^{-1} -\\ &(\E_{\theta}[X_cX_c^T]-\E_{\theta}[X_c]\E_{\theta}[X_c]^T)^{-1})).
\end{split}
\end{equation}

Because the second entry of the canonical parameter $\theta$ is $(\Sigma_d^{-1} - \Sigma_c^{-1})$, we get the backward mapping of $\Delta$ as
\begin{equation} 
  \begin{split}
& ((\E_{\theta}[X_dX_d^T]-\E_{\theta}[X_d]\E_{\theta}[X_d]^T)^{-1} \\
-&(\E_{\theta}[X_cX_c^T]-\E_{\theta}[X_c]\E_{\theta}[X_c]^T)^{-1}) \\
= & \hat{\Sigma}_d^{-1} - \hat{\Sigma}_c^{-1}
\end{split}
\label{eq:vanillaDiff}
\end{equation}
This can be easily inferred from two sample covariance matrices $\hat{\Sigma}_d$ and $\hat{\Sigma}_c$ (Att: when under low-dimensional settings).

\subsection{Theorems of Proxy Backward Mapping $T_v$  Being Invertible}
\label{seca:backwardI}

Based on \cite{yang2014elementary} for any matrix A, the element wise operator $T_v$ is defined as:
\[
[T_v(A)]_{ij}=
\begin{cases}
 A_{ii} + v  & if\ i=j \\
 sign(A_{ij})(|A_{ij}|-v)& otherwise, i \neq j
\end{cases}
\]

Suppose we apply this operator $T_v$ to the sample covariance matrix $\dfrac{X^{T}X}{n}$ to obtain $T_v(\dfrac{X^{T}X}{n})$. Then, $T_v(\dfrac{X^{T}X}{n})$ under high dimensional settings will be invertible with high probability, under the following conditions:\\
\textbf{Condition-1} ($\Sigma$-Gaussian ensemble) Each row of the design matrix $X \in \RR^{n \times p}$ is i.i.id sampled from $N(0,\Sigma)$.\\
\textbf{Condition-2} The covariance $\Sigma$ of the $\Sigma$-Gaussian ensemble is strictly diagonally dominant: for all row i, $\delta_i := \Sigma_{ii}-\Sigma_{j \neq i} \geq \delta_{min} > 0 $ where $\delta_{min}$ is a large enough constant so that $||\Sigma||\infty \leq \dfrac{1}{\delta_{min}}$. 

This assumption guarantees that the matrix $T_v(\dfrac{X^{T}X}{n})$ is invertible, and its induced $\ell_{\infty}$ norm is well bounded. 
Then the following theorem holds:\\
\begin{theorem}
Suppose Condition-1 and Condition-2 hold. Then for any $v \geq 8(max_i \Sigma_{ii})\sqrt(\dfrac{10\tau\log p'}{n})$, the matrix $T_v(\dfrac{X^{T}X}{n})$ is invertible with probability at least $1-4/{p'}^{\tau-2}$ for $p' := max\{n,p\}$ and any constant $\tau > 2$.
\end{theorem}

\subsection{Useful lemma(s) of Error Bounds on Proxy Backward Mapping $T_v$}
\label{sec:proofbm}

\begin{lemma}
\label{le:1}
(Theorem 1 of~\cite{rothman2009generalized}). Let $\delta$ be $\max_{ij}|[\frac{X^TX}{n}]_{ij}-\Sigma_{ij}|$. Suppose that $\nu > 2\delta$. Then, under the conditions (C-Sparse$\Sigma$), and as $\rho_v(\cdot)$ is a soft-threshold function, we can deterministically guarantee that the spectral norm of error is bounded as follows:

\begin{equation}
	||| T_v(\hat{\Sigma}) - \Sigma |||_\infty \leq 5\nu^{1-q}c_0(p)+3\nu^{-q}c_0(p)\delta
	\label{eq:proof2_18}
\end{equation}
\end{lemma}

\begin{lemma}
\label{le:2}
(Lemma 1 of~\cite{ravikumar2011high}). Let $\mathcal{A}$ be the event that

\begin{equation}
	|| \frac{X^TX}{n} - \Sigma ||_\infty \leq  8(\max_i \Sigma_{ii})\sqrt{\frac{10\tau \log p'}{n}}
\end{equation}

where $p' := \max(n,p)$ and $\tau$ is any constant greater than 2. Suppose that the design matrix X is i.i.d. sampled from $\Sigma$-Gaussian ensemble with $n \geq 40\max_i\Sigma_{ii}$. Then, the probability of event $\mathcal{A}$ occurring is at least $1-4/p'^{\tau-2}$.

\end{lemma}

\section{THEORETICAL ANALYSIS OF ERROR BOUNDS}
\label{sec:theoryMore}

\subsection{Background: Error bounds of Elementary Estimators}

\methodName formulations are  special cases of the following generic formulation for the elementary estimator. 
\begin{equation}
\label{eq:ee}
  \begin{split}
    &\argmin\limits_{\theta} \mathcal{R}(\theta)\\
    &\text{subject to:} \mathcal{R}^*(\theta -\hat{\theta}_n) \le \lambda_n 
    \end{split}
\end{equation}
Where $\mathcal{R}^*(\cdot)$ is the dual norm of $\mathcal{R}(\cdot)$,  
\begin{equation}
\mathcal{R}^*(v) := \sup\limits_{u \ne 0}\frac{<u,v>}{\mathcal{R}(u)} = \sup\limits_{\mathcal{R}(u) \le 1}<u,v>.
\end{equation}

Following the unified framework \cite{negahban2009unified}, we first decompose the parameter space into a subspace pair$(\mathcal{M},\bar{\mathcal{M}}^{\perp})$, where $\bar{\mathcal{M}}$ is the closure of $\mathcal{M}$. Here $\bar{\mathcal{M}}^{\perp}:= \{ v \in \RR^p | <u,v> = 0, \forall u \in \bar{\mathcal{M}} \}$.
 $\mathcal{M}$ is the \textbf{model subspace} that typically has a much lower dimension than the original high-dimensional space. $\bar{\mathcal{M}}^{\perp}$ is the \textbf{perturbation subspace} of parameters. For further proofs, we assume the regularization function in ~\eref{eq:ee} is \textbf{decomposable} w.r.t the subspace pair $(\mathcal{M},\bar{\mathcal{M}}^{\perp})$.

\textbf{(C1)} $\mathcal{R}(u+v) = \mathcal{R}(u) + \mathcal{R}(v)$, $\forall u \in \mathcal{M}, \forall v \in \bar{\mathcal{M}}^{\perp}$. 

\cite{negahban2009unified} showed that most regularization norms are decomposable corresponding to a certain subspace pair.
\begin{definition}
\label{def:psi}
\textbf{Subspace Compatibility Constant} \\
Subspace compatibility constant is defined as $\Psi(\mathcal{M},|\cdot|):= \sup\limits_{u \in \mathcal{M}\backslash\{ 0 \}} \frac{\mathcal{R}(u)}{|u|}$ which captures the relative value between the error norm $|\cdot|$ and the regularization function $\mathcal{R}(\cdot)$. 
\end{definition}

For simplicity, we assume there exists a true parameter $\theta^*$ which has the exact structure w.r.t a certain subspace pair. Concretely: 

\textbf{(C2)} $\exists$ a subspace pair $(\mathcal{M},\bar{\mathcal{M}}^{\perp})$ such that the true parameter satisfies $\text{proj}_{\mathcal{M}^{\perp}}(\theta^*) = 0$

Then we have the following theorem.
\begin{theorem}
\label{theo:2}
    Suppose the regularization function in ~\eref{eq:ee} satisfies condition \textbf{(C1)}, the true parameter of ~\eref{eq:ee} satisfies condition \textbf{(C2)}, and $\lambda_n$ satisfies that $\lambda_n \ge \mathcal{R}^*(\hat{\theta}_n - \theta^*)$. Then, the optimal solution $\hat{\theta}$ of ~\eref{eq:ee} satisfies:
    \begin{equation}
        \mathcal{R^*}(\hat{\theta} - \theta^*)\le 2 \lambda_n
    \end{equation}
    \begin{equation}
    \label{eq:theo2:1}
        ||\hat{\theta} - \theta^*||_2 \le 4\lambda_n\Psi(\bar{\mathcal{M}})
    \end{equation}
    \begin{equation}
    \label{eq:theo2:2}
        \mathcal{R}(\hat{\theta} - \theta^*) \le 8\lambda_n\Psi(\bar{\mathcal{M}})^2
    \end{equation}
    
\end{theorem}

\begin{proof}
Let $\delta := \hat{\theta} - \theta^*$ be the error vector that we are interested in.

\begin{equation}
\begin{split}
	\mathcal{R}^*(\hat{\theta}-\theta^*) = \mathcal{R}^*(\hat{\theta}-\hat{\theta}_n+\hat{\theta}_n-\theta^*) \\ \leq \mathcal{R}^*(\hat{\theta}_n-\hat{\theta})+\mathcal{R}^*(\hat{\theta}_n-\theta^*)\leq 2\lambda_n
\end{split}
\end{equation}

By the fact that $\theta^*_{\mathcal{M}^\perp}=0$, and the decomposability of $\mathcal{R}$ with respect to $(\mathcal{M},\mathcal{\bar{M}}^\perp)$

\begin{equation}
\begin{split}	
& \mathcal{R}(\theta^*) \\
& = \mathcal{R}(\theta^*) + \mathcal{R}[\Pi_{\bar{\mathcal{M}}^\perp}(\delta)]- \mathcal{R}[\Pi_{\bar{\mathcal{M}}^\perp}(\delta)] \\
& = \mathcal{R}[\theta^*+\Pi_{\bar{\mathcal{M}}^\perp}(\delta)] - \mathcal{R}[\Pi_{\bar{\mathcal{M}}^\perp}(\delta)] \\
& \leq \mathcal{R}[\theta^* +\Pi_{\bar{\mathcal{M}}^\perp}(\delta) +\Pi_{\bar{\mathcal{M}}}(\delta)] + \mathcal{R}[\Pi_{\bar{\mathcal{M}}}(\delta)] \\ 
&-\mathcal{R}[\Pi_{\bar{\mathcal{M}}^\perp}(\delta)] \\
& = \mathcal{R}[\theta^* + \delta] + \mathcal{R}[\Pi_{\bar{\mathcal{M}}}(\delta)] -\mathcal{R}[\Pi_{\bar{\mathcal{M}}^\perp}(\delta)] 
\end{split}
\label{eq:proof18}
\end{equation}

Here, the inequality holds by the triangle inequality of norm. Since \eref{eq:ee} minimizes $\mathcal{R}(\hat{\theta})$, we have $\mathcal{R}(\theta^*+\Delta) = \mathcal{R}(\hat{\theta}) \leq \mathcal{R}(\theta^*)$. Combining this inequality with \eref{eq:proof18}, we have:

\begin{equation}
	\mathcal{R}[\Pi_{\bar{\mathcal{M}}^\perp}(\delta)] \leq \mathcal{R}[\Pi_{\bar{\mathcal{M}}}(\delta)]
	\label{eq:proof19}
\end{equation}

Moreover, by Hölder's inequality and the decomposability of $\mathcal{R}(\cdot)$, we have:

\begin{equation}
\begin{split}
	& ||\Delta||^2_2 = \langle \delta,\delta \rangle \leq \mathcal{R}^*(\delta)\mathcal{R}(\delta) \leq 2\lambda_n\mathcal{R}(\delta)\\
	& = 2\lambda_n[\mathcal{R}(\Pi_{\bar{\mathcal{M}}}(\delta)) + \mathcal{R}(\Pi_{\bar{\mathcal{M}}^\perp}(\delta))] \leq 4\lambda_n \mathcal{R}(\Pi_{\bar{\mathcal{M}}}(\delta)) \\
	& \leq 4\lambda_n\Psi(\bar{\mathcal{M}})||\Pi_{\bar{\mathcal{M}}}(\delta)||_2
\end{split}
\label{eq:proof20}
\end{equation}{}

where $\Psi(\bar{\mathcal{M}})$ is a simple notation for $\Psi(\bar{\mathcal{M}},||\cdot||_2)$.

Since the projection operator is defined in terms of $||\cdot||_2$ norm, it is non-expansive: 
$|| \Pi_{\bar{\mathcal{M}}}(\Delta)||_2 \leq || \Delta ||_2$. Therefore, by \eref{eq:proof20}, we have:

\begin{equation}
|| \Pi_{\bar{\mathcal{M}}}(\delta)||_2 \leq 4\lambda_n\Psi(\bar{\mathcal{M}}),
\label{eq:proof21}
\end{equation}

and plugging it back to \eref{eq:proof20} yields the error bound 
\eref{eq:theo2:1}.

Finally, \eref{eq:theo2:2} is straightforward from \eref{eq:proof19} and \eref{eq:proof21}.

\begin{equation}
\begin{split}
	& \mathcal{R}(\delta) \leq 2 \mathcal{R}(\Pi_{\bar{\mathcal{M}}}(\delta))\\
	& \leq 2\Psi(\bar{\mathcal{M}})||\Pi_{\bar{\mathcal{M}}}(\delta) ||_2 \leq 8\lambda_n\Psi(\bar{\mathcal{M}})^2.
\end{split}
\end{equation}

\end{proof}

\subsection{Error Bounds of \methodName}
Theorem~\ref{theo:2}, provides the error bounds via $\lambda_n$ with respect to three different metrics. In the following, we focus on one of the metrics, Frobenius Norm to evaluate the convergence rate of our \methodName estimator. 
\subsubsection{Error Bounds of \methodName through $\lambda_n$ and $\epsilon$}

\begin{theorem}
\label{theo:4}
 Assuming the true parameter $\Delta^*$ satisfies the conditions \textbf{(C1)(C2)} and $\lambda_n \geq \mathcal{R}^*(\hat{\Delta} - \Delta^*)$, then the optimal point $\hat{\Delta}$ has the following error bounds:
\end{theorem}

\begin{equation}
||\hat{\Delta} - \Delta^*||_F \le (4\max(\sqrt{s_E},\epsilon\sqrt{s_G})\lambda_n 
\end{equation}

Proof:\\

\methodName uses $\mathcal{R}(\cdot) = ||W_E \circ \cdot||_1 + \epsilon ||\cdot||_{\mathcal{G},2}$ because it is a superposition of two norms: $\mathcal{R}_1 = ||W_E\circ||_1$ and $\mathcal{R}_2 = \epsilon||\cdot||_{\mathcal{G},2}$. Based on the results in\cite{negahban2009unified}, $\Psi(\bar{\mathcal{M}}_1) = \sqrt{s_E}$. 

Assuming ground truth $W_E^*$, we assume the model space $\mathcal{M}(S)$, where for set of edges $S=\{i,j | \Delta_{(i,j)}=0\}$, and $n(S)=s_E$,($s$ non zero entries),then  without loss of generality, setting $W_{S} > 1$, indicating $\psi(M)=\sqrt{s_E}$.
Similarly, from \cite{negahban2009unified}, $\Psi(\bar{\mathcal{M}}_2) = \sqrt{s_\mathcal{G}}$, where $s$ is the number of nonzero entries in $\Delta$ and $s_\mathcal{G}$ is the number of groups in which there exists at least one nonzero entry. Therefore,   $\Psi(\bar{\mathcal{M}}) = \max(\sqrt{s_E}),\epsilon\sqrt{s_G})$. Hence,Using this in Equation~\eref{eq:theo2:1}, $||\hat{\Delta} - \Delta^*||_F \le 4(\max(\sqrt{s_E}),\epsilon\sqrt{s_G})\lambda_n$.

\subsubsection{Proof of~\coref{cor:1}-Derivation of the  \methodName error bounds}
\label{proof:diffee}
\label{subsec:theo4proof}

To derive the convergence rate for \methodName,  we introduce the following two sufficient conditions on the $\Sigma_c$ and $\Sigma_d$, to show that the proxy backward mapping $\hat{\theta}_n =B^*(\hat{\phi})= [T_v(\hat{\Sigma}_d)]^{-1} - [T_v(\hat{\Sigma}_d)]^{-1}$ is well-defined\cite{wang2017fastchange}: \\
\textbf{(C-MinInf$-\Sigma$):} The true $\Omega_c^*$ and $\Omega_d^*$ of \eref{def:diffNet} have bounded induced operator norm  i.e., $|||{\Omega_c}^*|||_{\infty} := \sup\limits_{w \ne 0 \in \R^p} \frac{||{\Omega_c}^*w||_{\infty}}{||w||_{\infty}} \le W_{E_{min}}^{c*}\kappa_1 $ and $|||{\Omega_d}^*|||_{\infty} := \sup\limits_{w \ne 0 \in \R^p} \frac{||{\Omega_d}^*w||_{\infty}}{||w||_{\infty}} \le W_{E_{min}}^{d*}\kappa_1$.
Here, intuitively, $W_{E_{min}}^{c*}$ corresponds to the largest ground truth weight index associated with non zero entries in $\Omega_c^{*}$. For set $S_{nz}=\{(i,j) | \Omega_{c_{ij}}^{*} = 0\}$, $W_{E_{S_{nz}}}> W_{E_{min}}^{c*}$.

\textbf{(C-Sparse-$\Sigma$):} The two true covariance matrices $\Sigma_c^*$ and $\Sigma_d^*$ are ``approximately sparse'' (following \cite{bickel2008covariance}). For some constant $0 \le q < 1$ and $c_0(p)$, $\max\limits_i\sum\limits_{j=1}^p|[\Sigma_{c}^*]_{ij} |^q \le c_0(p) $ and  $\max\limits_i\sum\limits_{j=1}^p|[\Sigma_{d}^*]_{ij}|^q \le c_0(p) $. \footnote{This indicates for some positive constant $d$, $[\Sigma_{c}^*]_{jj} \le d$ and $[\Sigma_{d}^*]_{jj} \le d$ for all diagonal entries. Moreover, if $q = 0$, then this condition reduces to $\Sigma_d^*$ and $\Sigma_c^*$ being sparse.} 
We additionally require $\inf\limits_{w \ne 0 \in \R^p} \frac{||\Sigma_c^*w||_{\infty}}{||w||_{\infty}} \ge \kappa_2$ and $\inf\limits_{w \ne 0 \in \R^p} \frac{||\Sigma_d^*w||_{\infty}}{||w||_{\infty}} \ge \kappa_2$.

 We assume the true parameters $\Omega_c^*$ and $\Omega_d^*$ satisfies \textbf{C-MinInf$\Sigma$} and \textbf{C-Sparse$\Sigma$} conditions.

Using the above theorem and conditions, we have the following corollary for convergence rate of \methodName (Att: the following corollary is the same as the Corollary~ \ref{cor:1} in the main draft. We repeat it here to help readers read the manuscript more easily):

\begin{corollary}
\label{cor:1v}
    In the high-dimensional setting, i.e., $p > \max(n_c,n_d)$, let $v:= a\sqrt{\frac{\log p}{\min(n_c,n_d)}}$. Then for $\lambda_n := \frac{\Gamma\kappa_1 a}{4\kappa_2}\sqrt{\frac{\log p}{\min(n_c,n_d)}}$,  Let $\min(n_c,n_d) > c \log p$, with a probability of at least $1-2C_1\exp(-C_2p\log (p))$, the estimated optimal solution $\hat{\Delta}$ has the following error bound:\\

 \begin{equation}
\begin{split}
    &||\hat{\Delta} - \Delta^*||_F \\
    &\le  \frac{\Gamma a\max((\sqrt{s_E}),\epsilon\sqrt{s_G})}{\kappa_2}\sqrt{\frac{\log p}{\min(n_c,n_d)}}
\end{split}
\end{equation}
where $a$, $c$, $\kappa_1$ and $\kappa_2$ are constants. Here $\Gamma=32\kappa_1\dfrac{\max(W_{E_{\min}}^{c*},W_{E_{\min}}^{d*})}{W_{E_{\min}}}$
\end{corollary}
 \begin{proof}

In the following proof, we first prove $||\Omega_c^* - [T_v(\hat{\Sigma}_c)]^{-1}||_{\infty} \le \lambda_{n_c}$. Here $\lambda_{n_c} = \frac{\Gamma\kappa_1 a}{\kappa_2}\sqrt{\frac{\log p'}{n_c}}$ and $p' = \max(p,n_c)$

The condition (C-Sparse$\Sigma$) and condition (C-MinInf$\Sigma$) also hold for $\Omega_d^*$ and $\Sigma_d^*$. We first start with $|| \Omega_c^* - [T_v(\hat{\Sigma}_c)]^{-1}||_{\infty}$:

\begin{equation}
	\begin{split}
		& || \Omega_{c}^* - [T_{v}(\hat{\Sigma}_{c})]^{-1}||_{\infty} = ||[T_{v}(\hat{\Sigma}_{c})]^{-1}(T_{v}(\hat{\Sigma}_{c})\Omega_{c}^*-I)||_{\infty} \\
		& \leq ||| [T_{v}(\hat{\Sigma}_{c})w]|||_\infty||T_v(\hat{\Sigma}_{c})\Omega_{c}^*-I||_\infty \\
		& = |||[T_v(\hat{\Sigma}_{c})]^{-1}|||_\infty||\Omega_{c}^*(T_v(\hat{\Sigma}_{c})-\Sigma_{c}^*)||_\infty \\
		& \leq |||[T_v(\hat{\Sigma}_{c})]^{-1}|||_\infty|||\Omega_{c}^*|||_\infty||T_v(\hat{\Sigma}_{c})-\Sigma_{c}^*||_\infty.
	\end{split}
	\label{eq:proof2_19}
\end{equation}

We first compute the upper bound of $|||[T_v(\hat{\Sigma}_{c})]^{-1}|||_\infty$. By the selection $v$ in the statement, ~\lref{le:1} and~\lref{le:2} hold with probability at least $1-4/p'^{\tau-2}$. Armed with \eref{eq:proof2_18}, we use the triangle inequality of norm and the condition (C-Sparse$\Sigma$): for any $w$,

\begin{equation}
	\begin{split}
		& || T_v(\hat{\Sigma}_{c})w||_\infty = || T_v(\hat{\Sigma}_{c})w -\Sigma w + \Sigma w||_\infty \\
		& \geq || \Sigma w ||_\infty - || (T_v(\hat{\Sigma}_{c})-\Sigma)w||_\infty \\
		& \geq \kappa_2||w||_\infty - || (T_v(\hat{\Sigma}_{c})-\Sigma)w||_\infty \\
		& \geq (\kappa_2 - || (T_v(\hat{\Sigma}_{c})-\Sigma)w||_\infty ) ||w||_\infty
	\end{split}
\end{equation}

Where the second inequality uses the condition (C-Sparse$\Sigma$). Now, by~\lref{le:1} with the selection of $v$, we have

\begin{equation}
	||| T_v(\hat{\Sigma}_{c}) -\Sigma|||_\infty \leq c_1(\frac{\log p'}{n_{c}})^{(1-q)/2}c_0(p)
\end{equation}

where $c_1$ is a constant related only on $\tau$ and $\max_i\Sigma_{ii}$. Specifically, it is defined as $6.5\times(16(\max_i \Sigma_{ii})\sqrt{10\tau})^{1-q}$. Hence, as long as $n_{c}>(\frac{2c_1c_0(p)}{\kappa_2})^{\frac{2}{1-q}}\log p'$ as stated, so that $||| T_v(\hat{\Sigma}_{c})-\Sigma|||_\infty \leq \frac{\kappa_2}{2}$, we can conclude that $||T_{v}(\hat{\Sigma}_{c})w||_\infty \geq \frac{\kappa_2}{2}||w||_\infty$, which implies $||| [T_v(\hat{\Sigma}_{c})]^{-1}|||_\infty \leq \frac{2}{\kappa_2}$.

The remaining term in \eref{eq:proof2_19} is $||T_v(\hat{\Sigma}_{c})-\Sigma_{c}^*||_\infty$; $|| T_v(\hat{\Sigma}_{c})-\Sigma_{c}^*||_\infty \leq || T_v(\hat{\Sigma}_{c})-\hat{\Sigma}_{c} ||_\infty +||\hat{\Sigma}_{c} - \Sigma_{c}^*||_\infty$. By construction of $T_v(\cdot)$ in (C-Thresh) and by~\lref{le:2}, we can confirm that $||T_v(\hat{\Sigma}_{c}) - \hat{\Sigma}_{c} ||_\infty$ as well as $||\hat{\Sigma}_{c}-\Sigma_{c}^*||_\infty$ can be upper-bounded by $v$.

Similarly, the $[T_v(\hat{\Sigma}_d)]^{-1}$ has the same result.

Finally, 
\begin{align}
& ||(1\varoslash W_E) \circ \left(\Delta^* -\left([T_v(\hat{\Sigma}_d)]^{-1} - [T_v(\hat{\Sigma}_c)]^{-1}\right) \right)||_{\infty} \\
&\le ||(1\varoslash W_E) \circ \left(\Omega_d - [T_v(\hat{\Sigma}_d)]^{-1} \right)||_{\infty}\\ 
&+||(1\varoslash W_E) \circ \left(\Omega_c - [T_v(\hat{\Sigma}_c)]^{-1} \right)||_{\infty}\\ 
&\le \dfrac{1}{W_{E_{\min}}} \left(\frac{4W_{E_{\min}}^{c*}\kappa_1 a}{\kappa_2}\sqrt{\frac{\log p'}{n_c}} + \frac{4W_{E_{\min}}^{d*}\kappa_1 a}{\kappa_2}\sqrt{\frac{\log p'}{n_d}}\right)\\
&\le \dfrac{1}{W_{E_{\min}}} \left(\frac{8 \max(W_{E_{\min}}^{c*},W_{E_{\min}}^{d*})\kappa_1 a}{\kappa_2}\sqrt{\frac{\log p'}{\min(n_c,n_d)}}\right)
\end{align}

We assume    $W_{E_{\min}} > 1$.
By Theorem~\ref{theo:4}, we know if $\lambda_n \ge \mathcal{R}^*(\hat{\Delta} - \Delta^*)$,
\begin{equation*}
||\hat{\Delta} - \Delta^*||_F \le (4\max(\sqrt{s_E},\epsilon\sqrt{s_G})\lambda_n)
\end{equation*}

Suppose $p > \max(n_c,n_d)$ we have that

\begin{equation}
\begin{split}
&||\hat{\Delta} - \Delta^*||_F  \\
&\le \frac{\Gamma a\max((\sqrt{s_E}),\epsilon\sqrt{s_G})}{\kappa_2}\sqrt{\frac{\log p}{\min(n_c,n_d)}}
\end{split}
\end{equation}

Here, $\Gamma=32\kappa_1\dfrac{\max(W_{E_{\min}}^{c*},W_{E_{\min}}^{d*})}{W_{E_{\min}}}$. Note that in the case of {\diffee}, $\Gamma=32\kappa_1{\max(W_{E_{\min}}^{c*},W_{E_{\min}}^{d*})}$. 

By combining all together, we can confirm that the selection of $\lambda_n$ satisfies the requirement of~\rref{theo:4}, which completes the proof.

\end{proof}

\subsection{Error bound under misspecified $W$}
\label{subsec:misspecify_bound}
In preceding subsections, we show the error bound if the weight matrix $W_E$ comply with the true parameters. Here in this subsection, we prove the error bound if weight matrix $W$ is misspecified.

In \methodName, $\mathcal{R}(\cdot) = ||W_E \circ \cdot||_1 + \epsilon ||\cdot||_{\mathcal{G},2}$. Since the true parameters satisfies condition  \textbf{(C2)}, there exists a pair of subspace $(\mathcal{M},\bar{\mathcal{M}}^{\perp})$, such that the true parameter satisfies $\text{proj}_{\mathcal{M}^{\perp}}(\theta^*) = 0$, also $dim(\mathcal{M})= s_E$. For simplicity, we assume $\mathcal{M} =  \bar{\mathcal{M}}$.

\begin{theorem}
\label{theo:mis-scc}

For a general weight $W_E$ whose non-zero entries do not comply with the subspace $M$, the subspace compatibility constant $\Psi(\mathcal{M})$ satisfies:
\begin{equation}
\Psi(\mathcal{M}, \mathcal{R}) \leq \Psi(\mathcal{M}, ||W_E \circ \cdot||_1) + \epsilon \Psi(\mathcal{M},||\cdot||_{\mathcal{G},2}) \leq \sqrt{||W_{E_{sub}}||_2} + \epsilon \sqrt{s_G},
\end{equation}
where $W_{E_{sub}}$ represents a subset of $W_E$ containing its $s_E$ largest values, and $s_G$ is the number of groups.
\end{theorem}

\begin{proof} Based on the definition of subspace compatible constant,
\begin{equation}
\begin{split}
\Psi(\mathcal{M}, \mathcal{R}) &= \sup\limits_{u \in \mathcal{M}\backslash\{ 0 \}} \frac{\mathcal{R}(u)}{||u||_2} =  \sup\limits_{u \in \mathcal{M}\backslash\{ 0 \}} \frac{||W_E \circ u||_1 + \epsilon ||u||_{\mathcal{G},2}}{||u||_2},\\
&\leq   \sup\limits_{u \in \mathcal{M}\backslash\{ 0 \}} \frac{||W_E \circ u||_1}{||u||_2} +  \sup\limits_{u \in \mathcal{M}\backslash\{ 0 \}}\epsilon \frac{||u||_{\mathcal{G},2}}{||u||_2}.
\end{split}
\end{equation}
Considering the first term $\sup\limits_{u \in \mathcal{M}\backslash\{ 0 \}} \frac{||W_E \circ u||_1}{||u||_2}$, only $s_E$ entries of $u$ are non-zero, also with Holder inequality,
\begin{align}
||W_E \circ u||_1 & = ||W_{E_{\neq 0}} \circ u_{\neq 0}||_1 \\
& \leq ||W_{E_{\neq 0}}||_2^{1/2} ||u_{\neq 0}||_2^{1/2} \\
& \leq ||W_{E_{sub}}||_2^{1/2}||u||_2^{1/2}, \end{align}
because $u$ lies on the unit sphere, we have $\sup\limits_{u \in \mathcal{M}\backslash\{ 0 \}} \frac{||W_E \circ u||_1}{||u||_2} = \sqrt{||W_{E_{sub}}||_2}$. From the equation, the first term degenerates to the $\sqrt{s_E}$, if $W_E$ is true. For the second term, based on the results in \cite{negahban2009unified}, $ \sup\limits_{u \in \mathcal{M}\backslash\{ 0 \}} \frac{||u||_{\mathcal{G},2}}{||u||_2} = \sqrt{s_G}$.

Combining two upper bounds, we finish the proof.
\end{proof}

We next can prove the general error bound through $\lambda_n, W_E$ and $\epsilon$.

\begin{theorem}
\label{theo:misbound}
Given a random weight matrix $W_E$,  assuming the true parameter $\Delta^*$ satisfies the conditions \textbf{(C1)(C2)} and $\lambda_n \geq \mathcal{R}^*(\hat{\Delta} - \Delta^*)$, then the optimal point $\hat{\Delta}$ has the following error bounds:
 \begin{equation}
||\hat{\Delta} - \Delta^*||_F \le 4 (\sqrt{||W_{E_{sub}}||_2} + \epsilon \sqrt{s_G})\lambda_n 
\end{equation}
\end{theorem}
\begin{proof}
The conclusion is obvious from~\rref{theo:2} and~\rref{theo:mis-scc}. 
\end{proof}

With~\rref{theo:mis-scc} and \rref{theo:misbound} at hand, we are able to prove the error bound given a misspecified weight matrix $W_E$. Before doing so, following \cite{wang2017fastchange}, we define a variant of \textbf{C-MinInf$-\Sigma$} condition \textbf{C-MinInf$-\Sigma$-V2}, which is not relying on the weight matrix $W_E$.

\textbf{(C-MinInf$-\Sigma$-V2):} The true $\Omega_c^*$ and $\Omega_d^*$ of \eref{def:diffNet} have bounded induced operator norm  i.e., $\exists  \kappa_1$ , such that $|||{\Omega_c}^*|||_{\infty} := \sup\limits_{w \ne 0 \in \R^p} \frac{||{\Omega_c}^*w||_{\infty}}{||w||_{\infty}} \le \kappa_1 $ and $|||{\Omega_d}^*|||_{\infty} := \sup\limits_{w \ne 0 \in \R^p} \frac{||{\Omega_d}^*w||_{\infty}}{||w||_{\infty}} \le \kappa_1$.

Using the above theorem and conditions, we have the following corollary for convergence rate of \methodName given a misspecified weight matrix $W_E$.

\begin{corollary}
\label{cor:mis-1v}
    In the high-dimensional setting, i.e., $p > \max(n_c,n_d)$, let $v:= a\sqrt{\frac{\log p}{\min(n_c,n_d)}}$. Then for $\lambda_n := \frac{\Gamma a}{4\kappa_2}\sqrt{\frac{\log p}{\min(n_c,n_d)}}$,  Let $\min(n_c,n_d) > c \log p$, with a probability of at least $1-2C_1\exp(-C_2p\log (p))$, the estimated optimal solution $\hat{\Delta}$ has the following error bound:\\

 \begin{equation}
\begin{split}
    &||\hat{\Delta} - \Delta^*||_F \\
    &\le  \frac{\Gamma a(\sqrt{||W_{E_{sub}}||_2} + \epsilon \sqrt{s_G})}{\kappa_2}\sqrt{\frac{\log p}{\min(n_c,n_d)}}
\end{split}
\end{equation}
where $a$, $c$, $\kappa_1$ and $\kappa_2$ are constants. Here $\Gamma=\dfrac{32\kappa_1}{W_{E_{\min}}}$.
\end{corollary}

\begin{proof}
The proof is similar to that of~\rref{cor:1v}. 
Notice:
\begin{align}
& ||(1\varoslash W_E) \circ \left(\Delta^* -\left([T_v(\hat{\Sigma}_d)]^{-1} - [T_v(\hat{\Sigma}_c)]^{-1}\right) \right)||_{\infty} \\
&\le ||(1\varoslash W_E) \circ \left(\Omega_d - [T_v(\hat{\Sigma}_d)]^{-1} \right)||_{\infty}\\ 
&+||(1\varoslash W_E) \circ \left(\Omega_c - [T_v(\hat{\Sigma}_c)]^{-1} \right)||_{\infty}\\ 
&\le \dfrac{1}{W_{E_{\min}}} \left(\frac{4\kappa_1 a}{\kappa_2}\sqrt{\frac{\log p'}{n_c}} + \frac{4\kappa_1 a}{\kappa_2}\sqrt{\frac{\log p'}{n_d}}\right)\\
&\le \dfrac{1}{W_{E_{\min}}} \left(\frac{8 \kappa_1 a}{\kappa_2}\sqrt{\frac{\log p'}{\min(n_c,n_d)}}\right)
\end{align}

Thus, if $\lambda$ is set as in the statement, by Theorem~\ref{theo:4}, we have the following error bound:
 \begin{equation}
\begin{split}
    &||\hat{\Delta} - \Delta^*||_F \\
    &\le  \frac{\Gamma a(\sqrt{||W_{E_{sub}}||_2} + \epsilon \sqrt{s_G})}{\kappa_2}\sqrt{\frac{\log p}{\min(n_c,n_d)}}
\end{split}
\end{equation}

\end{proof}
\section{ KDIFFNET-POET: ALTERNATIVE BACKWARD MAPPING VIA POET}
\label{sec:poet}
POET based covariance estimation\cite{fan_statistical_2013} assume each observation $X_i$ follows the following factor model:
\begin{equation}
	X_{i,t} = b_i^T f_t + u_{i,t}, \quad i = 1, \dots, n, t = 1, \cdots, p.
\end{equation}
where $B = (b_1, b_2, \cdots, b_n)\in \mathbb{R}^{n\times p}$ is the loading matrix, $f_t$ are the common factors and $u_t$ is the error term. Then we have:
\begin{equation}
	{\Sigma}_p = B cov(f) B' + \Sigma_U
\end{equation}
POET estimates large covariance matrices in approximate factor models by thresholding principal orthogonal complements.

We use the estimated $\hat{\Sigma}_p$ as the  $\hat{\Sigma}$ in Equation~\ref{eq:backwsigma}. 

\subsection{Useful lemma(s) of POET}
\label{subsec:proofpoet}

We introduce three assumptions:

\textbf{Condition-1} (Bounded assumption) Eigenvalues of the $p \times p$ matrix $n^{-1}B'B$ are bounded away from both zero and infinity as $n \to \infty$.

\textbf{Condition-2} (Strict stationary) (i) $\{u_t,f_t\}_{t\geq1}$ is strictly stationary. In addition, $\mathbf{E}u_{it} = \mathbf{E}(u_{it}f_{jt}) = 0$ for all $i \leq n, j \leq p$ and $t \leq p$. (ii) There exist constants $c_1,c_2 \geq 0$ such that $\lambda_{min}(\Sigma_u) > c_1, ||\Sigma_u|| < c_2$, and $\min_{i, j} var(u_{it}u_{jt}) > c_1$. (iii) There exist $r_1,r_2 >0$ and $b_1,b_2 >0$, such that for any $s>0, i<n$ and $j< n, P(|u_{it}| > s) < \exp(-(s/b_1)r_1 ), P (|f_{jt}| > s) < \exp(-(s/b_2)r_2 )$.

\textbf{Condition-2} (Bounded expectation) There exists $M >0$ such that for all$ i\leq n, t\leq p$ and $s\leq p$, we have (i) $||b||_{max} < M$,
(ii) $\mathbf{E}[n^{-1/2}(u_s'u_t) - \mathbf{E}u_s'u_t)]^4 < M$,
(iii) $\mathbf{E}||n^{-1/2} \sum_{i=1}^n b_iu_{it}||^4 < M$.

Note the POET operator as $P(\hat{\Sigma})$, we can derive the error bound for POET operator. 
\begin{lemma}
\label{le:poet}
when $\{f_t\}$ are all unobservable and the three conditions hold, we have:
\begin{equation}
	||P(\hat{\Sigma})- \Sigma||_\infty = O_p((\dfrac{K^3\sqrt{\log K} + K \sqrt{\log n}}{\sqrt{p}} + \dfrac{K^3}{\sqrt{n}})^{1/2})
\end{equation}
where $K$ is the selected number of the spectrums in POET operator.
\end{lemma}

\begin{proof}

Alternatively, if we apply POET operator, the conclusion remains the same. The skeleton of the proof will follow the exactly the same idea except for one place. In order to satisfy the following inequality:
\begin{equation}
||| T_v(\hat{\Sigma}_{c}) -\Sigma|||_\infty \leq \dfrac{\kappa_2}{2}
\end{equation}
We choose $n_c\geq (\dfrac{k^3\log k  + k^3}{(\kappa_2/2)^2 - k})^2$, since:
\begin{equation}
\begin{split}
||| T_v(\hat{\Sigma}_{c}) -\Sigma|||_{\infty} &\leq (\dfrac{K^3\sqrt{\log K} + K \sqrt{\log n}}{\sqrt{p}} + \dfrac{K^3}{\sqrt{n}})^{1/2}\\
&\leq (\dfrac{K^3\sqrt{\log K} + K \sqrt{\log n} + K^3}{\sqrt{n}})^{1/2}\\
&\leq (\dfrac{K^3\sqrt{\log K} + K \sqrt{ n} + K^3}{\sqrt{n}})^{1/2}
\end{split}
\end{equation}
Plug $n_c\geq (\dfrac{k^3\log k  + k^3}{(\kappa_2/2)^2 - k})^2$ into the inequality, we will get $ ||| T_v(\hat{\Sigma}_{c}) -\Sigma|||_\infty \leq \dfrac{\kappa_2}{2}$.

\end{proof}

\section{BAYESIAN INTERPRETATION}
\label{sec:theoryMore}
We can interpret the additional edge-level knowledge via a Bayesian interpretation.
Essentially we assume the $\{i,j\}$-th entry of $\Delta$ follows a Laplace distribution:
\begin{equation}
P(\Delta_{i,j}|W_{i,j}, \sigma) \sim \dfrac{W_{i,j}}{\sigma}exp(-\dfrac{W_{i,j}\times |\Delta_{i,j}|}{\sigma})
\label{eq:bayesian}
\end{equation}

When $W_{i,j}$ is larger, $P(\Delta_{i,j}|W_{i,j},\sigma)$ tends to concentrate on 0. Similarly, the group evidence corresponds to a scale mixture of normals \citep{kyung2010penalized}.

\newpage

\begin{center}
{\Large \bf Part B: Supplementary Materials for Experimental Setup, Real Data, Simulated Data and More Results \\}
\end{center}

\section{MORE DETAILS ON EXPERIMENTAL SETUP:}
\subsection{Experimental Setup}

\label{sec:expsetMore}
The hyper-parameters in our experiments are  $v$, $\lambda_n$, $\epsilon$ and $\lambda_2$. In detail:
\begin{itemize}[noitemsep,topsep=0pt,parsep=0pt,partopsep=0pt]
    \item To compute the proxy backward mapping in~\eqref{eq:\methodName}, {\diffee}, and JEEK we vary $v$  for soft-thresholding $v$ from the set $\{ 0.001i|i = 1,2,\dots,1000 \}$ (to make $T_v(\Sigma_c)$ and $T_v(\Sigma_d)$  invertible).
    \item $\lambda_n$ is the hyper-parameter in our \methodName formulation. According to our convergence rate analysis in Section~\ref{sec:theory}, $\lambda_n \ge C \sqrt{\frac{\log p}{\min(n_c,n_d)}}$,  we choose 
    $\lambda_n$ from a range of $\{0.01 \times  \sqrt{\frac{\log p}{\min(n_c,n_d)}} \times i| i \in \{ 1,2,3,\dots, 100 \} \}$.  
     For \methodNameV case, we tune over $\lambda_n$ from a range of $\{0.1 \times  \sqrt{\frac{\log p}{\min(n_c,n_d)}} \times i| i \in \{ 1,2,3,\dots, 100 \} \}$. We use the same range to tune $\lambda_1$ for SDRE. Tuning for NAK is done by the package itself.
    \item $\epsilon$: For \methodNameEV experiments, we tune $\epsilon \in \{ 0.0001,0.01,1,100 \} \}$.
    \item $\lambda_2$ controls individual graph's sparsity in JGLFUSED. We choose  $\lambda_1 = 0.0001$ (a very small value) for all experiments to ensure  only the differential network is sparse. 
\end{itemize}

\paragraph{Evaluation Metrics:}
\begin{itemize}
\item{F1-score:}~ We use the edge-level F1-score as a measure of the performance of each method. $\text{F1} = \frac{2\cdot\text{Precision}\cdot\text{Recall}}{\text{Precision} + \text{Recall}}$, where $\text{Precision} = \frac{\text{TP}}{\text{TP} + \text{FP}}$ and $\text{Recall} = \frac{\text{TP}}{\text{TP}+\text{FN}}$. The better method achieves a higher F1-score. We choose the best performing $\lambda_n$ using validation and report the performance on a test dataset. 
\item{Time Cost:}~We use the execution time (measured in seconds or log(seconds)) for a method as a measure of its scalability. The better method uses less time\footnote{The machine that we use for experiments is an Intel Core i7 CPU with a 16 GB memory.}
\end{itemize}

\section{EXPERIMENTAL DETAILS ON REAL DATA FOR BRAIN CONNECTOME RESULTS}
\label{sec:expmorebrain}
\subsection{Additional Details: ABIDE}
\label{subsec:exp1more}
In this experiment, we evaluate \methodName in a real-world downstream classification task on a publicly available resting-state fMRI dataset: ABIDE\citep{di2014autism}.  The ABIDE data aims to understand human brain connectivity and how it reflects neural disorders \cite{van2013wu}. The data is retrieved from the Preprocessed Connectomes Project \citep{craddock2014preprocessed}, where preprocessing is performed using the Configurable Pipeline for the Analysis of Connectomes (CPAC) \citep{craddock2013towards} without global signal correction or band-pass filtering. ABIDE includes two groups of human subjects: autism and control. After preprocessing with this pipeline, $871$ individuals remain ($468$ diagnosed with autism). Signals for the 160 (number of features $p=160$) regions of interest (ROIs) in the often-used Dosenbach Atlas \citep{dosenbach2010prediction} are examined.  We utilize three types of additional knowledge: $W_E$ based on the spatial distance between $160$ brain regions of interest(ROI)  \citep{dosenbach2010prediction} and two types of available node groups from Dosenbach Atlas\citep{dosenbach2010prediction}: one with $40$ unique groups about  macroscopic brain structures  (G1) and the other with $6$ higher level node groups having the same functional connectivity(G2).

  To evaluate the learnt differential structure in the absence of a ground truth graph, we utilize the non-zero edges from the estimated graph in  downstream classification. We tune over $\lambda_n$ and pick the best $\lambda_n$ using validation. The subjects are randomly partitioned into three equal sets: a training set, a validation set, and a test set. Each estimator produces $\hat{\Omega}_{c}- \hat{\Omega}_{d}$ using the training set. Then, the nonzero edges in the difference graph are used for feature selection. Namely, for every edge between ROI x and ROI y, the mean
value of x*y over time was selected as a feature. These features are fed to a logistic regressor with ridge penalty, which is tuned via cross-validation on the validation set. Finally, accuracy is calculated on the test set. We repeat this process for $3$ random seeds. For all methods, we choose $\lambda_n$ to vary the fraction of zero edges(non edges) of the inferred graphs from $0.01\times i | i \in \{50,51,52,\dots,70\}$.   We repeat the experiment for $3$ random seeds and report the average test accuracy. Figure~\ref{fig:abide_lambda} compares  \methodNameEV, \methodNameE, \methodNameV and baselines on ABIDE, using the $y$ axis for classification test accuracy (the higher the better) and the $x$ axis for the computation speed per $\lambda_n$ (negative seconds, the more right the better).  \methodName-EG1, incorporating both edge($W_E$) and (G1) group  knowledge, achieves the highest accuracy of $60.5\%$  for distinguishing the autism versus the control subjects without sacrificing computation speed.

\section{EXPERIMENT DETAILS ON REAL DATA FOR GENETIC NETWORKS RESULTS}
\subsection{Experiment 3: Epigenetic Network Estimation from Histone Modification Signals}
\label{subsec:hmmore}
\label{subsec:exp2more}
\vspace{-1mm}

\paragraph{Data Processing:} We use the cell type specific median expression to threshold the values into upregulated and downregulated genes.   We partition the $19795$ genes equally into train, validation and test set genes.  For each gene, we divide the $10,000$ basepair (bp) DNA region ($\pm 5000$ bp) around the transcription start site (TSS) into bins of length $100$ bp. Each bin includes $100$ bp long adjacent positions flanking the TSS of a gene.  We further pool each of the HM signals into $25$ bins using the max value.Gene expression measurements(RPKM) are available through the REMC database\citep{kundaje2015integrative}.We use the cell type specific median expression to threshold the expression into low and high expression.   We partition the $19795$ genes into 6599 train,  6599 validation and  6597 test set genes.
\paragraph{Prior Knowledge:} Further, to incorporate the prior knowledge that signals spatially closer to each other along the genome are more likely to interact  in the gene regulation process, we use   genomic distance (using relative difference of bin positions)  as $W_E$. Similar to the previous case, we utilize the quadratic features from the estimated differential non-zero edges in  downstream gene expression classification.

 {\bf Qualitative Interpretation}: \methodName can both make use of the spatial prior as well as estimate biologically consistent networks. As
expected, we observe a relationship among promoter and structural histone modification marks (H3K4me3 and H3K36me3).  Similarly, the estimated networks show interactions between promoter mark (H3K4me3) and distal promoter mark (H3K4me1) also reported by \cite{dong2012modeling}. 
\begin{figure}[t]
        \centering
        \includegraphics[width=\linewidth]{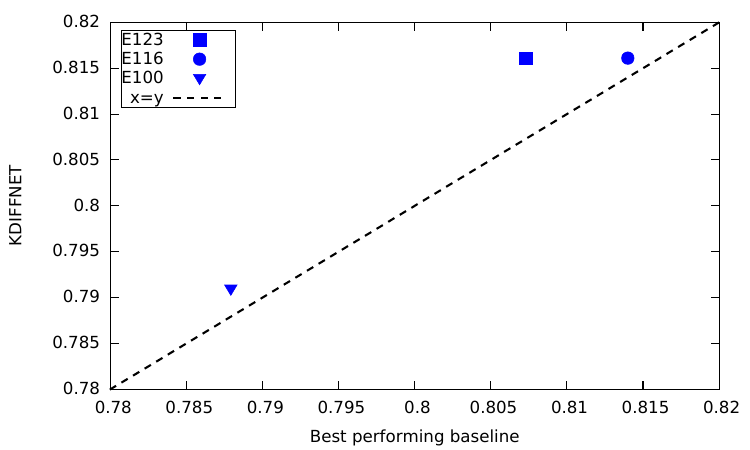}
        \caption{Epigenomic Dataset:  \methodName achieves highest Accuracy (averaged over $3$ splits) in comparison to the best performing baseline. (points above the diagonal $x=y$ line mean \methodName better). We provide detailed results in Table~\ref{tab:hm_auc}. }
        \label{fig:hm}
    \end{figure}

 Figure~\ref{fig:hmheat} shows heatmaps representing epigenetic networks learnt by \methodName is comparison to {\diffee}. As
expected, we observe a relationship among promoter and structural histone modification marks (H3K4me3 and H3K36me3)  Similarly, \cite{dong2012modeling} also reported a combinatorial correlation between promoter mark (H3K4me3) and distal promoter mark (H3K4me1).
\begin{figure*}
    \centering
    \includegraphics[width=0.9\textwidth]{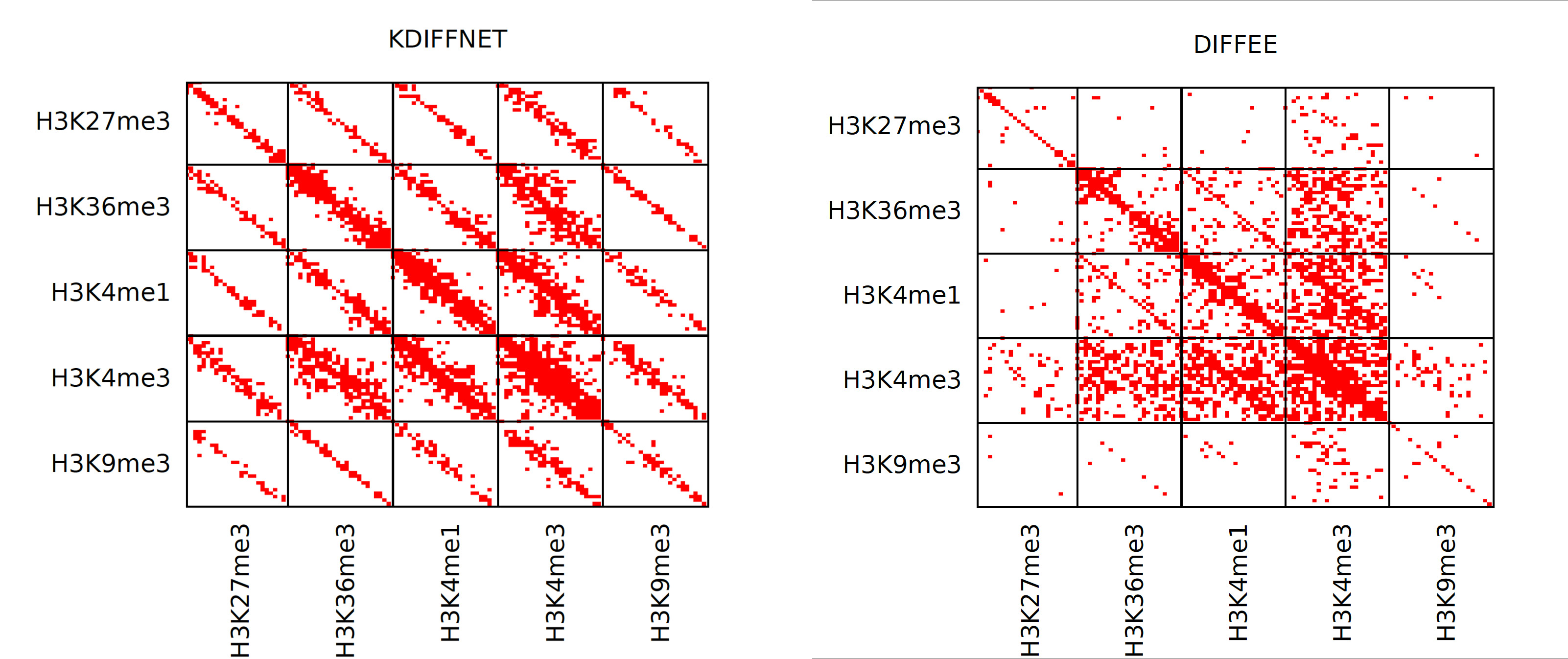}
    \caption{Epigenomic Dataset:  Learnt Epigenetic Network represented as heatmaps: \methodName can discover biologically consistent interactions alongwith incorporating spatial information.}
    \label{fig:hmheat}
    
\end{figure*}

Table~\ref{tab:hm_time} shows the time cost of \methodNameE and baselines of estimating epigenetic network for cell type E123. 

\begin{table}[H]
    \centering
    \tiny
    \begin{tabular}{|c|c|c|c|} \hline
    Method &  E123  & E116 & E100 \\ \hline
    \methodNameE & {\bf 0.8161$\pm$0.044}  & {\bf 0.8161$\pm$0.032} & {\bf 0.7909$\pm$0.0299}\\ \hline
      {\diffee}   &  0.8073$\pm$0.050 &  0.8132$\pm$0.038 & 0.7879$\pm$0.036\\ \hline
        JEEK &  0.8113$\pm$0.042 & 0.8140$\pm$0.036 & 0.7880$\pm$0.034 \\ \hline

    \end{tabular}
    \caption{\methodName achieves highest Test Accuracy (averaged over $3$ splits) and standard deviation for three cell types E123, E116 and E003.}
    \label{tab:hm_auc}
\end{table}

\begin{table}[H]
    \centering
    \tiny
    \begin{tabular}{|c|c|c|c|} \hline
    Method &  E123  & E116 & E100 \\ \hline
    \methodNameE & 0.002($\pm$ 0.000)  & 0.002($\pm$0.001) & 0.001($\pm$ 0.000)\\ \hline
      {\diffee}   &  0.001($\pm$ 0.000) &  0.001($\pm$0.000) & 0.001($\pm$ 0.000)\\ \hline
        JEEK &  3.004($\pm$0.092) & 3.116($\pm$0.0646) & 3.409	($\pm$ 0.227) \\ \hline

    \end{tabular}
    \caption{ Average time cost(seconds) averaged over three data splits and standard deviation  for three cell types E123, E116 and E003.}
    \label{tab:hm_time}
\end{table}

 \subsection{Experiment 4: Differential Genetic Network Identification from Gene Expression using SARS-CoV-2 and related datasets}
\label{subsec:covid-1}
\vspace{-1mm}

 Genes interact with each other for cellular signaling and regulatory processes. Discovering these interactions is important for identifying causal maps of molecular interactions  as well as for using networks as bio markers. \sref{sec:geneticback} reviews data driven literature of extracting genetic networks and differential network identification from gene expression data. 
Complex diseases like the recent pandemic COVID-19 are the result of interactions between viruses and human (host) genes as well as interactions amongst human genes. The invading of the host by the virus perturbs the host's gene expression and leads to rewiring mechanisms, consequentially gaining and losing interactions\citep{dimitrov2004virus}. Understanding and identifying these changes  following viral infection  in the host genetic network is essential for the development of antiviral therapies.

{\bf Human Respiratory Viruses (including SARS-CoV-2) vs Control Dataset:} In this experiment, we use the gene expression dataset measured across $\sim 20k$ from \cite{blanco2020sars}. This dataset measures the transcriptional response from the SARS-CoV-2 virus. Samples from primary human lung epithelium (NHBE) mock treated with SARS-CoV-2, IAV, a IAV that lacks the NS1 protein (IAVdNS1) and treated with human interferon-beta were collected. It also includes samples measured from lung alveolar (A549) cells and RSV or IAV transformed lung-derived Calu-3 cells infected with SARS-CoV-2. Additionally, uninfected human lung biopsies were also derived from two human subjects and a single male COVID-19 deceased patient. 

{\bf Mouse Respiratory Virus vs Control Dataset:} We use another similar dataset regarding viral respiratory infections from \cite{xiong2014genomic}. This dataset includes gene expression measurements collected from  mice with 2 or 4 days post viral infection whose lungs were used for total RNA-Seq. This dataset contains samples infected with multiple respiratory viruses and corresponding mock conditions. We aim to learn the differential graph between the virus infected samples($n_d=32$) and the control mock samples($n_c=88$). Similar to the previous case, we use the STRING database and DAVID databases for edge and group knowledge. We follow the same classification procedure as mentioned in the aforementioned case.  Figure~\ref{fig:resp} shows the obtained classification performance. 
\paragraph{Hyperparameters and evaluation pipeline}: To evaluate the different methods, we use a pairwise linear classification setting. In detail, we use the quadratic features from the estimated differential non-zero edges to classify a virus infected sample from a control sample. For every $(i,j)$ in the estimated graph, we use $x_i*x_j$ as a feature in a linear classification setting with elastic penalty.    For all methods, we validate over $\lambda_n$ values that vary the fraction of zero edges(non-edges) of the inferred graphs from $0.01\times i | i \in \{50,51,52,\dots,98\}$.   These features are fed to a logistic regressor with ridge penalty,
which is trained via cross-validation on the train set. Finally, we report the accuracy on the test set. We use leave-three-out validation and hence, choose the best hyperparameters using the average validation set performance.

Our objective is to learn a differential graph between the virus infected condition and the control condition. For this purpose, we use the virus infected data samples as one class($n_d=38$) and the uninfected mock samples are used as the control samples($n_c=25$).  
Due to the lower number of samples in the dataset, we choose the top ranked $100$ genes with the highest variance in the log of rpkm gene expression counts. In Figure~\ref{fig:genevar}, we show the variance of the log of the gene expression in rpkm. Thus, our final number of features $p=100$.

For group level knowledge, we use group evidence from DAVID\citep{dennis2003david} using their gene functional classification.  To incorporate information regarding known interactions, we use the STRING\citep{szklarczyk2019string} database. To account for the few number of samples, for our backward mapping, we use POET\citep{fan2013large} as an estimation of an invertible covariance matrix.

\paragraph{Results:} Our classification results are shown in Figure~\ref{fig:covid-human}. This pairwise classification strategy also helps to deal with model misspecification issues, as validation performance is an indicator of whether additional knowledge can be useful for estimation. 
This pairwise classification strategy also helps to deal with model misspecification issues, as validation performance is an indicator of whether additional knowledge can be useful for estimation. 
\begin{figure*}[th]
\vspace{-2mm}
{\begin{minipage}{\textwidth}
\begin{center}
\begin{subfigure}{.5\textwidth}
  \centering
  \includegraphics[width=1.0\linewidth]{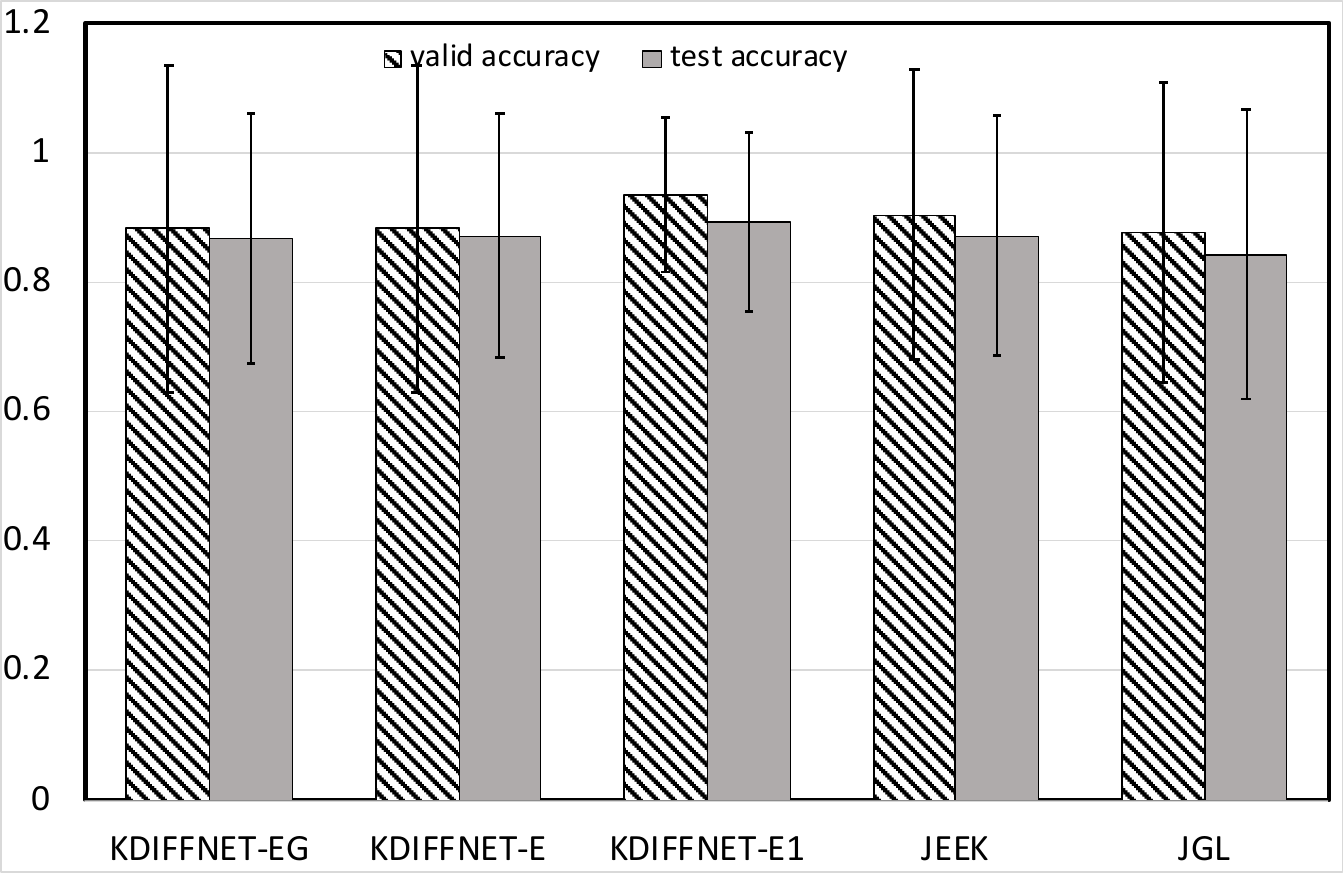}
  \caption{}
  \label{fig:covid-human}
\end{subfigure}%
\begin{subfigure}{.5\textwidth}
  \centering
  \includegraphics[width=1.0\linewidth]{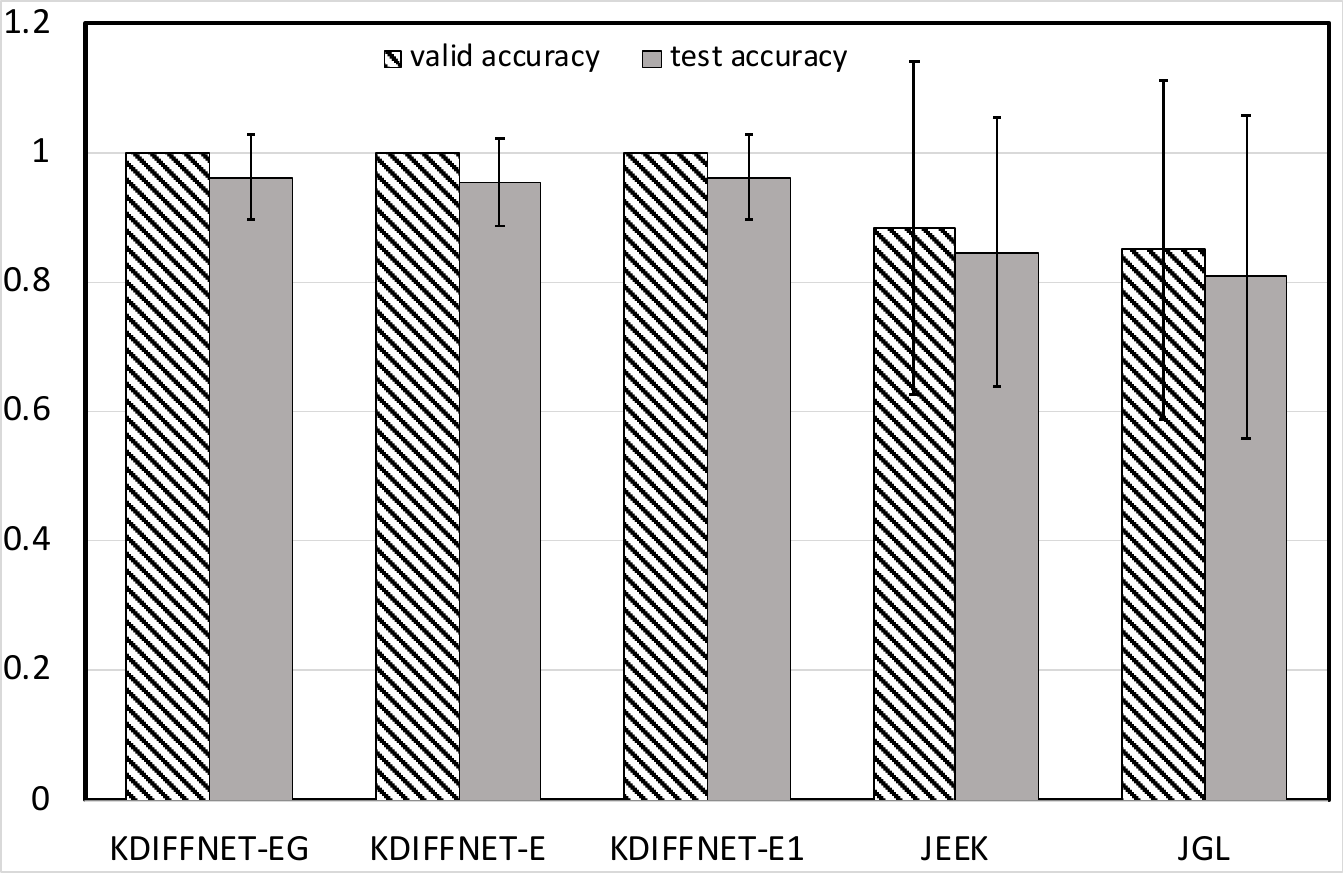}
  \caption{}
  \label{fig:resp}
\end{subfigure}
\end{center}
\end{minipage}%
\hfill
\begin{minipage}{\textwidth}\caption{\footnotesize Validation and Test Accuracy on gene expression datasets : (a)  Human Respiratory Viruses (including SARS-CoV-2) and (b) Mice Respiratory Viruses . \label{fig:covid}}
\end{minipage}}
\vspace{-5mm}
\end{figure*}
\begin{figure*}[h]
\centering
\begin{subfigure}{.5\textwidth}
  \centering
  \includegraphics[width=\textwidth]{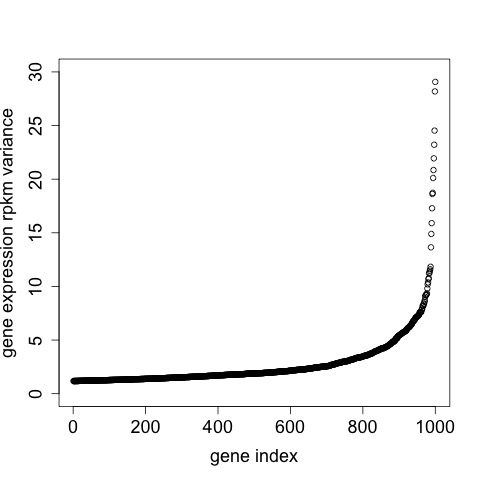}
   
    \label{fig:genevar_covid}
\end{subfigure}%
\begin{subfigure}{.5\textwidth}
  \centering
  \includegraphics[width=\textwidth]{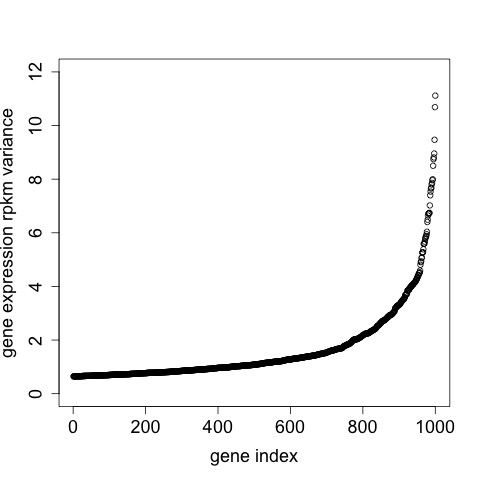}
    
    \label{fig:genevar_resp}
\end{subfigure}
\caption{Variance of gene expression measurements(log of rpkm values). We show the top ranked $1000$ genes: (LEFT) For Human Respiratory Viruses and (RIGHT) For Mouse Respiratory Viruses Dataset.}
\label{fig:genevar}
\end{figure*}
\section{MORE DETAILS ON SIMULATED DATA}
\label{sec:simulatedall}
 We  use simulation to evaluate \methodName for improving differential structure estimation by making use of extra knowledge. In the following subsections, we present details about the data generation, followed by the results under multiple settings.

\subsection{Simulation Dataset Generation}
\label{subsec:simgen}

We generate simulated datasets with a clear underlying differential structure between two conditions, using the following method: 
\paragraph{Data Generation for Edge Knowledge (KE):}~Given a known weight matrix $W_E$ (e.g., spatial distance matrix between $p$ brain regions), we set $W^d = inv.logit(-W_E)$. We use the assumption that higher the value of $W_{ij}$, lower the probability of that edge to occur in the true precision matrix. This is motivated by the role of spatial distance in brain connectivity networks: farther regions are less likely to be connected and vice-versa.  We select different levels in the matrix $W^d$, denoted by $s$, where if $W^d_{ij}>s_l$,  $\Delta^d_{ij}=0.5$, else $\Delta^d_{ij}=0$, where $\Delta^d \in {\R}^{p \times p}$. We denote by $s$ as the sparsity, i.e. the number of non-zero entries in $\Delta^d$. ${B}_I$ is a random graph with  each edge ${B}_{{I}_{ij}}=0.5$ with probability $p$. $\delta_c$ and $\delta_d$ are selected large enough to guarantee positive definiteness. 
\begin{equation}
\label{eref:e-sim}
    {\Omega}_d = \Delta^d + {B}_I + \delta_d I
    \end{equation}
    \begin{equation}
    {\Omega}_c = {B}_I + \delta_c I
\end{equation}
\begin{equation}
    {\Delta} = {\Omega}_d - {\Omega}_c
\end{equation}
There is a clear differential structure in ${\Delta} = {\Omega}_d - {\Omega}_c$, controlled by $\Delta^d$. To generate data from two conditions that follows the above differential structure, we generate two blocks of data samples following Gaussian distribution using $N(0,{\Omega}_c^{-1})$ and $N(0,{\Omega}_d^{-1})$. We only use these data samples to approximate the differential GGM to compare to the ground truth ${\Delta}$. 
 \paragraph{Data Generation for Vertex Knowledge (KG):} In this case, we simulate the case of extra knowledge of nodes in known groups.   
  Let the node group size,i.e., the number of nodes with a similar interaction pattern in the differential graph be $m$. We select the block diagonals of size $m$ as groups in $\Delta^g$. If two variables $i,j$ are in a group $g'$, in $\Delta^g_{ij}=0.5$, else $\Delta^g_{ij}=0$, where $\Delta^g \in \R^{p \times p}$.  We denote by $s_G$ as the number of groups in $\Delta^g$. ${B}_I$ is a random graph with  each edge ${B}_{{I}_{ij}}=0.5$ with probability $p$. %
    
  \begin{equation}
     \label{eref:ev-sim}
  {\Omega}_d =\Delta^g + {B}_I + \delta_d I 
    \end{equation}
    \begin{equation}
   {\Omega}_c ={B}_I + \delta_c I 
\end{equation}
\begin{equation}
    {\Delta} = {\Omega}_d - {\Omega}_c 
   \end{equation}

$\delta_c$ and $\delta_d$ are selected large enough to guarantee positive definiteness. We generate two blocks of data samples following Gaussian distribution using $N(0,{\Omega}_c^{-1})$ and $N(0,{\Omega}_d^{-1})$.
  \paragraph{Data Generation for both Edge and Vertex Knowledge (KEG):} In this case, we simulate the case of overlapping group and edge knowledge.   
  Let the node group size,i.e., the number of nodes with a similar interaction pattern in the differential graph be $m$. We select the block diagonals of size $m$ as groups in $\Delta^g$. If two variables $i,j$ are in a group $g'$, in $\Delta^g_{ij}=1/3$, else $\Delta^g_{ij}=0$, where $\Delta^g \in \R^{p \times p}$. %
 
   For the edge-level knowledge component, given a known weight matrix $W_E$, we set $W^d = inv.logit(-W_E)$. Higher the value of $W_{E_{ij}}$, lower the value of $W^d_{ij}$, hence lower the probability of that edge to occur in the true precision matrix. We select different levels in the matrix $W^d$, denoted by $s$, where if $W^d_{ij}>s_l$, we set $\Delta^d_{ij}=1/3$, else $\Delta^d_{ij}=0$. We denote by $s$ as the number of non-zero entries in $\Delta^d$. ${B}_I$ is a random graph with  each edge ${B}_{{I}_{ij}}=1/3$ with probability $p$. %
    
  \begin{equation}
     \label{eref:ev-sim}
  {\Omega}_d =\Delta^d + \Delta^g + {B}_I + \delta_d I 
    \end{equation}
    \begin{equation}
   {\Omega}_c ={B}_I + \delta_c I 
\end{equation}
\begin{equation}
    {\Delta} = {\Omega}_d - {\Omega}_c 
   \end{equation}

$\delta_c$ and $\delta_d$ are selected large enough to guarantee positive definiteness. Similar to the previous case, we generate two blocks of data samples following Gaussian distribution using $N(0,{\Omega}_c^{-1})$ and $N(0,{\Omega}_d^{-1})$. We only use these data samples to approximate the differential GGM to compare to the ground truth ${\Delta}$. 

We consider three different types of known edge knowledge $W_E$ generated from the spatial distance between different brain regions and simulate groups to represent related anatomical regions. These three are distinguished by different $p=\{116,160,246\}$ representing spatially related  brain regions.  We generate three types of datasets:Data-EG (having both edge and vertex knowledge), Data-G(with edge-level extra knowledge) and Data-V(with known node groups knowledge). We generate two blocks of data samples $\bX_c$ and $\bX_d$ following Gaussian distribution using $N(0,{\Omega}_c^{-1})$ and $N(0,{\Omega}_d^{-1})$. We use these data samples to estimate the differential GGM to compare to the ground truth ${\Delta}$.  We vary the sparsity of the true differential graph ($s$) and the number of control and case samples ($n_c$ and $n_d$ respectively) used to estimate the differential graph. For each case of $p$, we vary  $n_c$ and $n_d$ in $\{p/2,p/4,p,2p\}$ to account for both high dimensional and low dimensional cases. The sparsity of the underlying differential graph is controlled by $s=\{0.125,0.25,0.375,0.5\}$ and $s_G$ as explained above. This results in $126$ different datasets representing diverse settings: different number of dimensions $p$, number of samples $n_c$ and $n_d$, multiple levels of sparsity $s$ and number of groups $s_G$ of the differential graph for both KE and KEG data settings. 
Figure~\ref{fig:exptsetting} summarizes the different settings for simulation datasets.
\begin{figure*}[th]
    \centering
    \includegraphics[width=\textwidth]{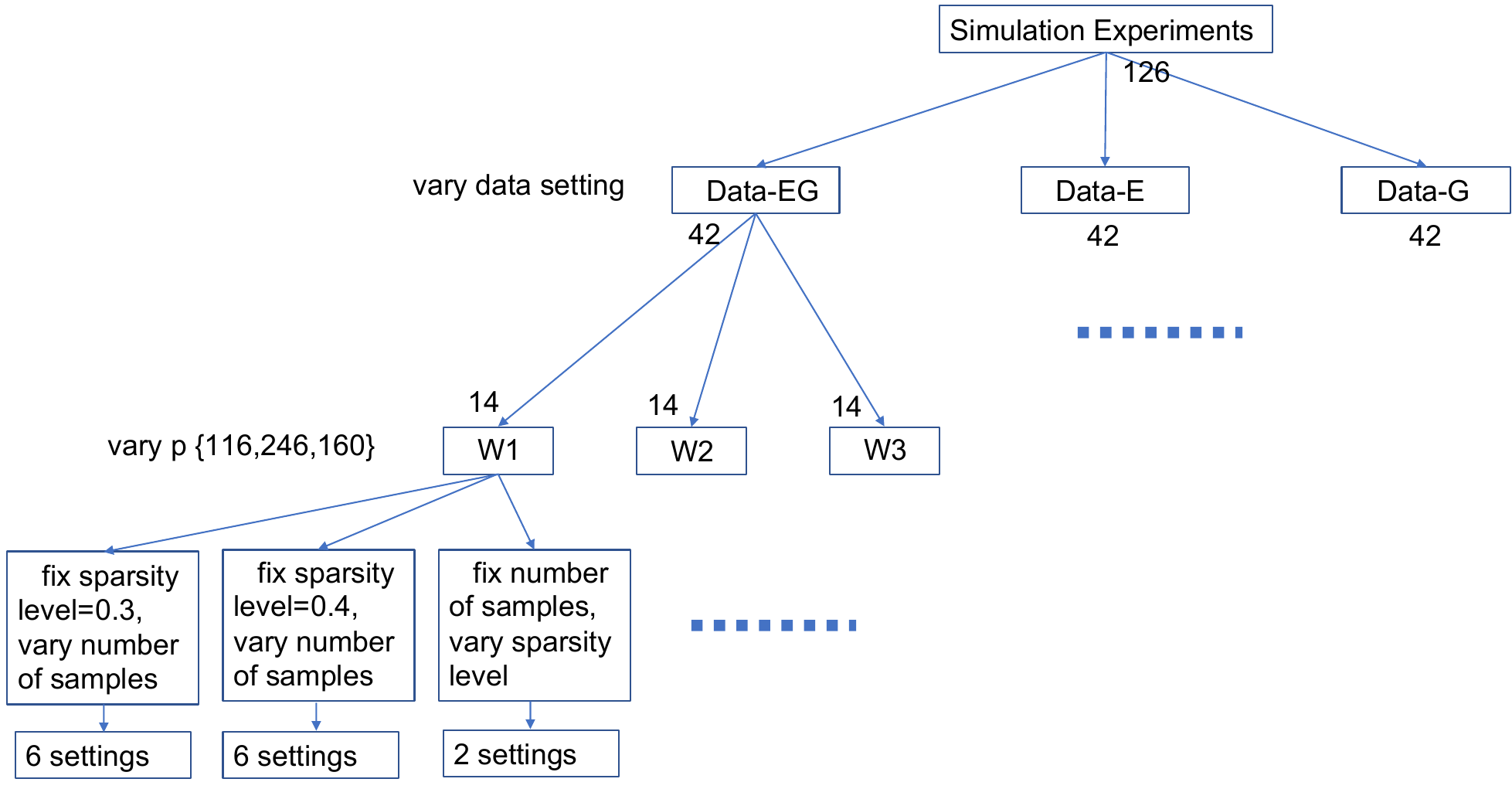}
    \caption{A schematic showing the different experimental settings for simulation experiments.}
    \label{fig:exptsetting}
\end{figure*}
{\bf Experiment Design}: We consider three different types of known edge knowledge $W_E$ generated from the spatial distance between different brain regions and simulate groups to represent related anatomical regions. These three are distinguished by different $p=\{116,160,246\}$ representing spatially related  brain regions.  We vary  $n_c$ and $n_d$ in $\{p/2,p/4,p,2p\}$ to account for both high dimensional and low dimensional cases. The sparsity of the underlying differential graph is controlled by $s=\{0.125,0.25,0.375,0.5\}$ and $s_G$. This results in $126$ different datasets representing diverse settings: multiple $p$, number of samples $n_c$ and $n_d$,  sparsity $s$ and number of groups $s_G$ of the differential graph for both KE and KEG data settings. 
\section{MORE ANALYSIS AND DETAILS ON RESULTS ON SIMULATED DATASETS:}

\subsection{Simulated Results: When we compare with Deep Neural Network based models(GNN)}

\label{subsec:gat}
\begin{table*}[h]
\centering
\begin{tabular}{|c|c|c|c|c|c|c|}\hline
W2  & level & samples & hidden & layers & GAT F1 Score   &  \methodNameEV F1 Score \\ \hline
246 & 4     & 61      & 64     & 1      & 0.0054 & 0.9384                \\\hline
246 & 4     & 123     & 64     & 3      & 0.0102  & 0.9397                 \\\hline
246 & 4     & 246     & 32     & 2      & 0.0095 & 0.9365                \\\hline
246 & 4     & 492     & 64     & 1      & 0.0205   & 0.9430                 \\\hline
246 & 5     & 61      & 5      & 3      & 0.0114  & 0.9225                 \\\hline
246 & 5     & 123     & 64     & 1      & 0.0136  & 0.9219                 \\\hline
246 & 5     & 246     & 32     & 3      & 0.0231   & 0.9248                 \\\hline
246 & 5     & 492     & 16     & 2      & 0.0740  & 0.9302           \\  \hline  
\end{tabular}
\caption{Comparison of \methodNameEV and GAT\cite{velivckovic2017graph} for differential graph recovery. }
\label{tab:gat}
\end{table*}

We compare with Graph Attention Networks\citep{velivckovic2017graph}. Although not designed for differential parameter learning, we explore the graphs learnt by the attention weights in relation to the true differential graph. We formulate it as a classification task, that is each distribution represents
a labeled class. In detail, for each sample, we predict the corresponding data block $\in \{c,d\}$.   We validate over number of layers $\in \{1,2,3,4,5\}$ and hidden size $\{5,16,32,64\}$ for $W2$, $p=246$ and varying samples in $\{61,123,246,492\}$ for train,validation and test sets in each setting. We use one attention head in this setting. We train the models using ADAM optimizer with learning rate $0.0005$ and train each model for $300$ epochs. We pick the model based on the epoch with best validation set classification performance. We use the  training set samples to select a threshold for binarizing the aggregated difference of attention weights across the samples from the two data blocks(classes). We report the F1-Score on the aggregated difference from the classes using attention weights from the test data samples. Table~\ref{tab:gat} shows the GAT performance and corresponding \methodNameEV performance for the different settings. 

We adapt a recently proposed deep learning based neighborhood selection method to estimate network structure\citep{ke2019learning}. We use the setting as proposed for a single task from \cite{Sekhon2020RelateAP}. We compare to the simulation case with samples $n_c = n_d \in \{p, 2p\}$ and $p=\{116,246,160\}$ with sparsity level $5$. We use the same datasets as used in the simulation experiments. We use an MLP layer of size $4\times p$. We show the results in Table~\ref{tab:dl_edge_acc}. We validate over sparsity regularization $\lambda_n\in \{1e-03,1e-04, 1e-05\}$. In such high dimensional cases, MLP based deep models are not able to learn the correct differential structures, as indicated by lower edge level F1 score. 

\begin{table}[]
    \centering
    \begin{tabular}{|c|c|c|c|}\hline
      Method  & W1 & W2  & W3 \\ \hline
        \methodName & 0.74/0.74 & 0.92/0.93 & 0.94/0.94\\ \hline 
         DL & 0.54/0.55  & 0.54/0.54 & 0.56/0.56 \\ \hline
    \end{tabular}
    \caption{Edge Recovery Accuracy of \methodName vs deep learning(DL) based neighborhood selection methods for $n_c=n_d \in \{p,2p\}$. }
    \label{tab:dl_edge_acc}
\end{table}

\subsection{Simulated Results: when our knowledge is partial}

\paragraph{Varying proportion of known edges:} We generate $W_E$ matrices with $p=150$ using Erdos Renyi Graph \citep{erdds1959random}. We use the generated  graph as prior edge knowledge $W_E$. Additionally, we  simulate $15$ groups of size $10$ as explained in ~\sref{subsec:simgen}. We simulate $\Omega_c$ and $\Omega_d$ as explained in ~\sref{subsec:simgen}. Figure~\ref{fig:proportion} presents the performance of \methodNameEV, \methodNameE and {\diffee} with varying proportion of known edges. 

\begin{figure*}[htp]
\centering
\includegraphics[width=.55\textwidth]{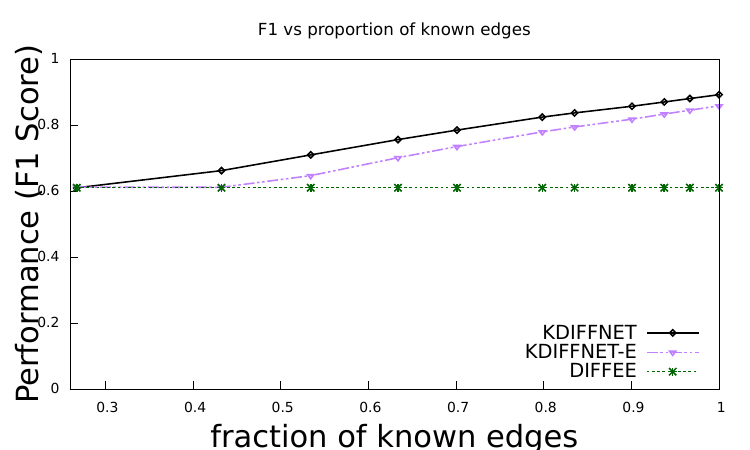}
\caption{F1-Score of \methodNameEV,\methodNameE  and {\diffee} with varying proportion of known edges.}
\label{fig:proportion}
\end{figure*}

\methodNameEV has a higher F1-score than \methodNameE as it can additionally incorporate known group information. As expected, with increase in the proportion of known edges, F1-Score improves for both \methodNameEV and \methodNameE. In contrast {\diffee} cannot make use of additional information and the F1-Score remains the same. 

\subsection{Simulated Results: When Tuning  hyperparameters and Varying $p$}

\paragraph{Scalability in $p$:}  To evaluate the scalability of \methodName and baselines to large $p$, we also generate larger $W_E$ matrices with $p=2000$ using Erdos Renyi Graph \cite{erdds1959random}, similar to the aforementioned design. Using the generated  graph as prior edge knowledge $W_E$, we design $\Omega_c$ and $\Omega_d$ as explained in ~\sref{subsec:simgen}. For the case of both edge and vertex knowledge, we fix the number of groups to $100$ of size $10$. We evaluate the scalability of \methodNameEV and baselines measured in terms of computation cost per $\lambda_n$.

\begin{figure*}[htp]

\centering
\includegraphics[width=.55\textwidth]{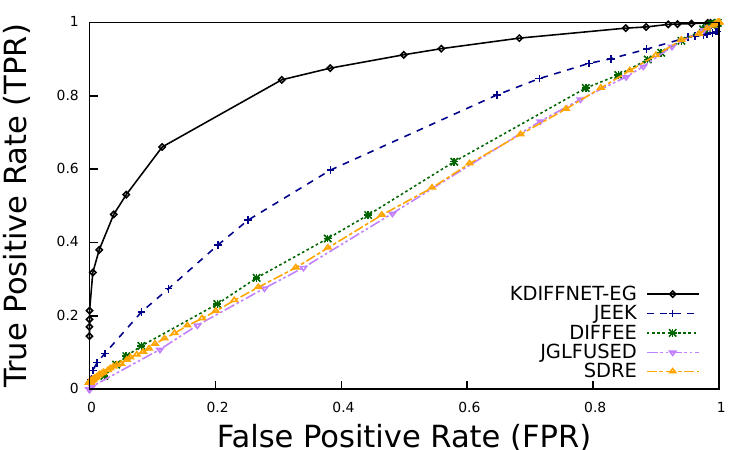}
\caption{ Area Under Curve (AUC) Curves for \methodName and baselines at different hyperparameter values $\lambda$.}
\label{fig:summary_tuning}
\end{figure*}

\begin{figure*}[htp]
\centering
\includegraphics[width=.55\textwidth]{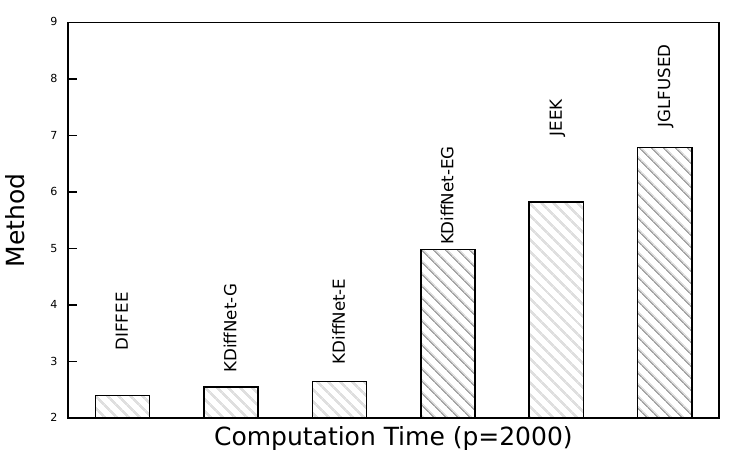}\hfill
\caption{Scalability of \methodName:  Computation Time (log milliseconds) per $\lambda_n$ for large $p=2000$: \methodNameEV has reasonable time cost with respect to baseline methods. \methodNameE and \methodNameV are fast due to closed form. \label{fig:p_time} \label{fig:summary_p}}
\end{figure*}

Figure~\ref{fig:summary_p} shows the computation time cost per $\lambda_n$ for all methods.  Clearly, \methodName takes the least time, for very large $p$ as well. 
\paragraph{Choice of $\lambda_n$:} For \methodName, we show the performance of all the methods as a function of choice of $\lambda_n$. Figure~\ref{fig:summary_tuning} shows the True Positive Rate(TPR) and False Positive Rate(FPR) measured by varying $\lambda_n$ for $p=116$, $s=0.5$ and $n_c=n_d=p/2$ under the Data-EG setting. Clearly, \methodNameEV achieves the highest Area under Curve (AUC) than all other baseline methods. \methodNameEV also outperforms JEEK and NAK that take into account edge knowledge but cannot model the known group knowledge. 
 \subsection{Simulated Results: When we have both edge and group knowledge:}

\paragraph{Edge and Vertex Knowledge (KEG):} We use \methodName (Algorithm~\ref{alg:pp}) to infer the differential structure in this case.

\begin{figure*}[htp]

\centering
\includegraphics[width=.75\textwidth]{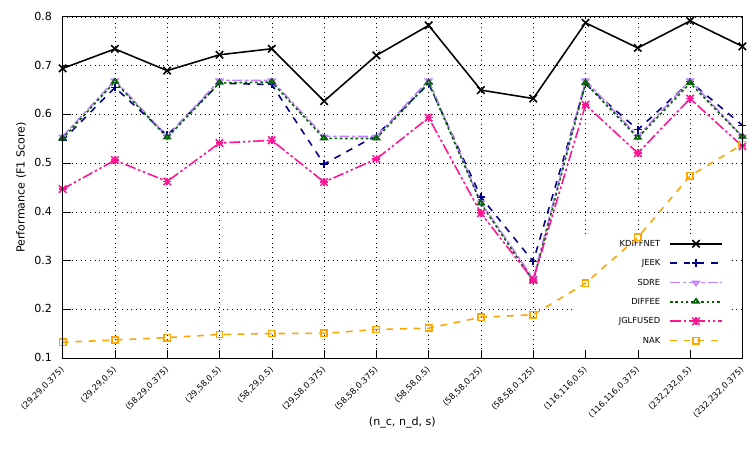}\hfill
\includegraphics[width=.75\textwidth]{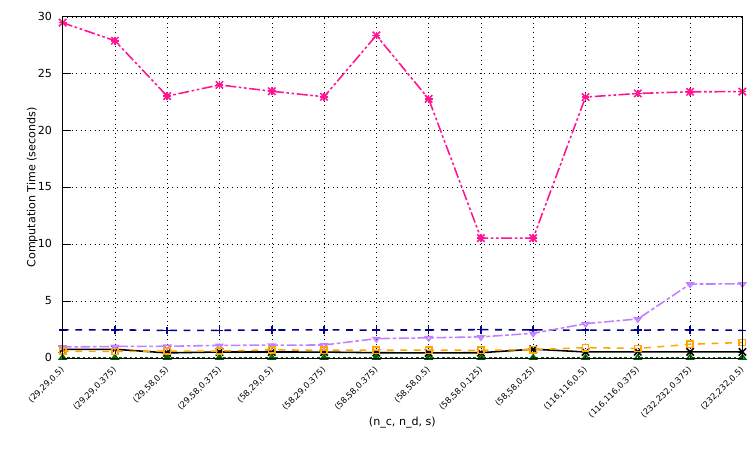}\hfill

\caption{\methodName Edge and Vertex Knowledge Simulation Results for $p=116$ for different settings of $n_c,n_d$ and $s$: (a) The test F1-score  and (b) The average computation time (measured in seconds) per $\lambda_n$ for \methodName and baseline methods.}
\label{fig:summary_o_w1}

\end{figure*}
Figure~\ref{fig:summary_o_w1}(a) shows the performance in terms of F1 Score of \methodName in comparison to the baselines for $p=116$, corresponding to 116 regions of the brain. \methodName outperforms the best baseline in each case by an average improvement of $~414\%$. \methodNameEV does better than JEEK and NAK that can model the edge information but cannot include group information. SDRE and {\diffee} are direct estimators but perform poorly indicating that adding additional knowledge aids differential network estimation.  JGLFUSED performs the worst on all cases.

Figure~\ref{fig:summary_o_w1}(b) shows the average computation cost per $\lambda_n$ of each method measured in seconds. In all settings, \methodName has lower computation cost than JEEK, SDRE and JGLFUSED  in different cases of varying $n_c$ and $n_d$, as well as with different sparsity of the differential network. \methodName is on average $~24\times$ faster than the best performing baseline. It is slower than {\diffee} owing to {\diffee}'s non-iterative closed form solution, however, {\diffee} does not have good prediction performance. Note that $B^*()$ in \methodName, JEEK and {\diffee} and  the kernel term in SDRE are precomputed only once prior to tuning across multiple $\lambda_n$.
\begin{figure*}[htp]

\centering
\includegraphics[width=.75\textwidth]{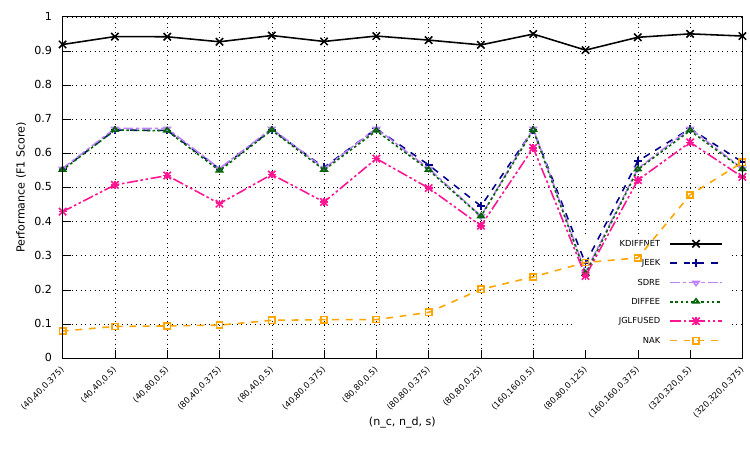}\hfill
\includegraphics[width=.75\textwidth]{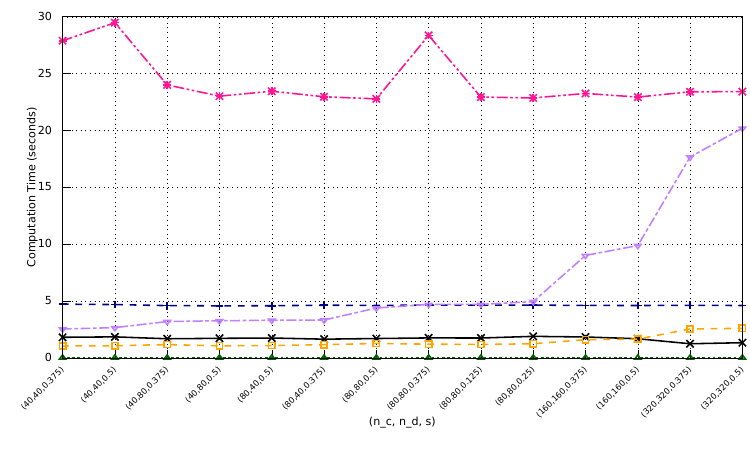}\hfill

\caption{\methodName Edge and Vertex Knowledge Simulation Results for $p=160$ for different settings of $n_c,n_d$ and $s$: (a) The test F1-score  and (b) The average computation time (measured in seconds) per $\lambda_n$ for \methodName and baseline methods.}
\label{fig:summary_o_w3}

\end{figure*}
In Figure~\ref{fig:summary_o_w3}(a), we plot the test F1-score for simulated datasets generated using $W$ with $p=160$, representing spatial distances between different $160$ regions of the brain.  This represents a larger and different set of spatial brain regions.  In $p=160$ case, \methodName outperforms the best baseline in each case by an average improvement of $~928\%$.  Including available additional knowledge is clearly useful as JEEK does relatively better than the other baselines. JGLFUSED performs the worst on all cases. 
Figure~\ref{fig:summary_o_w3}(b) shows the computation cost of each method measured in seconds for each case. %
\methodName is on average $~37\times$ faster than the best performing baseline.

\begin{figure*}[htp]

\centering
\includegraphics[width=.75\textwidth]{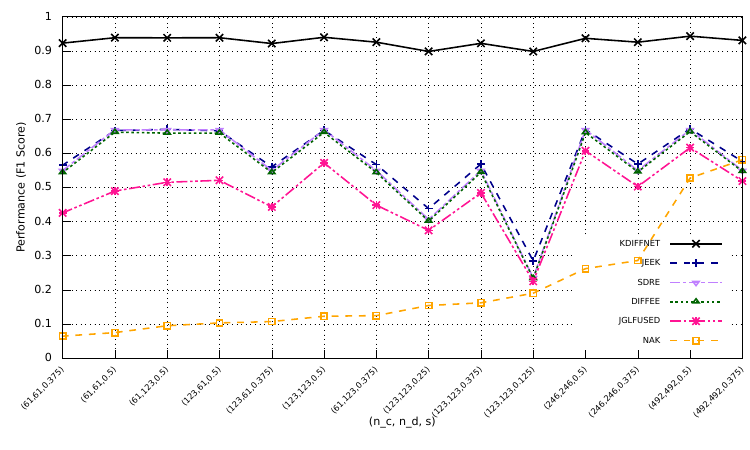}\hfill
\includegraphics[width=.75\textwidth]{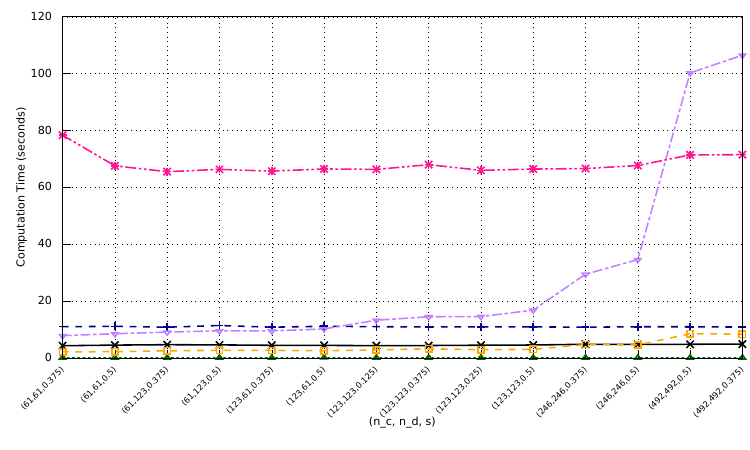}\hfill

\caption{\methodName Edge and Vertex Knowledge Simulation Results for $p=246$ for different settings of $n_c,n_d$ and $s$: (a) The test F1-score  and (b) The average computation time (measured in seconds) per $\lambda_n$ for \methodName and baseline methods}
\label{fig:summary_o_w2}

\end{figure*}
In Figure~\ref{fig:summary_o_w2}(a), we plot the test F1-score for simulated datasets generated using a larger $W_E$ with $p=246$, representing  spatial distances between different $246$ regions of the brain.  This represents a larger and different set of spatial brain regions.  In this case, \methodName outperforms the best baseline in each case by an average improvement of $1400\%$ relative to the best performing baseline.  In this case as well, including available additional knowledge is clearly useful as JEEK does relatively better than the other baselines, which do not incorporate available additional knowledge. JGLFUSED again performs the worst on all cases. 
Figure~\ref{fig:summary_o_w2}(b) shows the computation cost of each method measured in seconds for each case. In all cases, \methodName has the least computation cost in different settings of the data generation. \methodName is on average $~20\times$ faster than the best performing baseline. 

We cannot compare Diff-CLIME as it takes more than ~2 days to finish  $p=246$ case.

\subsection{Simulated Results: When we have only edge knowledge:}

\paragraph{Edge Knowledge (KE):} Given known $W_E$, we use \methodNameE to infer the differential structure in this case. 

\begin{figure*}[htp]

\centering
\includegraphics[width=.75\textwidth]{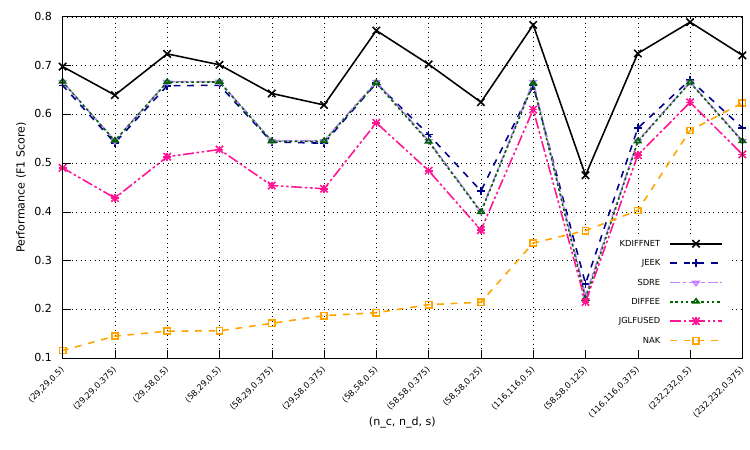}\hfill
\includegraphics[width=.75\textwidth]{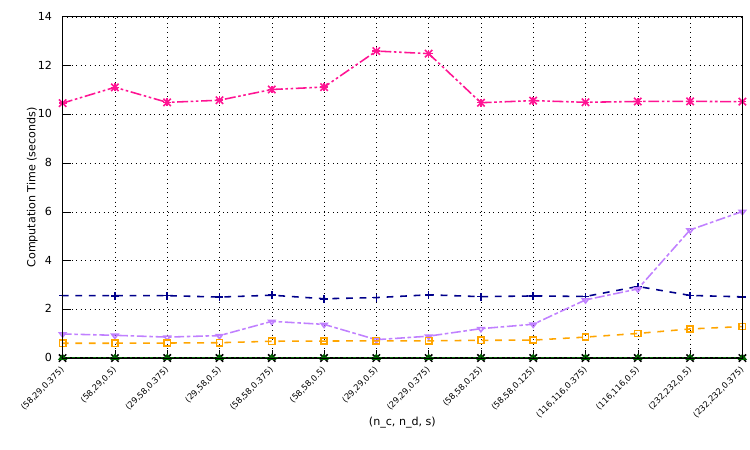}\hfill

\caption{\methodNameE Simulation Results for $p=116$ for different settings of $n_c,n_d$ and $s$: (a) The test F1-score  and (b) The average computation time (measured in seconds) per $\lambda_n$ for \methodNameE and baseline methods.}
\label{fig:summary_w1}

\end{figure*}
Figure~\ref{fig:summary_w1}(a) shows the performance in terms of F1-Score of \methodNameE in comparison to the baselines for $p=116$, corresponding to 116 spatial regions of the brain. In $p=116$ case, \methodNameE outperforms the best baseline in each case by an average improvement of $23\%$.  While JEEK, {\diffee} and SDRE perform similar to each other, JGLFUSED performs the worst on all cases.  

Figure~\ref{fig:summary_w1}(b) shows the computation cost of each method measured in seconds for each case. In all cases, \methodNameE has the least computation cost in different cases of varying $n_c$ and $n_d$, as well as with different sparsity of the differential network. For $p=116$, \methodNameE, owing to an entry wise parallelizable closed form solution, is on average $~2356\times$ faster than the best performing baseline.
\begin{figure*}[htp]

\centering
\includegraphics[width=.75\textwidth]{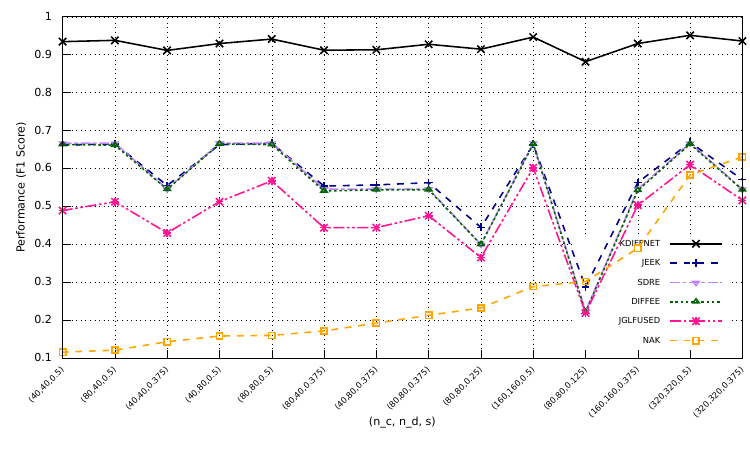}\hfill
\includegraphics[width=.75\textwidth]{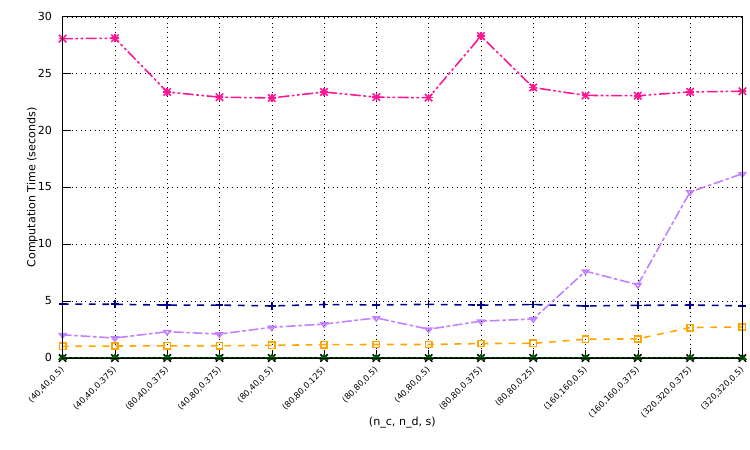}\hfill

\caption{\methodNameE Simulation Results for $p=160$ for different settings of $n_c,n_d$ and $s$: (a) The test F1-score  and (b) The average computation time (measured in seconds) per $\lambda_n$ for \methodNameE and baseline methods.}
\label{fig:summary_w3}

\end{figure*}
In Figure~\ref{fig:summary_w3}(a), we plot the test F1-score for simulated datasets generated using $W$ with $p=160$, representing  spatial distances between different $160$ regions of the brain.  This represents a larger and different set of spatial brain regions.  In $p=160$ case, \methodNameE outperforms the best baseline in each case by an average improvement of $67.5\%$.  Including available additional knowledge is clearly useful as JEEK does relatively better than the other baselines, which do not incorporate available additional knowledge. JGLFUSED performs the worst on all cases. 
Figure~\ref{fig:summary_w3}(b) shows the computation cost of each method measured in seconds for each case. In all cases, \methodNameE has the least computation cost in different cases of varying $n_c$ and $n_d$, as well as with different sparsity of the differential network. \methodNameE is on average $~3300\times$ faster than the best performing baseline.
\begin{figure*}[htp]

\centering
\includegraphics[width=.75\textwidth]{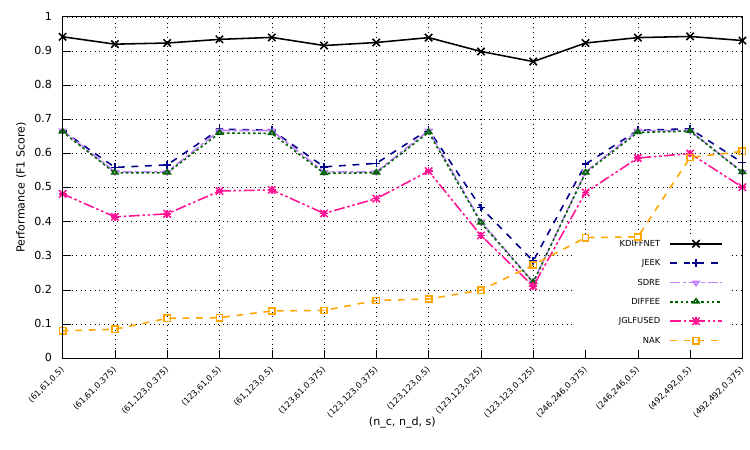}\hfill
\includegraphics[width=.75\textwidth]{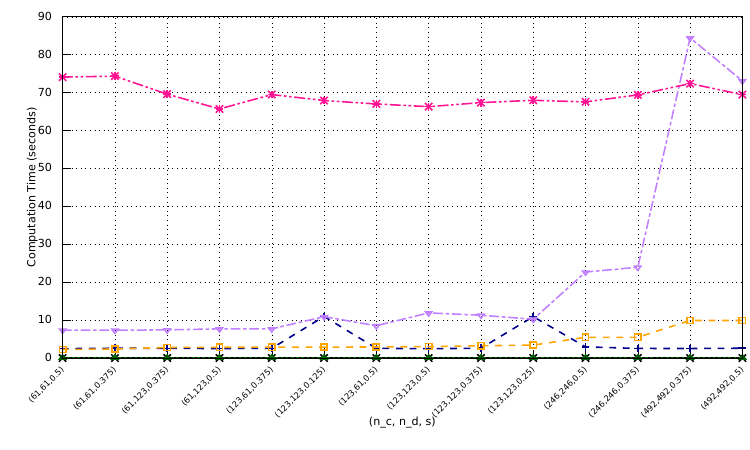}\hfill

\caption{\methodNameE Simulation Results for $p=246$ for different settings of $n_c,n_d$ and $s$: (a) The test F1-score  and (b) The average computation time (measured in seconds) per $\lambda_n$ for \methodNameE and baseline methods.}
\label{fig:summary_w2}

\end{figure*}

In Figure~\ref{fig:summary_w2}(a), we plot the test F1-score for simulated datasets generated using a larger $W$ with $p=246$, representing  spatial distances between different $246$ regions of the brain.  This represents a larger and different set of spatial brain regions.  In this case, \methodNameE outperforms the best baseline in each case by an average improvement of $66.4\%$ relative to the best performing baseline.  Including available additional knowledge is clearly useful as JEEK does relatively better than the other baselines, which do not incorporate available additional knowledge. JGLFUSED performs the worst on all cases. 
Figure~\ref{fig:summary_w2}(b) shows the computation cost of each method measured in seconds for each case. In all cases, \methodNameE has the least computation cost in different cases of varying $n_c$ and $n_d$, as well as with different sparsity of the differential network. \methodNameE is on average $~3966\times$ faster than the best performing baseline.

 \subsection{Simulated Results: When we have only group knowledge:}

\paragraph{Node Group Knowledge}: We use \methodNameV to estimate the differential network with the known groups as extra knowledge. We vary the number of groups $s_G$ and the number of samples $n_c$ and $n_d$ for each case of $p=\{116,160,246\}$. 
 Figure~\ref{fig:group1results} shows the F1-Score of \methodNameV and the baselines for $p=116$. \methodNameV clearly has a large advantage when extra node group knowledge is available. The baselines cannot model such available knowledge.  
 
 \begin{figure*}[htp]

\centering
\includegraphics[width=.75\textwidth]{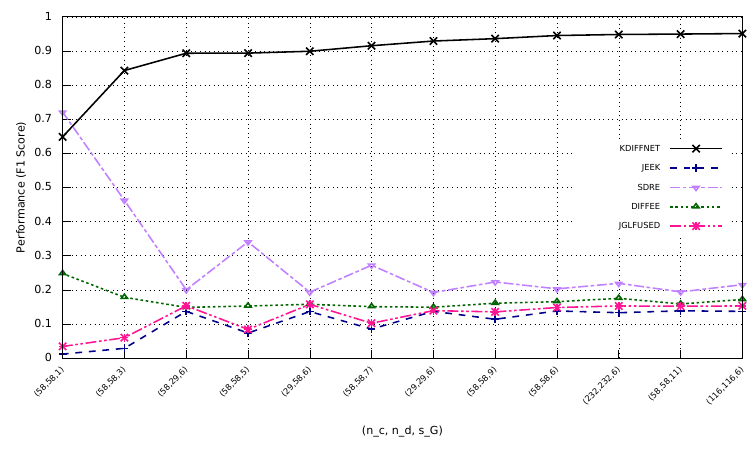}\hfill

\caption{\methodNameV Simulation Results for $p=246$ for different settings of $n_c,n_d$ and $s$: (a) The test F1-score  and (b) The average computation time (measured in seconds) per $\lambda_n$ for \methodNameE and baseline methods.}
\label{fig:group1results}

\end{figure*}
 \subsection{ Knowledge of Perturbed Hub
nodes}
\label{subsec:perturb_data}
We consider the case where there exists a set of nodes $k\in P$ such that the group of edges defined by $\Omega_{c_{k,j}}=0$ and $\Omega_{d_{k,j}}\neq 0$, where $\forall j \in \{1, \dots, p\}, k \in P$. Here, P denotes the set of perturbed nodes. 

To generate the simulation data, $\Delta{{k,j}}=1.0$ where $\forall j \in \{1, \dots, p\}, k \in P$. $ {\Omega}_d =\Delta + {B}_I + \delta_d I 
   $, ${\Omega}_c ={B}_I + \delta_c I $, finally, $ {\Delta} = {\Omega}_d - {\Omega}_c$. $\delta_c$ and $\delta_d$ are selected large enough to guarantee positive definiteness. We generate two blocks of data samples following Gaussian distribution using $N(0,{\Omega}_c^{-1})$ and $N(0,{\Omega}_d^{-1})$.

We report our results in Figure~\ref{fig:perturb}. We compare \methodNameV,   \methodNameE, JEEK, \diffee, and JGL-perturb. \methodNameV directly takes into account the perturbed groups, by imposing a group penalty on the relevant edges. For \methodNameE and JEEK, we set $W_{k,j}=0.1$.

\end{document}